\documentclass[accepted]{uai2025} 
                        
\usepackage[american]{babel}

\usepackage{natbib} 
\usepackage{mathtools} 
\usepackage{booktabs} 
\usepackage{amssymb}
\usepackage{wrapfig}
\usepackage{tikz} 

\usepackage{amsthm}
\usepackage{microtype}
\usepackage{graphicx}
\usepackage{subfigure}
\usepackage{booktabs} 

\usepackage{hyperref}

\usepackage{soul}
\usepackage{wrapfig}
\usepackage{multirow}
\usepackage{booktabs} 
\usepackage{mathrsfs}
\usepackage{colortbl}
\usepackage{wrapfig,lipsum,booktabs}
\usepackage{microtype}
\usepackage{graphicx}
\usepackage{bbm}
\usepackage{subfigure}
\usepackage{booktabs} 
\usepackage{verbatim}
\usepackage{amsmath}
\usepackage{amsfonts}
\usepackage{multirow}
\usepackage{booktabs}
\usepackage{makecell}
\usepackage{colortbl}
\usepackage{xcolor}
\usepackage{caption}
\usepackage[flushleft]{threeparttable}
\usepackage{thmtools}
\usepackage{thm-restate}
\usepackage{tcolorbox}

\usepackage{hyperref}
\usepackage{algorithm}
\usepackage{algorithmic}

\usepackage{xcolor}

\usepackage{xcolor}

\newcommand{\x}{\mathbf{x}}
\newcommand{\y}{\mathbf{y}}
\newcommand{\F}{\mathbf{F}}

\newcommand{\bd}{\mathbf{d}}
\newcommand{\bu}{\mathbf{u}}

\newtheorem{lem}{Lemma}
\newtheorem{assump}{Assumption}
\newtheorem{defn}{Definition}

\newcommand{\alg}{$\mathsf{STIMULUS}~$}
\newcommand{\algns}{$\mathsf{STIMULUS}$}
\newcommand{\algm}{$\mathsf{STIMULUS}$-$\mathsf{M}~$}
\newcommand{\algmns}{$\mathsf{STIMULUS}$-$\mathsf{M}$}
\newcommand{\algp}{$\mathsf{STIMULUS}^+~$}
\newcommand{\algpns}{$\mathsf{STIMULUS}^+~$}
\newcommand{\algmp}{$\mathsf{STIMULUS}$-$\mathsf{M}^+~$}
\newcommand{\algmpns}{$\mathsf{STIMULUS}$-$\mathsf{M}^+$}

\title{\algns:~Achieving Fast Convergence and Low Sample Complexity in Stochastic Multi-Objective Learning}

\author[4]{Zhuqing Liu$^{1}$, Chaosheng Dong$^{2}$, Michinari Momma$^{2}$, Simone Shao$^{2}$, \\ Shaoyuan Xu$^{2}$, Yan Gao$^{2}$, Haibo Yang$^{3}$, Jia Liu}

\affil[1]{%
    Computer Science and Engineering\\
    University of North Texas\\
    Denton, TX, USA
}
\affil[2]{%
    Amazon\\
    Seattle, WA, USA
}
\affil[3]{%
    Computing and Information Sciences\\
    Rochester Institute of Technology\\
    Rochester, NY, USA
  }
\affil[4]{%
    Electrical and Computer Engineering\\
    The Ohio State University\\
    Columbus, OH, USA
  }

\begin{document}
  
\maketitle


\begin{abstract}
Recently, multi-objective optimization (MOO) has gained attention for its broad applications in ML, operations research, and engineering.
However, MOO algorithm design remains in its infancy and many existing MOO methods suffer from unsatisfactory convergence rate and sample complexity performance.
To address this challenge, in this paper, we propose an algorithm called \alg (\ul{st}ochastic path-\ul{i}ntegrated \ul{mul}ti-gradient rec\ul{u}rsive e\ul{s}timator), a new and robust approach for solving MOO problems. 
Different from the traditional methods,
\alg introduces a simple yet powerful recursive framework for updating stochastic gradient estimates to improve convergence performance with low sample complexity.
In addition, we introduce an enhanced version of \algns, termed \algmns, which incorporates a momentum term to further expedite convergence.
We establish $\mathcal{O}(1/T)$ convergence rates of the proposed methods for non-convex settings and $\mathcal{O}(\exp{-\mu T})$ for strongly convex settings, where  $T$ is the total number of iteration rounds. Additionally, we achieve the state-of-the-art $O\left(n+\sqrt{n}\epsilon^{-1}\right)$  sample complexities for non-convex settings and $\mathcal{O}\left(n+ \sqrt{n} \ln ({\mu/\epsilon})\right)$ for strongly convex settings, where $\epsilon>0$ is a desired stationarity error.
Moreover, to alleviate the periodic full gradient evaluation requirement in \alg and \algmns, we further propose enhanced versions with adaptive batching called \algpns/ \algmp and provide their theoretical analysis.
\end{abstract}


\section{Introduction} \label{sec: intro}

{\bf 1) Background of multi-objective learning:}
Machine learning (ML) has always heavily relied on optimization formulations and algorithms.
While traditional ML problems generally focus on minimizing a single loss function, many emergent complex-structured multi-task ML problems require balancing {\em multiple} objectives that are often conflicting (e.g., multi-agent reinforcement learning~\citep{parisi2014policy}, multi-task fashion representation learning~\citep{jiao2022fine,jiao2023learning}, multi-task recommendation system~\citep{chen2019co,zhou2023multi}, multi-model learning in video captioning~\citep{pasunuru-bansal-2017-multi}, and multi-label learning-to-rank~\citep{mlltr2023kdd,querymlltr2023kdd}).
Such ML applications necessitate solving {\em multi-objective} optimization (MOO) problems, which can be expressed as:
\begin{small}
    \begin{align} \label{eq: moo}
\min_{\x \in \mathcal{D}} \F(\x) := [f_1(\x), \cdots, f_S(\x) ],
\end{align}
\end{small}
where $\x \in \mathcal{D} \subseteq \mathbb{R}^d$ is the model parameters.
Here, each $f_s$ denotes the objective function of task $s \in [S]$,
$f_s(\x)= \frac{1}{n}\sum_{j=1}^n f_{sj} (\x ; \xi_{sj})$, where $n$ denotes the total number of samples, $\xi_{sj}$ denotes the $j$-th sample for task $s$.
%
%
However, unlike traditional single-objective optimization, there may not exist a common $\x$-solution in MOO that can simultaneously minimize all objective functions.
Instead, a more relevant optimality criterion in MOO is the notion of \textit{Pareto-optimal solutions}, where no objective can be further improved without sacrificing other objectives.
Moreover, in settings where the set of objective functions are non-convex, searching for Pareto-optimal solutions is intractable in general.
In such scenarios, the goal of MOO is usually weakened to finding a {\em Pareto-stationary solution}, where no improving direction exists for any objective without sacrificing other objectives.
%

%

\begin{table*}[t!]
\centering
\begin{scriptsize}
\begin{threeparttable}
\caption{Convergence comparisons between MOO algorithms, where $n$ is the size of dataset; $\epsilon$ is the convergence error. Our proposed algorithms are marked in a shaded background.}
\label{tab}
\renewcommand{\arraystretch}{1.2}
\begin{tabular}{cccccc}
\toprule
\multirow{2}{*}{Algorithm} & \multirow{2}{*}{Multi-gradient} &  \multicolumn{2}{c}{Non-convex case} & \multicolumn{2}{c}{Strongly-Convex case} \\
\cmidrule(r){3-4} \cmidrule(l){5-6}
& & Rate & Sample Complexity & Rate & Sample Complexity \\
\midrule
  MGD~\citep{fliege2019complexity} & Deterministic & $\mathcal{O}\left(T^{-1}\right)$ & $\mathcal{O}\left(n \epsilon^{-1}\right)$ & $\mathcal{O}(\exp(-\mu T))$ & $\mathcal{O}\left( n\ln ({\mu/\epsilon})\right)$ \\
 \midrule 
SMGD~\citep{yang2022pareto} & Stochastic&$\mathcal{O}\left({{T^{-{1/2}}}}\right)$ & $\mathcal{O}\left(\epsilon^{-2}\right)$ & $\mathcal{O}\left(T^{-1}\right)$ & $\mathcal{O}\left( \epsilon^{-1}\right)$ \\\midrule MoCo~\citep{fernando2022mitigating} &Stochastic & $\mathcal{O}\left({{T^{-{1/2}}}}\right)$ & $\mathcal{O}\left(\epsilon^{-2}\right)$ & $\mathcal{O}\left(T^{-1}\right)$ & $\mathcal{O}\left( \epsilon^{-1}\right)$ 
\\\midrule MoCo+~\citep{10446038} &Stochastic & $\mathcal{O}\left({{T^{-{2/3}}}}\right)$ & $\mathcal{O}\left(\epsilon^{-{1.5}}\right)$ & - & -\\\midrule
 CR-MOGM~\citep{zhou2022on} &Stochastic &$\mathcal{O}\left({{T^{-{1/2}}}}\right)$ & $\mathcal{O}\left(\epsilon^{-2}\right)$ & $\mathcal{O}\left(T^{-1}\right)$ & $\mathcal{O}\left( \epsilon^{-1}\right)$ \\
\midrule
\arrayrulecolor{gray!20}
\rowcolor{gray!20}\textbf{ \algns/ \algm } &  Stochastic & {$\mathcal{O}\left(T^{-1}\right)$ }& {$\mathcal{O}\left(n+\sqrt{n}\epsilon^{-1}\right)$} & {$\mathcal{O}(\exp(-\mu T))$} & {$\mathcal{O}\left(n+ \sqrt{n} \ln ({\mu/\epsilon})\right)$} \\ \midrule
\rowcolor{gray!20}  \textbf{\algpns/ \algmp }& Stochastic  & {$\mathcal{O}\left(T^{-1}\right)$ }& {$\mathcal{O}\left(n+\sqrt{n}\epsilon^{-1}\right)$} & {$\mathcal{O}(\exp(-\mu T))$} & {$\mathcal{O}\left(n+ \sqrt{n} \ln ({\mu/\epsilon})\right)$} \\
\arrayrulecolor{black}
\bottomrule
\end{tabular}
\end{threeparttable}
\end{scriptsize}

\end{table*}

{\bf 2) Motivating application: Multi-label learning to rank (MLLTR) problem.} 
Problem~\eqref{eq: moo} can be applied to a number of interesting real-world problems. Here, we provide one concrete example to further motivate its practical relevance:

The learning to Rank (LTR) method is a common technique used to rank information based on relevance, but it often struggles with ambiguity because of the noisy nature of human-generated data, like product ratings. To tackle this, Multi-Label Learning to Rank (MLLTR) offers a more refined approach. MLLTR addresses the inherent challenges of traditional LTR methods by integrating multiple relevance criteria into the ranking model. This allows for a more comprehensive representation of diverse crucial objectives.

\begin{list}{\labelitemi}{\leftmargin=0.5em \itemindent=-0.2em \itemsep=-0.2em}
	\item {\em Learning to Rank:} 
Let $A$ be the training set, consisting of pairs $(\mathbf{a}_i,{b}_i)$ where $\mathbf{a}_i \in \mathbb{R}^d$ representing features, and $\mathbf{b}$ is the corresponding list of relevance labels $b_i$, and $ i = 1, \ldots, n $. We note that the lists $\mathbf{a}$ within the training set may not all be of the same length. $\mathbf{x}$ is the model parameter.

The goal of the learning-to-rank problem is to find a scoring function $f$ that optimizes a chosen Information Retrieval (IR) metric, such as Normalized Discounted Cumulative Gain (NDCG), on the test set. The scoring function $f$ is trained to minimize the mean of a surrogate loss $l$ across the training data:
$
f_{single}(\mathbf{x}) = \frac{1}{|A|} \sum_{(\mathbf{a}, \mathbf{b}) \in A} l( f(\mathbf{x};{\mathbf{a}}), \mathbf{b}).
$

	\item {\em Multi-label Learning to Rank:} Learning to Rank from multiple relevance labels. In the problem of Multi-label learning to rank (MLLTR), different relevance criteria are measured, providing multiple labels for each feature vector $\mathbf{a}_i\in \mathbb{R}^d$. The goal of MLLTR is still the same as that of LTR, which is to learn a scoring function $f(\x;\mathbf{a})$ that assigns a scalar value to each feature vector $\mathbf{a}_i\in \mathbb{R}^d$.
Here, we consider a set of training examples denoted by $ \mathbf{a}_i \in \mathbb{R}^d$, where $ i = 1, \ldots, n$. Associated with each training example $ \mathbf{a}_i $ is a vector of class labels:
$
\mathbf{b}_i = \left({b}_i^1, \ldots, {b}_i^K\right),$
indicating the labels of $\mathbf{a}_i $. Here, $ K $ is the total count of possible labels. In the multi-label learning to rank problem, the objective is to construct $ K $ distinct classification functions:
$
f_k({\x}): \mathbb{R}^d \rightarrow \mathbb{R}, \text{ for } k = 1, \ldots, K,
$
each tailored to a specific label.

In MLLTR, the cost is a vector-valued function: $f({{\x}}) = [f_1({\x}) ,f_2({\x}),f_K({\x})  ],$ naturally making it an MOO problem.
\end{list}

In the search ranking domain, the objective is to rank search results based on their relevance to user queries and other factors such as popularity, user feedback, and conversion rates. The loss function in search ranking not only considers relevance but also takes into account various performance metrics, such as click-through rates (CTR), dwell time, or conversion rates\cite{lyu2020deep, yang2020empirically,xiao2020deep}. The goal is to optimize the ranking of search results to maximize user satisfaction and engagement. Common loss functions used in search ranking include pairwise ranking loss\cite{kumar2020deep,jing2019deep,wang2021pairwise}, listwise loss\cite{revaud2019learning,yu2019wassrank}, or evaluation metrics like normalized discounted cumulative gain (NDCG)\cite{bruch2019analysis} or mean average precision (MAP)\cite{revaud2019learning}. These loss functions aim to capture the overall quality of the search ranking by considering both relevance and performance metrics.

The multi-label learning to rank problem typically involves a larger number of labels, which increases the dimensionality of the output space. This higher dimensionality often necessitates a greater number of samples to accurately train models, resulting in increased sample complexity. Therefore, this motivates us to propose a new family of algorithms for low sample complexity and fast convergence rates.

{\bf 3) Related works and motivation:} 
 To date, existing MOO algorithms in the literature can be generally categorized as gradient-free and gradient-based methods.
Typical gradient-free methods include evolutionary MOO algorithms and Bayesian MOO algorithms~\citep{zhang2007moea,deb2002fast,belakaria2020uncertainty,laumanns2002bayesian}. These techniques are suitable for small-scale problems but inefficient in solving high-dimensional MOO models (e.g., deep neural networks). 
Notably, gradient-based methods have attracted increasing attention recently due to their stronger empirical performances.
Specifically, following a similar token of (stochastic) gradient descent methods for single-objective optimization, (stochastic) multi-gradient descent (MGD/SMGD) algorithms have been proposed in~\citep{fliege2019complexity,fernando2022mitigating,zhou2022on,liu2021stochastic}.
The basic idea of MGD/SMGD is to iteratively update the $\x$-variable following a common descent direction for all the objectives through a time-varying convex combination of (stochastic) gradients of all objective functions.
Although MGD-type algorithms enjoy a fast $\mathcal{O}(1/T)$ convergence rate ($T$ denotes the number of iterations) in finding a Pareto-stationary solution, their $\mathcal{O}(n)$ per-iteration computation complexity in full multi-gradient evaluations becomes prohibitive when the dataset size $n$ is large.
As a result, SMGD-type algorithms are often more favored in practice thanks to the lower per-iteration computation complexity in evaluating stochastic multi-gradients.
However, due to the noisy stochastic multi-gradient evaluations, SMGD-type algorithms typically exhibit a slow $\mathcal{O}(1/\sqrt{T})$ convergence rate, which also induces a high $\mathcal{O}(\epsilon^{-2})$ sample complexity.
Although SMGD is easier to implement in practice thanks to the use of stochastic multi-gradient, it has been shown that the noisy common descent direction in SMGD could potentially cause divergence (cf. the example in Sec.~4 in \citep{zhou2022on}).
There also have been recent works on using momentum-based methods for bias mitigation in MOO, named MoCo~\citep{fernando2022mitigating}, MoCo+~\citep{10446038}, CR-MOGM~\citep{zhou2022on}. 
For easier comparisons, we summarize the state-of-the-art gradient-based MOO algorithms and their convergence rate results under non-convex and strongly convex settings in Table~\ref{tab}. We note that given the limited research on finite-sum multi-objective optimization, we included broader comparisons.

In light of these major limitations of SMGD-type algorithms, a fundamental question naturally emerges:

\begin{tcolorbox}[left=1.2pt,right=1.2pt,top=1.2pt,bottom=1.2pt]
\textbf{(Q)}: Is it possible to develop fast-convergent stochastic MOO algorithms in the sense of matching the convergence rate of deterministic MGD-type methods, while having a low per-iteration computation complexity as in SMGD-type algorithms, as well as achieving a low overall sample complexity?
\end{tcolorbox}

To be specific, our algorithms differ from them in the following key aspects: (i) Our algorithms only require a constant level step size, which is easier to tune in practice. (ii) Our STIMULUS family of algorithms has a lower sample complexity compared to all other existing methods.

{\bf 4) Technical Challenges:}
As in traditional single-objective optimization, a natural idea to achieve both fast convergence and low sample complexity in MOO is to employ the so-called ``variance reduction'' (VR) techniques to tame the noise in stochastic multi-gradients in SMGD-type methods.
However, due to the complex coupling nature of MOO problems, developing VR-assisted algorithms for SMGD-type algorithms faces the following challenges {\em unseen} in their single-objective counterparts:

(1) Since SMGD-type methods aim to identify the Pareto front (i.e., the set of all Pareto-optimal/stationary solutions), it is critical to ensure that the use of VR techniques does not introduce new bias into the already-noisy SGMD-type search process, which drives the search process toward certain regions of the Pareto front. 
(2) MOO problems often involve higher computational complexity compared to single-objective problems due to the need to evaluate multiple objectives simultaneously. Incorporating VR techniques adds another layer of complexity, as it requires additional computations to estimate and reduce variance across multiple objectives.
(3) Conducting theoretical analysis to prove the convergence performance of some proposed VR-based SMGD-type techniques also contains multiple challenges, including how to quantify multiple conflicting objectives, navigating trade-offs between them, handling the non-convexity objective functions, and managing the computational cost of evaluations.
All of these analytical challenges are quite different from those in single-objective optimization theoretical analysis,
which necessitate specialized proofs and analyses are needed to effectively tackle these challenges and facilitate efficient exploration of the Pareto optimality/stationarity.

{\bf 5) Main Contributions:}
The major contribution of this paper is that we overcome the aforementioned technical challenges and develop a suite of new VR-assisted SMGD-based MOO algorithms called \alg (\ul{st}ochastic path-\ul{i}ntegrated \ul{mul}ti-gradient rec\ul{u}rsive e\ul{s}timator) to achieve both fast convergence and low sample complexity in MOO.
Our main technical results are summarized as follows:

\begin{list}{\labelitemi}{\leftmargin=0.5em \itemindent=-0.2em \itemsep=-0.2em}
\item 
Our \alg algorithm not only enhances computational efficiency but also significantly reduces multi-gradient estimation variance, leading to more stable convergence trajectories and overcoming the divergence problem of SMGD. 
We theoretically establish a convergence rate of $\mathcal{O}(1/T)$ for \alg in non-convex settings (typical in ML), which further implies a low sample complexity of $O\left(n+\sqrt{n}\epsilon^{-1}\right)$. 
In the special setting where the objectives are strongly convex, we show that \alg has a linear convergence rate of $\mathcal{O}(\exp(-\mu T))$, which implies an even lower sample complexity of $\mathcal{O}\left( n+\sqrt{n} \ln ({\mu/\epsilon})\right)$. 

\item  
To further improve the performance of \algns, we develop an enhanced version called \algmns, which incorporates momentum information to expedite convergence speed. 
Also, to relax the requirement for periodic full multi-gradient evaluations in \alg and \algmns, we propose two enhanced variants called \algp and \algmp based on adaptive batching, respectively. 
We provide theoretical convergence and sample complexity analyses for all these enhanced variants. 
These enhanced variants expand the practical utility of \algns, offering efficient solutions that not only accelerate optimization processes but also alleviate computational burdens
in a wide spectrum of multi-objective optimization applications.

\item 
We conduct extensive experiments on a variety of challenging MOO problems to verify our theoretical results and illustrate the efficacy of the \alg algorithm family.  
Our experiments demonstrate the efficiency of the \alg algorithm family over existing state-of-the-art MOO methods, which underscore the robustness, scalability, and flexibility of our \alg algorithm family in complex MOO applications.
\end{list}

\section{Preliminaries} \label{sec: prelim}

To facilitate subsequent technical discussions, in this section, we first provide a primer on MOO fundamentals and formally define the notions of Pareto optimality/stationarity, $\epsilon$-stationarity in MOO, and the associated sample complexity.
Then, we will give an overview of the most related work in the MOO literature, thus putting our work into comparative perspectives.

{\bf Multi-objective Optimization: A primer.}
As introduced in Section~\ref{sec: intro}, MOO aims to optimize multiple objectives in Eq.~\eqref{eq: moo} simultaneously.
%
However, since in general there may not exist an $\x$-solution that minimizes all objectives at the same time in MOO, the more appropriate notion of optimality in MOO is the so-called {\em Pareto optimality,} which is formally defined as follows:

\begin{defn}[(Weak) Pareto Optimality]
\label{def:weakPareto}
Given two solutions $\x$ and $\y$, $\x$ is said to dominate $\y$ only if $f_s(\x) \leq f_s(\y), \forall s \in [S]$ and there exists at least one function, $f_s$, where $f_s(\x) < f_s(\y)$.
A solution $\x_*$ is Pareto optimal if no other solution dominates it.
A solution $\x$ is defined as weakly Pareto optimal if there is no solution $\y$ for which $f_s(\x) > f_s(\y), \forall s \in [S]$.
\end{defn}

Finding a Pareto-optimal solution in MOO is as complex as solving single-objective non-convex optimization problems and is NP-Hard in general. 
Consequently, practical efforts in MOO often aim to find a solution that meets the weaker notion called Pareto-stationarity (a necessary condition for Pareto optimality), which is defined as follows~\cite{fliege2000steepest,miettinen2012nonlinear}:

\begin{defn} [Pareto Stationarity] \label{defn:ParetoStationarity}
A solution $\x$ is Pareto-stationary if no common descent direction $\bd \in \mathbb{R}^d$ exists such that $\nabla f_s(\x)^{\top} \bd < 0, \forall s \in [S]$.
\end{defn}
Note also that in the special setting with strongly convex objective functions, Pareto-stationary solutions are Pareto-optimal.
Following directly from Pareto-stationarity in Definition~\ref{defn:ParetoStationarity}, gradient-based MOO algorithms strive to find a common descent (i.e., improving) direction $\bd \in \mathbb{R}^d$, such that $\nabla f_s(\x)^{\top} \bd \leq 0, \forall s \in [S]$. 
If such a direction does not exist at $\x$, then $\x$ is Pareto-stationary. 
Toward this end, the MGD method~\citep{desideri2012multiple} identifies an optimal weight $\boldsymbol{\lambda}^*$ for the multi-gradient set $\nabla \F(\x) \triangleq \{ \nabla f_s(\x), \forall s \in [S] \}$ by solving $\boldsymbol{\lambda}^*(\x) \in \operatorname*{argmin}_{\boldsymbol{\lambda} \in C} \| \boldsymbol{\lambda}^{\top} \nabla \F(\x) \|^2$. Consequently, the common descent direction can be defined as $\bd = \boldsymbol{\lambda}^{\top} \nabla \F(\x)$.
Then, MGD follows the iterative update rule $\x \leftarrow \x - \eta \bd$ in the hope that a Pareto-stationary point can be reached, where $\eta$ signifies a learning rate. 
SMGD~\cite{liu2021stochastic} follows a similar approach, but with full multi-gradients being replaced by stochastic multi-gradients. 
For both MGD and SMGD, it has been shown that if $\| \boldsymbol{\lambda}^{\top} \nabla \F(\x) \| = 0$ for some $\boldsymbol{\lambda} \in C$, where $C \triangleq \{ \y \in [0, 1]^S, \sum_{s \in [S]} y_s = 1 \}$, then $\x$ is a Pareto stationary solution \cite{fliege2019complexity,zhou2022on}.

Here, it is insightful to contrast vector-valued MOO with the linear scalarization method with fixed weights for MOO, which is also a relatively straightforward approach commonly seen in the MOO literature. We note that vector-valued MOO offers unique benefits that do not exist in linear scalarization. 
Specifically, MGD-type methods for vector-valued MOO dynamically calculate the weights for each objective based on the gradient information in each iteration. 
The dynamic weighting in MGD-type approach adapts much better to the landscapes of different MOO problems, which enables a much more flexible exploration on the Pareto front. 
In contrast, the linear scalarization method uses fixed or pre-defined weights for each objective.
As a result, linear scalarization methods are limited to identifying the convex hull of the Pareto front \citep{boyd2004convex,ehrgott2005multicriteria}, whereas (stochastic) multi-gradient methods, including our proposed VR-based algorithms, have the capability to uncover the Pareto front. 

In this paper, we focus on MOO problems in two settings: (i) non-convex MOO and (ii) strongly convex MOO.
Clearly, the non-convex setting is applicable to many learning problems in practice (e.g., neural network models).
The strongly convex setting is also interesting due to many applications in practice (e.g., linear models with quadratic regularizations).

Next, to introduce the notion of sample complexity in MOO, we first need the following definitions for the non-convex and strongly convex settings, respectively.
\begin{defn}[$\epsilon$-Stationarity (Nonconvex Setting)] \label{def:stationary}
A solution $\x$ is $\epsilon$-stationary in MOO problem if the common descent direction at $\x$ satisfies the following condition: $\min_{\boldsymbol{\lambda} \in C} \mathbb{E} \| \boldsymbol{\lambda}^{\top} \nabla \F(\x) \|^2 \leq \epsilon$
in non-convex MOO problems, where $C \triangleq \{ \y \in [0, 1]^S, \sum_{s \in [S]} y_s = 1 \}$.
\end{defn}

\begin{defn}[$\epsilon$-Optimality (Strongly-Convex Setting)]
\label{def:optimality}
In the strongly-convex setting, a solution $\x$ is $\epsilon$-optimal if $\mathbb{E}[\|\x-\x^*\|^2]\leq \epsilon$ in MOO problems, where $\x^*$ is a Pareto-optimal solution of Problem~(\ref{eq: moo}).
\end{defn}

With the above definitions, we are now in a position to define the concept of sample complexity in MOO as follows:

\begin{defn} [Sample Complexity] The sample complexity in MOO is defined as the total number of incremental first-order oracle (IFO) calls required by a MOO algorithm to converge to an $\epsilon$-stationary (or $\epsilon$-optimal in the strongly convex setting) point, where one IFO call evaluates the multi-gradient $\nabla_{\mathbf{x}} f_{sj}(\mathbf{x};\xi_{sj})$ for all tasks $s$. \end{defn}

\section{The \alg Algorithm Family} \label{sec: alg}

In this section, we first present the basic version of the \alg algorithm in Section~\ref{subsec:stimulus}, which is followed by its momentum and adaptive-batching variants in Sections~\ref{subsec:stimulus-m} and \ref{subsec:stimulusp}, respectively.

\subsection{The \alg Algorithm} \label{subsec:stimulus}

Our \alg algorithm is presented in Algorithm~\ref{alg}, where we propose a new variance-reduced (VR) multi-gradient estimator. 
It can be seen from Algorithm~\ref{alg} that our proposed VR approach has a double-loop structure, where the inner loop is of length $q>0$.
More specifically, different from MGD where a full multi-gradient direction $\mathbf{u}_t^s = \nabla f_s(\x_{t})$, $\forall s \in [S]$ is evaluated in all iterations, our \alg algorithm only evaluates a full multi-gradient every $q$ iterations (i.e., $\mathrm{mod}(t,q)=0$).
For all other iterations $t$ with $\mathrm{mod}(t,q)\ne 0$, our \alg algorithm uses a {\em stochastic} multi-gradient estimator $\bu_t^s$ based on a mini-batch $\mathcal{A}$ with a recursive correction term as follows:
\begin{small}
\begin{align}\label{vr}
\mathbf{u}_t^s = \mathbf{u}_{t-1}^s &+ \frac{1}{|\mathcal{A}|} \sum_{j\in \mathcal{A}}( \nabla f_{sj} (\x_{t};\xi_{sj}  )\notag\\&- \nabla f_{sj} (\x_{t-1};\xi_{sj} ) ), \text{for all }s \in [S].
\end{align}
\end{small}

Eq.~\eqref{vr} shows that the estimator is constructed iteratively based on information from $\x_{t-1}$ and $\bu_{t-1}^s$, both of which are obtained from the previous update. 
We will show later in Section~\ref{sec: convergence} that, thanks to the $q$-periodic full multi-gradients and the recursive correction terms, \alg is
able to achieve a convergence rate of $\mathcal{O}(1/T)$. Moreover, due to the stochastic subsampling in mini-batch $\mathcal{A}$, \alg has a lower sample complexity than MGD. 
In \algns, the update rule for parameters in $\x$ is written as:
$
\x_{t+1} = \x_{t} - \eta \bd_t,
$
where $\eta$ is the learning rate.
Here, the direction $\bd_t$ is defined as $\bd_t := \sum_{s \in [S]} \lambda_{t}^{s} \mathbf{u}_{t}^s$, where the $\lambda_t^s$-values are obtained by solving the following quadratic optimization problem:
\begin{small}
\begin{align}\label{mgda} 
     \min_{\lambda_t^s\geq 0} \Big \|  \sum_{s \in [S]} \lambda_{t}^s \mathbf{u}_{t}^s \Big\|^2, \,\,
    \mathrm{s.t.} \,\, \sum_{s \in [S]} \lambda_{t}^s = 1.    
\end{align}
\end{small}
%
The iterative update in Eqs.~\eqref{mgda} follows the same token as in the MGDA algorithm \citep{mukai1980algorithms,sener2018multi,lin2019pareto,fliege2019complexity}.

\begin{algorithm}[htbp]
\caption{\alg algorithm and its variants.}
\label{alg} 
\begin{algorithmic}[1] 
\REQUIRE Initial point $\x_0$, parameters $T$, $q$.
\STATE Initialize: Choose $\x_0$.
\FOR {$t = 0, 1, \ldots, T$}
    \IF {$\mathrm{mod}(t, q) = 0$}
        \IF {\alg or \algmns}
            \STATE Compute: $\mathbf{u}_t^s\!\!=\!\!\frac{1}{n}\sum_{j=1}^n \nabla f_{sj} (\x_{t};\xi_{sj}  ),\! \forall  s \!\in\! [S].$
        \ENDIF
        \IF {\algp or \algmp}
            \STATE Compute: $\mathbf{u}_t^s$ as in Eq.~\eqref{STIMULUSP1}.
        \ENDIF
    \ELSE
        \STATE  Compute $\mathbf{u}_t^s$ as in Eq.~\eqref{vr}.
    \ENDIF
     \STATE  Compute $\boldsymbol{\lambda}_t^* \in [0, 1]^S$ by solving Eq. \eqref{mgda}.
    \STATE Compute: $\bd_t = \sum_{s \in [S]} \lambda_{t}^{s,*} \mathbf{u}_{t}^s$.
    \IF {\alg or \algpns}
        \STATE Update: $\x_{t+1} = \x_{t} - \eta \bd_t$.
    \ENDIF 
    \IF {\algm or \algmpns}
        \STATE Update: $\x_{t+1} = \x_{t} + \alpha(\x_{t}-\x_{t-1}) - \eta \bd_t $.
    \ENDIF 
\ENDFOR
\end{algorithmic}
\end{algorithm}

\subsection{The \algm Algorithm} \label{subsec:stimulus-m}

Although it can be shown that \alg achieves a theoretical $\mathcal{O}(1/T)$ convergence rate, it could be sensitive to the choice of learning rate and suffer from similar oscillation issues in practice as gradient-descent-type methods do in single-objective optimization when some objectives are ill-conditioned.

To further improve the empirical performance of \algns, we now propose a momentum-assisted enhancement for \alg called \algmns. 
The idea behind \algm is to take into account the past trajectories to smooth the update direction. 
Specifically, in addition to the combined iterative update as in $
\x_{t+1} = \x_{t} - \eta \bd_t
$ and \eqref{mgda}, the update rule in \algm incorporates an $\alpha$-parameterized momentum term as follows:
\begin{small}
\begin{align}\label{vrm_update}
\x_{t+1}   =   \x_{t}   -   \eta \bd_t   +   \underbrace{\alpha(\x_{t}  -  \x_{t-1})}_{\mathrm{Momentum}}, \forall s   \in   [S],
\end{align}
\end{small}
where $\alpha   \in   (0, 1)$ is the momentum coefficient.

\subsection{\algpns/\algmp Algorithms} \label{subsec:stimulusp}

Note that in both \alg and \algmns, one still needs to evaluate a full multi-gradient every $q$ iteration, which remains computationally demanding in the large data regime. 
Moreover, if the objectives are in an expectation or ``online'' form rather than the finite-sum setting, it is infeasible to compute a full multi-gradient.
To address these limitations, we propose two {\em adaptive-batching} enhanced versions for \alg and \algm called \algp and \algmpns, respectively. 
Specifically, rather than using a $q$-periodic full multi-gradient $\mathbf{u}_t^s = \nabla f_s(\x_{t})=\frac{1}{n}\sum_{j=1}^n \nabla f_{sj} (\x_{t};\xi_{sj}  )$, $\forall s \in [S]$, in iteration $t$ with $\mathrm{mod}(t, q) = 0$, we utilize an adaptive-batching stochastic multi-gradient as follows:
\begin{small}
\begin{align}\label{STIMULUSP1}
\mathbf{u}_t^s = \frac{1}{|\mathcal{N}_s|} \sum_{j\in \mathcal{N}_s}\nabla f_{sj}(\mathbf{x}_{t};\xi_{sj} ), \quad \forall s \in [S],
\end{align}
\end{small}
where $\mathcal{N}_s$ is an $\epsilon$-adaptive batch sampled from the dataset uniformly at random with size:
\begin{small}
\begin{align}\label{STIMULUSP2}
|\mathcal{N}_s| = \min \left\{ c_\gamma \sigma^2\gamma_{t}^{-1}, c_\epsilon \sigma^2 \epsilon^{-1}, n\right\}.
\end{align}
\end{small}
We choose constants $c_\gamma \geq 8$, $c_{\epsilon}\geq \eta$ in non-convex case and $c_{\gamma}\geq \frac{8\mu}{\eta}, c_{\epsilon}\geq \frac{\mu}{2}$ in strongly-convex case (see detailed discussions in Section~\ref{sec: convergence}). 
The $\sigma^2$ represents the variance bound of stochastic gradient norms (cf. Assumption.~\ref{ass3}). 
In \algp, we choose 
 $\gamma_{t+1} = \sum_{i = (n_k-1) q}^t \frac{\|\bd_i\|^2}{q},$ while in the momentum based algorithm \algmpns, we choose $\gamma_{t+1} =\sum_{i = (n_k-1) q}^t \|\alpha^{(t-i)}\bd_{i}\|^2/q. $ The term $\gamma_{t+1} $ offers further refinement to improve convergence.


\section{Pareto Stationarity Convergence Analysis} \label{sec: convergence}

In this section, we theoretically analyze the Pareto stationarity convergence of our \alg algorithms in non-convex and strongly convex settings, beginning with two necessary assumptions.

\begin{assump}[$L$-Lipschitz Smoothness] \label{assump: smooth}
    There exists a constant $L>0$ such that $\| \nabla f_s(\x) - \nabla f_s(\y) \| \leq L \| \x - \y \|, \forall \x, \y \in \mathbb{R}^d$, $\forall s \in [S]$.
\end{assump}

\begin{assump}[Bounded Variance]	\label{ass3}
There exists a constant	$\sigma>0$ such that for all $\x\in \mathbb{R}^d$, $\mathbb{E}\|  \nabla_{\x}f_s(\x;\xi)- \nabla_{\x}f_s(\x)\|^2 \leq \sigma^2$, $\forall s\in S.$
\end{assump}	
%
With these assumptions, we are now in a position to discuss the Pareto stationary convergence of the \alg family.

\subsection{Pareto-Stationarity Convergence of \algns} \label{subsec: STIMULUS}

\textbf{1)~\algns: The Non-convex Setting.} First, we show that the basic \alg algorithm achieves an $\mathcal{O}(1/T)$ convergence rate for non-convex MOO problems in the following theorem.
Note that this result matches that of the deterministic MGD method.

\begin{restatable}[\alg for Non-convex MOO]{theorem}{STIMULUS_NonC}
\label{thm:STIMULUS_nonC}
Under Assumption~\ref{assump: smooth}, let $\eta \leq  \frac{1}{2L}$, if at least one objective function $f_s(\cdot)$, $s \in [S]$ is bounded from below by $f_s^{\min}$, then the sequence $\{\x_t \}$ output by \alg satisfies: $\frac{1}{T}\sum_{t=0}^{T-1}\min_{\boldsymbol{\lambda} \in C} \mathbb{E} \| \boldsymbol{\lambda}^{\top} \nabla \F(\x_t) \|^2  =\mathcal{O}(1/T).$
\end{restatable}

Following from Theorem.~\ref{thm:STIMULUS_nonC}, we immediately have the following sample complexity for the \alg algorithm by choosing $ q = |\mathcal{A}|=\lceil\sqrt{n}\rceil$: 

\begin{restatable}[Sample Complexity of \algns]{corollary}{STIMULUS_NCRate}
\label{cor:STIMULUS_NC}
By choosing $\eta \leq  \frac{1}{2L}, q = |\mathcal{A}|=\lceil\sqrt{n}\rceil$, the overall sample complexity of  \alg for finding an $\epsilon$-stationary point for non-convex MOO problems is $\mathcal{O}\left(\sqrt{n} \epsilon^{-1}+n\right)$. 
\end{restatable}

Several interesting remarks regarding Theorem~\ref{thm:STIMULUS_nonC} and Corollary~\ref{cor:STIMULUS_NC} are in order:
%
{\bf 1) } Our proof of \algns's Pareto-stationarity convergence only relies on standard assumptions commonly used in first-order optimization techniques. 
This is in stark contrast to prior research, where unconventional and hard-to-verify assumptions were required (e.g., an assumption on the convergence of $\x$-sequence is used in~\cite{fliege2019complexity}).
{\bf 2) } 
While both MGD and our methods share the same $\mathcal{O}(1/T)$ convergence rate, \alg enjoys a substantially lower sample complexity than MGD. 
More specifically, the sample complexity of \alg is reduced by a factor of $\sqrt{n}$ when compared to MGD. 
This becomes particularly advantageous in the ``big data'' regime where $n$ is large. 
 

{\bf 2) \algns: The Strongly Convex Setting.} 
Now, we consider the strongly convex setting, which is more tractable but still of interest in many learning problems in practice (e.g., multi-objective ridge regression).
In the strongly convex setting, we have the following additional assumption:
\begin{assump}[$\mu$-Strongly Convex Function]\label{assump: SC}
    Each objective $f_s(\x)$, $s \in [S]$ is a $\mu$-strongly convex function, i.e., $f_s(\y) \geq f_s(\x) + \nabla f_s(\x) (\y - \x) + \frac{\mu}{2} \| \y - \x \|^2$, $\forall \x,\y$, for some $\mu >0$.
\end{assump}
\begin{assump}\label{assump: add}
For any objective function $f_j$, there exists a positive real number $c_j$ such that for any $\mathbf{x}$ in $\mathbb{R}^n$ the following relation holds
$
f_j(\mathbf{x})-f_j\left(\mathbf{x}^*\right) \geq \frac{c_j}{2}\left\|\mathbf{x}-\mathbf{x}^*\right\|^2 \text { a.s. ;} j \in S .
$
\end{assump}
Assumption \ref{assump: add} asserts that the function value increases at least quadratically as you move away from $\x_*$, ensuring consistent progress towards the optimum. 
It is a reasonable assumption since it is also based on the strong convexity property.
The above assumption has also been adopted in \cite{mercier2018stochastic}.

For strongly convex MOO problems, the next result says that \alg achieves a much stronger expected linear Pareto-optimality convergence performance:

\begin{restatable}[\alg for $\mu$-Strongly Convex MOO] {theorem}{STIMULUS_SC}
\label{thm:STIMULUS_SC}
Under Assumption~\ref{assump: smooth}, \ref{assump: SC}, \ref{assump: add}, let $\eta \leq  \min\{\frac{1}{2}, \frac{1}{2\mu},\frac{1}{8L},\frac{\mu}{64L^2} \}$, $ q = |\mathcal{A}|=\lceil\sqrt{n}\rceil$.
Under Assumptions~\ref{assump: smooth}--\ref{assump: add}, pick $\x_t$ as the final output of \alg with probability $w_t = ( 1 - \frac{3\mu \eta }{4})^{1-t}$.
Then, we have $\mathbb{E}\|\x_t-\x^*\|^2 \leq \| \x_0 - \x^* \|^2 \mu \exp( - \frac{3\eta \mu T}{4}).$
\end{restatable}

Further, Theorem~\ref{thm:STIMULUS_SC} immediately implies following with logarithmic sample complexity (in terms of $\epsilon$) \alg with a proper choice of learning rate and $q = |\mathcal{A}|=\lceil\sqrt{n}\rceil$. 
\begin{restatable}[Sample Complexity of \algns]{corollary}{vr_mooSCRate}
\label{cor:STIMULUS_SC}
By choosing $\eta \leq  \min\{\frac{1}{2}, \frac{1}{2\mu},\frac{1}{8L},\frac{\mu}{64L^2}\}, q = |\mathcal{A}|=\lceil\sqrt{n}\rceil\}$, the overall sample complexity of \alg for solving strongly convex MOO is $\mathcal{O}\left(n+ \sqrt{n} \ln ({\mu/\epsilon})\right)$.
\end{restatable}

There are also several interesting insights from Theorem~\ref{thm:STIMULUS_SC} and Corollary~\ref{cor:STIMULUS_SC} regarding \algns's performance for solving strongly convex MOO problems: 
{\bf 1)} \alg achieves an expected linear convergence rate of $\mathcal{O}(\mu \exp(-\mu T))$. 
Interestingly, this convergence rate matches that of MGD for strongly convex MOO problems as well as gradient descent for strongly convex single-objective optimization. 
%
{\bf 2)} Another interesting feature of \alg for strongly convex MOO stems from its use of randomly selected outputs $\x_t$ along with associated weights $w_t$ from the trajectory of $\x_t$, which is inspired by the similar idea for stochastic gradient descent (SGD)~\citep{ghadimi2013stochastic}. 
Note that, for implementation in practice, one does not need to store all $\x_t$-values.
Instead, the algorithm can be implemented by using a random clock for stopping~\citep{ghadimi2013stochastic}.

\subsection{Pareto Stationarity Convergence of \algmns} \label{subsec: STIMULUS_m}

Next, we turn our attention to the Pareto stationarity convergence of the \algm algorithm.
Again, we analyze \algm in non-convex and strongly convex settings:

\begin{restatable}[\algm for Non-convex MOO] {theorem}{STIMULUS_M_NonC}
\label{STIMULUSM_NonC}
Let $\eta_t = \eta \leq \min\{ \frac{1}{2L}, \frac{1}{2}\}, q = |\mathcal{A}|=\lceil\sqrt{n}\rceil$.
Under Assumptions~\ref{assump: smooth}, if at least one objective function $f_s(\cdot)$, $s \in [S]$, is bounded from below by $f_s^{\min}$, then the sequence $\{\x_t \}$ output by \algm satisfies $\frac{1}{T}\sum_{t=0}^{T-1}\min_{\boldsymbol{\lambda} \in C} \mathbb{E} \| \boldsymbol{\lambda}^{\top} \nabla \F(\x_t) \|^2  =\mathcal{O}(\frac{1}{T}).$
\end{restatable}

Similar to the basic \alg algorithm, by choosing the appropriate learning rate and inner loop length parameters, we immediately have the following sample complexity result for \algm for solving non-convex MOO problems:

\begin{restatable}[Sample Complexity of \algmns]{corollary}{STIMULUS_M_NCRate}
\label{STIMULUS_M_NCRate}
By choosing $\eta_t = \eta \leq \min\{ \frac{1}{2L}, \frac{1}{2}\}, q = |\mathcal{A}|=\lceil\sqrt{n}\rceil$. The overall sample complexity of  \algm under non-convex objective functions is $\mathcal{O}\left(\sqrt{n} \epsilon^{-1}+n\right)$.
\end{restatable}

The next two results state the Pareto optimality and sample complexity results for \algmns:
\begin{restatable}[\algm for $\mu$-Strongly Convex MOO] {theorem}{STIMULUSm_SC}
\label{thm:STIMULUSm_SC}
Let $\eta \leq  \min\{\frac{1}{2},\frac{1}{2\mu},\frac{1}{8L},\frac{\mu}{64L^2} \}, q = |\mathcal{A}|=\lceil\sqrt{n}\rceil$.
Under Assumption~\ref{assump: smooth}, \ref{assump: SC}, \ref{assump: add}, pick $\x_t$ as the final output of \algm with probability $w_t = ( 1 - \frac{3\mu \eta }{4})^{1-t}$.
Then, we have
$\mathbb{E}\|\x_t-\x^*\|^2 \leq \| \x_0 - \x_* \|^2 \mu \exp( - \frac{3\eta \mu T}{4}).$
\end{restatable}

\begin{restatable}[Sample Complexity of \algmns]{corollary}{vrm_mooSCRate}
\label{cor:STIMULUSm_SC}
By choosing $\eta \leq  \min\{\frac{1}{2},\frac{1}{2\mu},\frac{1}{8L},\frac{\mu}{64L^2} \}, q = |\mathcal{A}|=\lceil\sqrt{n}\rceil$, the overall sample complexity of \algm for solving strongly convex MOO is $\mathcal{O}\left(n+ \sqrt{n} \ln ({\mu/\epsilon})\right)$.
\end{restatable}

We remark that the convergence rate upper bound of \algm is the same as that in Theorem~\ref{thm:STIMULUS_SC}, which suggests a potentially loose convergence upper bound in Theorem~\ref{thm:STIMULUSm_SC} due to the technicality and intricacies in analyzing momentum-based stochastic multi-gradient algorithms for solving non-convex MOO problems. 
Yet, we note that even this potentially loose convergence rate upper bound in Theorem~\ref{thm:STIMULUSm_SC} already suffices to establish a linear convergence rate for \algm in solving strongly convex MOO problems.
Moreover, we will show later in Section~\ref{sec:exp} that this momentum-assisted method significantly accelerates the empirical convergence speed performance.
It is also worth noting that there are two key differences in the proofs of Theorem~\ref{STIMULUSM_NonC} and \ref{cor:STIMULUSm_SC} compared to those of the momentum-based stochastic gradient algorithm for single-objective non-convex optimization: 1) our proof exploits the martingale structure of the $\bu_t^s$. 
This enables us to tightly bound the mean-square error term $\mathbb{E}\left\|\nabla f_s\left(\x_t\right)-\bu_t^s\right\|^2$ under the momentum scheme. 
In contrast, in the traditional analysis of stochastic algorithms with momentum, this error term corresponds to the variance of the stochastic estimator and is typically assumed to be bounded by a universal constant. 
2) Our proof requires careful manipulation of the bounding strategy to effectively handle the accumulation of the mean-square error $\mathbb{E}\left\|\nabla f_s\left(\x_k\right)-\bu_t^s\right\|^2$ over the entire optimization trajectory in non-convex MOO. 

\begin{figure*}[t!]
    \centering
    \subfigure[Training loss convergence in terms of iterations.]{
        \includegraphics[width=0.24\textwidth]{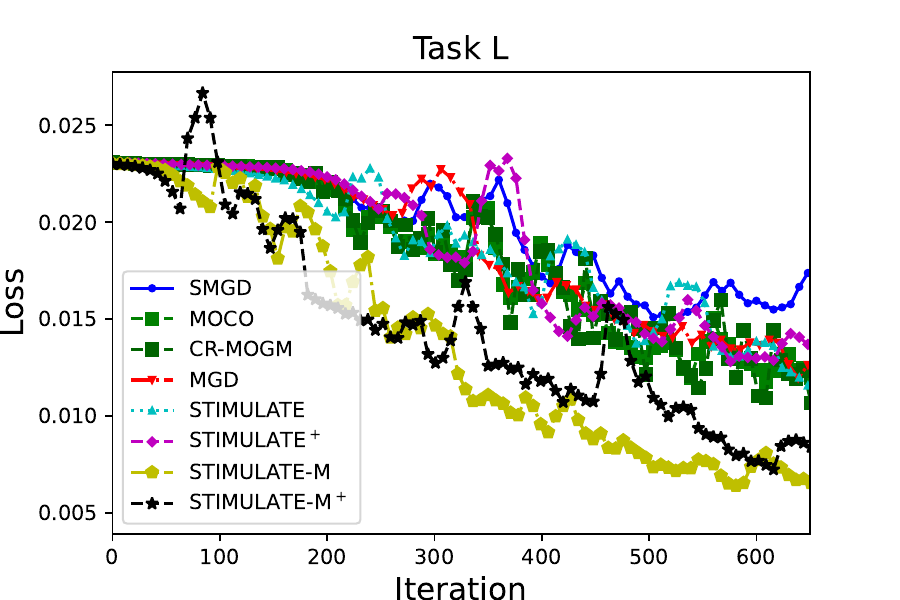}
        \includegraphics[width=0.24\textwidth]{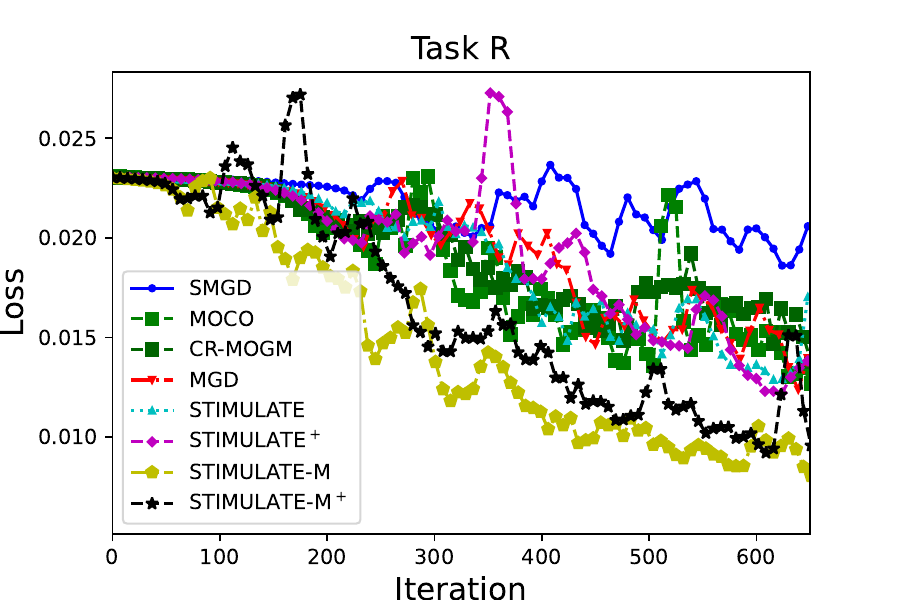}
    }
    \hfill
    \subfigure[Training loss convergence in terms of samples.]{
        \includegraphics[width=0.24\textwidth]{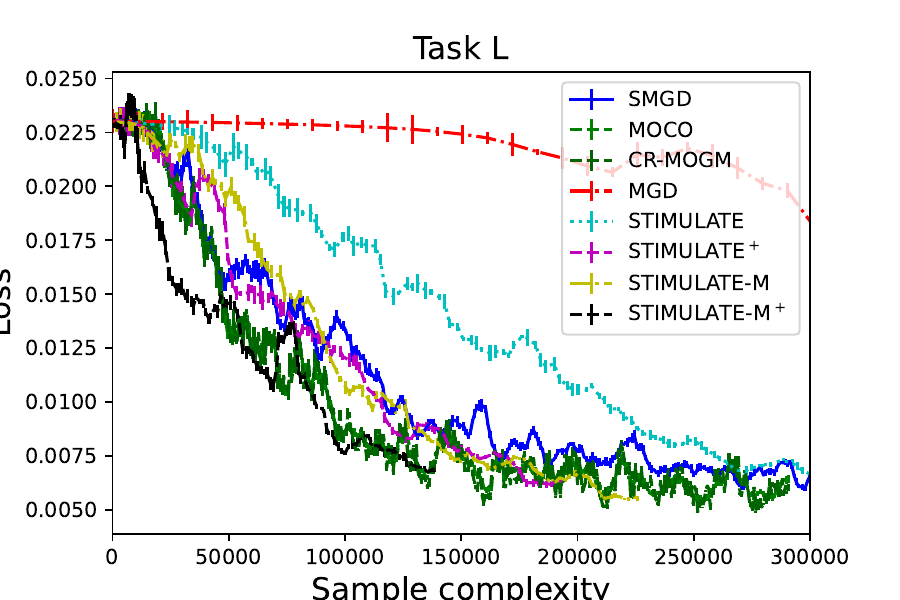}
        \includegraphics[width=0.24\textwidth]{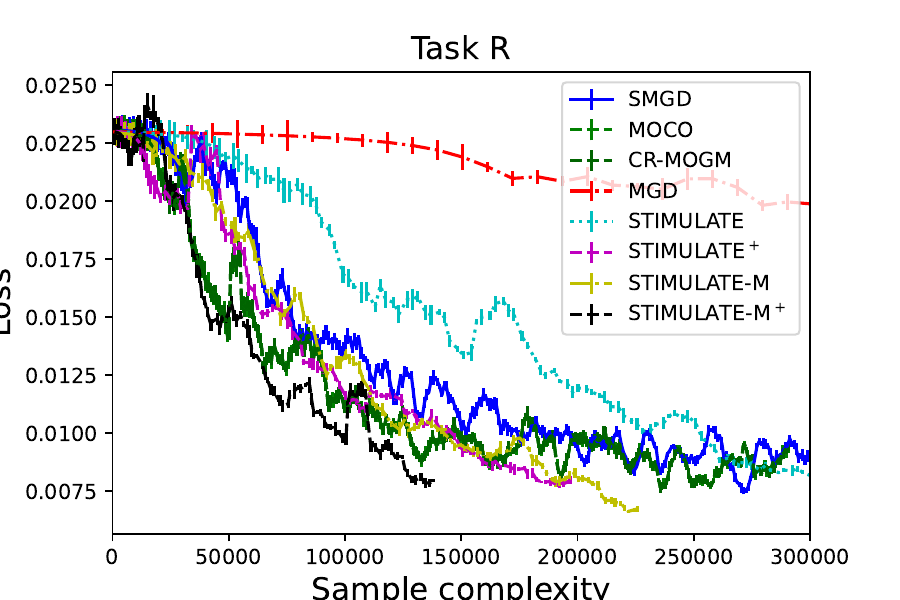}
    }
    \caption{Training loss convergence comparisons between different MOO algorithms.}
    \label{fig:compare_mnist}
\end{figure*}

\subsection{Pareto Stationarity Convergence Results of \algp and \algmpns} \label{subsec: STIMULUSp_m}


Next, we present the Pareto stationarity convergence and the associated sample complexity results of the \algpns/\algmp algorithms for non-convex MOO as follows:

\begin{restatable}[\algpns/\algmpns]{theorem}{STIMULUSm_P_NonC}
\label{thm:STIMULUSmp_nonC}
Let $\eta \leq \min\{ \frac{1}{4L}, \frac{1}{2}\},q = |\mathcal{A}|=\lceil\sqrt{n}\rceil$. 
By choosing $c_\gamma$ and $c_\epsilon$ as such that $c_\gamma \geq 8$, and $c_{\epsilon}\geq \eta$,
under Assumptions~\ref{assump: smooth} and \ref{ass3}, if at least one function $f_s(\cdot)$, $s \in [S]$ is bounded from below by $f_s^{\min}$, then the sequence $\{\x_t \}$ output by \algpns/\algmp satisfies:
$\frac{1}{T}\sum_{t=0}^{T-1}\min_{\boldsymbol{\lambda} \in C} \mathbb{E} \| \boldsymbol{\lambda}^{\top} \nabla \F(\x_t) \|^2   =\mathcal{O}(\frac{1}{T}).$
\end{restatable}

\begin{restatable}[Sample Complexity]{corollary}{STIMULUS_P_NCRate}
\label{STIMULUS_Mp_NCRate}
By choosing $\eta \leq \min\{ \frac{1}{4L}, \frac{1}{2}\},q = |\mathcal{A}|=\lceil\sqrt{n}\rceil$, $c_\gamma \geq 8$, and $c_{\epsilon}\geq \eta$. The overall sample complexity of  \algpns/ \algmp under non-convex objective functions is $\mathcal{O}\left(\sqrt{n} \epsilon^{-1}+n\right)$.
\end{restatable}


\begin{restatable}[\algpns/\algmpns] {theorem}{STIMULUSP_SC}
\label{thm:STIMULUSP_SC}
Let $\eta \leq \min\{\frac{1}{2},\frac{1}{2\mu},\frac{1}{8L},\frac{\mu}{64L^2} \},c_{\gamma}\geq \frac{8\mu}{\eta}, c_{\epsilon}\geq \frac{\mu}{2}, q = |\mathcal{A}|=\lceil\sqrt{n}\rceil$.
Under Assumptions~\ref{assump: smooth}-~\ref{assump: add}, pick $\x_t$ as the final output of the \algpns/\algmp algorithm with weights $w_t = ( 1 - \frac{3\mu \eta }{4})^{1-t}$.
Then, it holds that $\mathbb{E}\|\x_t-\x^*\|^2 \leq \| \x_0 - \x_* \|^2 \mu \exp( - \frac{3\eta \mu T}{4}).$

\end{restatable}

\begin{restatable}[Sample Complexity]{corollary}{vr_moomSCRate}
\label{cor:STIMULUSmp_SC}
By choosing $\eta \leq \min\{\frac{1}{2},\frac{1}{2\mu},\frac{1}{8L},\frac{\mu}{64L^2} \},c_{\gamma}\geq \frac{8\mu}{\eta}, c_{\epsilon}\geq \frac{\mu}{2}, q = |\mathcal{A}|=\lceil\sqrt{n}\rceil$,
the overall sample complexity of  \algpns/ \algmp for solving strongly-convex MOO is $\mathcal{O}\left(n+ \sqrt{n} \ln ({\mu/\epsilon})\right)$.
\end{restatable}

We note that, although the theoretical sample complexity bounds of \algpns/ \algmp are the same as those of \algns/ \algmns, respectively, the fact that \algp and \algmp do not need full multi-gradient evaluations implies that \algns/ \algm use significantly fewer samples than \algns/ \algm in the large dataset regime. 
Our experimental results in the next section will also empirically confirm this.


\section{Experimental Results} \label{sec:exp}

In this section, we conduct numerical experiments to validate our \alg algorithm family, focusing on non-convex MOO problems, while results for strongly convex and 8-objective MOO experiments are in the appendix.


\textbf{1) Two-Objective Experiments on the MultiMNIST Dataset:}
First, we test the convergence performance of our \alg using the ``MultiMNIST'' dataset~\citep{sabour2017dynamic}, which is a multi-task learning version of the MNIST dataset \citep{lecun2010mnist} from LIBSVM repository. 
Specifically, MultiMNIST converts the hand-written classification problem in MNIST into a two-task problem, where the two tasks are task ``L'' (to categorize the top-left digit) and task ``R'' (to classify the bottom-right digit).
The goal is to classify the images of different tasks. 
We compare our \alg algorithms with MGD, SMGD, CR-MOGM, and MOCO.
All algorithms use the same randomly generated initial point. 
The learning rates are chosen as $\eta=0.3,\alpha=0.5$,  constant $c=c_{\gamma}=c_{\epsilon} = 32$ and solution accuracy $\epsilon = 10^{-3}$.
The batch-size for MOCO, CR-MOGM and SMGD is $96$.
The full batch size for MGD is $1024$, and the inner loop batch-size $|\mathcal{N}_s|$ for \algns, \algmns, \algpns, \algmpns is $96$. 
As shown in Fig.~\ref{fig:compare_mnist}(a), SMGD exhibits the slowest convergence speed, while MOCO has a slightly faster convergence.
MGD and our \alg algorithms have comparable performances. 
The \algm/\algmp algorithms converge faster than MGD, \alg, and \algp, primarily due to the use of momentum. 
Fig.~\ref{fig:compare_mnist}(b) highlights differences in sample complexity. 
MGD suffers the highest sample complexity, while \algp and \algmp demonstrate a more efficient utilization of samples in comparison to \alg and \algmns.
These results are consistent with our theoretical analyses as outlined in Theorems~\ref{thm:STIMULUS_nonC}, \ref{STIMULUSM_NonC}, and \ref{thm:STIMULUSmp_nonC}.

\textbf{2) 40-Objective Experiments with the CelebA Dataset:}

Lastly, we conduct large-scale 40-objective experiments with the CelebA dataset \citep{liu2015deep}, which contains 200K facial images annotated with 40 attributes. 
Each attribute corresponds to a binary classification task, resulting in a 40-objective  problem.

\begin{wrapfigure}{r}{0.27\textwidth}
  \includegraphics[width=0.28\textwidth]{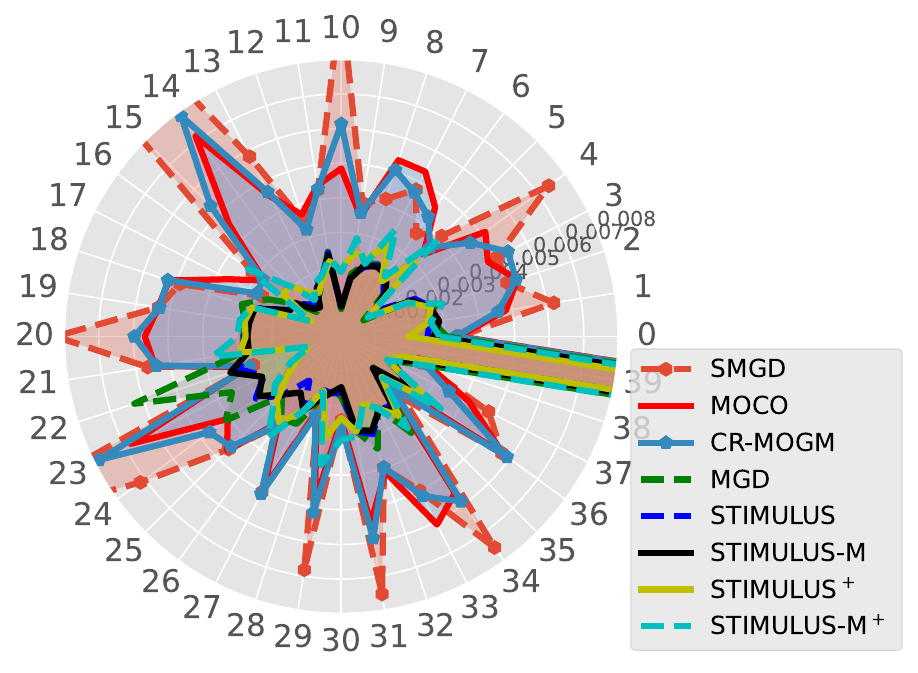}
\caption{Training loss convergence comparison (40-task).}
\label{fig_compare_40tasks}
\end{wrapfigure}

We use a ResNet-18 \cite{he2016deep} model without the final layer for each attribute, and we attach a linear layer to each attribute for classification. 
In this experiment, we set $\eta=0.0005,\alpha=0.01$, the full batch size for MGD is $1024$, and the batch size for SMGD, CR-MOGM and MOCO and the inner loop batch size $|\mathcal{N}_s|$ for \algns, \algmns, \algpns, \algmpns is $32$. 
As shown in Fig.~\ref{fig_compare_40tasks}, MGD, \algns, \algmns, \algpns, and \algmpns significantly outperform SMGD, CR-MOGM and MOCO in terms of training loss. 
Also, we would like to note that \algp and \algmp consume fewer sample (approximately 11,000) samples compared to \alg and \algm, which consume approximately 13,120 samples, and MGD, which consumes roughly 102,400 samples. 
These results are consistent with our theoretical results in Theorems~\ref{thm:STIMULUS_nonC}, \ref{STIMULUSM_NonC}, and \ref{thm:STIMULUSmp_nonC}.

\section{Conclusion} 
\label{sec: conclusion}

In this paper, we proposed \algns, a new variance-reduction-based stochastic multi-gradient-based algorithm to achieve fast convergence and low sample complexity multi-objective optimization (MOO). 
We analyze its Pareto stationarity convergence and sample complexity under non-convex and strongly convex settings. To enhance empirical convergence, we propose \algm, which incorporates momentum. To reduce the periodic full multi-gradient evaluation in \alg and \algmns, we introduce adaptive batching versions, \algpns/\algmpns, with theoretical performance analysis. Overall, our \alg algorithm family advances MOO algorithm design and analysis.

\section{Acknowledgement}
This work is supported in part by NSF grants CAREER CNS-2110259, CNS-2112471, IIS-2324052, DARPA YFA D24AP00265, DARPA HR0011-25-2-0019, ONR grant N00014-24-1-2729, and AFRL grant PGSC-SC-111374-19s.

\bibliography{MOO, DNN}

\appendix
\allowdisplaybreaks


\newpage
 \onecolumn
\section{Proof of convergence of \algns} \label{appdx:VR_MOO_nonconvex}

 \begin{table}[htbp]
     \centering
    \caption{List of key notation.}
    \label{tab:list of notations}
    {
    \begin{tabular}{l|l}
        \hline
        Notation      & Definition                               \\ \hline
        $ n $         & Total number of samples per task         \\ \hline
        $ s $         & Objective/task index          \\ \hline
        $ S $         & Total number of objectives/tasks       \\ \hline   
        $ t $         & Iteration number index    \\ \hline
        $ T $         & Total number of iterations     \\ \hline
        $ \x \in \mathbb{R}^d $         & Model parameters in Problem~\eqref{eq: moo}       \\ \hline
        $ \x_* \in \mathbb{R}^d $         & A pareto optimal solution of Problem~\eqref{eq: moo}       \\ \hline
        $ \eta $         & The learning rate     \\ \hline
        $ \alpha$         & The momentum constant   \\ \hline   
$\epsilon$ & 
      Stationarity error in  Def. \ref{def:stationary}\\ \hline   
      $\mu$ & 
      Strongly-convex constant in  Assumption \ref{assump: SC}\\ \hline  
    \end{tabular}
    }
\end{table}

For clarity of notation, we drop $*$ for $\lambda$, that is, we use $\lambda_t^s$ to represent the solution of quadratic problem for task $s$ in the $t$-th round.

\begin{lem} \label{lem:bounded1}
Let Assumption 1 hold. The gradient estimator $\bu_t^s$ satisfies for all $(n_t -1)q + 1 \leq t \leq  n_t q - 1$:
    \begin{align}
 \mathbb{E}_t\| \nabla f_s(\x_t) - \bu_t^s \|^2 \leq \frac{L^2}{|\mathcal{A}|} \sum_{i=\left(n_t-1\right) q}^t \mathbb{E}\|\x_{i+1}-\x_{i}\|^2 +\mathbb{E}_t\|\nabla f_s(\x_{\left(n_t-1\right) q}) - \bu_{\left(n_t-1\right) q}^s \|^2 .
    \end{align}
\end{lem}

\textbf{Proof of Lemma. \ref{lem:bounded1}.}
\begin{proof}
From Lemma 1 in \cite{fang2018spider}, we have
 \begin{align}
 \mathbb{E}_t\| \nabla f_s(\x_t)  &- \bu_t^s \|^2 \stackrel{(a)}{=}   \mathbb{E}_t\| \nabla f_s(\mathbf{x}_{t-1}) - \mathbf{u}_{t-1}^s \|^2 \notag\\+& \mathbb{E}_t \| \frac{1}{|\mathcal{A}|} \sum_{j\in \mathcal{A}}\left( \nabla f_{sj} (\mathbf{x}_{t};\xi_{sj}  ) - \nabla f_{sj} (\mathbf{x}_{t-1};\xi_{sj} )  +\nabla f_s(\mathbf{x}_{t-1})-\nabla f_s(\mathbf{x}_{t}) \right) \|^2\notag\\
\stackrel{(b)}{\le} & \mathbb{E}_t\| \nabla f_s(\mathbf{x}_{\left(n_t-1\right) q}) - \mathbf{u}_{\left(n_t-1\right) q}^s \|^2 +  L^2 \sum_{i=\left(n_t-1\right) q}^t\frac{1}{|\mathcal{A}|} \mathbb{E}\|\mathbf{x}_{i+1}-\mathbf{x}_{i}\|^2.
\end{align}

$(a)$ stems from Proposition 1 in \cite{fang2018spider}, where the expectation of the gradient difference is broken down.
$(b)$ leverages Eq. (2.3) from \cite{fang2018spider}, applying a bound based on the Lipschitz continuity of the gradient.

Telescoping over from $\left(n_t-1\right) q+1 \text { to } t \text {, where } t \leq n_t q-1$, we obtain that 
 \begin{align}
 &\mathbb{E}_t\| \nabla f_s(\x_t) - \bu_t^s \|^2  \leq \mathbb{E}_t\| \nabla f_s(\x_{\left(n_t-1\right) q}) - \bu_{\left(n_t-1\right) q}^s \|^2 +  L^2 \sum_{i=\left(n_t-1\right) q}^t\frac{1}{|\mathcal{A}|} \mathbb{E}\|\x_{i+1}-\x_{i}\|^2
     \end{align}

Then, we have
 \begin{align}
\mathbb{E}_t \| \nabla f_s(\x_t) - \bu_t^s \|^2 \leq \frac{L^2}{|\mathcal{A}|} \sum_{i=\left(n_t-1\right) q}^t \mathbb{E}\|\x_{i+1}-\x_{i}\|^2 +\mathbb{E}_t\|\nabla f_s(\x_{\left(n_t-1\right) q}) - \bu_{\left(n_t-1\right) q}^s \|^2 .
    \end{align}
    
\end{proof}

\begin{lem} \label{lem:update2}
For general $L$-smooth functions $\{ f_s, s \in [S] \}$, choose the learning rate $\eta$ s.t. $\eta \leq  \frac{1}{2L}$, the update $\bd_t$ of the algorithm satisfies:
    \begin{align}
   f_s(\x_{t+1})   &\leq  f_s(\x_t) + \frac{\eta}{2}  \| \nabla f_s(\x_t) - \bu_t^s \|^2 -\frac{\eta}{4} \| \bd_t \|^2.
\end{align}
\end{lem}

\textbf{Proof of Lemma. \ref{lem:update2}.}
\begin{proof}

\begin{align}
    f_s(\x_{t+1}) &\stackrel{(a)}{\le}  f_s(\x_t) + \left< \nabla f_s(\x_t), -\eta \bd_t \right> + \frac{1}{2}L \| \eta \bd_t \|^2 \notag\\
    &=  f_s(\x_t) - \eta\left< \nabla f_s(\x_t) - \bu_t^s, \bd_t \right> - \eta \left< \bu_t^s, \bd_t \right> + \frac{1}{2}L \| \eta \bd_t \|^2 \notag\\
    &\stackrel{(b)}{\le} f_s(\x_t) -\eta\left< \nabla f_s(\x_t) - \bu_t^s,  \bd_t \right> - \eta \| \bd_t \|^2 + \frac{1}{2}L \| \eta \bd_t \|^2 \notag\\
    &\stackrel{(c)}{\le}  f_s(\x_t) + \frac{\eta}{2} \notag\| \nabla f_s(\x_t) - \bu_t^s \|^2 + \frac{1}{2} \eta \| \bd_t \|^2 - \eta \| \bd_t \|^2 + \frac{1}{2}L \eta^2 \| \bd_t \|^2 \notag\\
    &= f_s(\x_t) + \frac{\eta}{2}  \| \nabla f_s(\x_t) - \bu_t^s \|^2 - \eta \left( \frac{1}{2}- \frac{1}{2}L \eta \right) \| \bd_t \|^2.
\end{align}
(a) follows from  the objective function $f_s$ is $L$-smooth.  $(b)$ follows from $\left< \bu_t^s, \bd_t \right> \geq \| \bd_t \|^2$ since $\bd_t$ is a general solution in the convex hull of the family of vectors $\{\bu_t^s, s \in [S] \}$ (see Lemma 2.1~\cite{desideri2012multiple}). (c) follows from the triangle inequality.

By setting $\left( \frac{1}{2} - \frac{L}{2} \eta \right) \geq \frac{1}{4}$, that is, $\eta \leq  \frac{1}{2L}$,  we have 
\begin{align}
   f_s(\x_{t+1})   &\leq  f_s(\x_t) + \frac{\eta}{2}  \| \nabla f_s(\x_t) - \bu_t^s \|^2 -\frac{\eta}{4} \| \bd_t \|^2.
\end{align}
\end{proof}

\textbf{Proof of Theorem.~\ref{thm:STIMULUS_nonC}}

\begin{proof}
Taking expectation on both sides of the inequality in Lemma. \ref{lem:update2}, we have

\begin{align}
  &  \mathbb{E}[f_s(\x_{t+1})] \stackrel{(a)}{\le}  \mathbb{E}[f_s(\x_t)] + \frac{\eta}{2}   \mathbb{E}\| \nabla f_s(\x_t) - \bu_t^s \|^2 -\frac{\eta}{4}  \mathbb{E}\| \bd_t \|^2\notag\\&
 \stackrel{(b)}{\le}  \mathbb{E}[f_s(\x_t)] -\frac{\eta}{4}  \mathbb{E}\| \bd_t \|^2+ \mathbb{E}\frac{\eta}{2}  [\frac{L^2}{|\mathcal{A}|} \sum_{i=\left(n_t-1\right) q}^t \mathbb{E}\|\x_{i+1}-\x_{i}\|^2 +\mathbb{E}\|\nabla f_s(\x_{\left(n_t-1\right) q}) - \bu_{\left(n_t-1\right) q}^s \|^2] \notag\\&
\stackrel{(c)}{=}  \mathbb{E}[f_s(\x_t)] -\frac{\eta}{4} \mathbb{E} \| \bd_t \|^2+ \frac{\eta}{2}  [\frac{L^2}{|\mathcal{A}|} \sum_{i=\left(n_t-1\right) q}^t \eta^2 \mathbb{E}\|\bd_i\|^2 ].
\end{align}
$(a)$ follows from Lemma. \ref{lem:update2}. $(b)$ follows from the Lemma. \ref{lem:bounded1}. $(c)$ follows from the update rule of $\x$ as shown in Eq. \eqref{STIMULUSP1} and $\mathbb{E}\|\nabla f_s(\x_{\left(n_t-1\right) q}) - \bu_{\left(n_t-1\right) q}^s \|^2=0$ as shown in Line 5 in our Algorithm. \ref{alg}.

Next, telescoping the above inequality over $t$ from $\left(n_t-1\right) q$ to $t$ where $t \leq n_t q-1$ and noting that for $\left(n_t-1\right) q \leq j \leq n_t q-1, n_j=n_t$, we obtain

\begin{align} \label{eqs16}
  &   \mathbb{E}[f_s(\x_{t+1})]  \notag\\&\leq   \mathbb{E}[f_s(\x_{\left(n_t-1\right) q})] -\frac{\eta}{4} \sum_{j=\left(n_t-1\right) q}^t  \mathbb{E}\| \bd_j \|^2+ \frac{\eta}{2}  [\frac{L^2}{|\mathcal{A}|} \sum_{j=\left(n_t-1\right) q}^t \sum_{i=\left(n_t-1\right) q}^j\eta^2 \mathbb{E}\|\bd_i\|^2 ]  \notag\\& \stackrel{(a)}{\le}   \mathbb{E}[f_s(\x_{\left(n_t-1\right) q})] -\frac{\eta}{4} \sum_{j=\left(n_t-1\right) q}^t  \mathbb{E}\| \bd_j \|^2+ \frac{\eta}{2}  [\frac{L^2}{|\mathcal{A}|} \sum_{j=\left(n_t-1\right) q}^t \sum_{i=\left(n_t-1\right) q}^t\eta^2 \mathbb{E}\|\bd_i\|^2 ] \notag\\& \stackrel{(b)}{\le}   \mathbb{E} [f_s(\x_{\left(n_t-1\right) q})] -\frac{\eta}{4} \sum_{j=\left(n_t-1\right) q}^t \mathbb{E} \| \bd_j \|^2+ \frac{\eta^3 q}{2}  [\frac{L^2}{|\mathcal{A}|} \sum_{j=\left(n_t-1\right) q}^t  \mathbb{E}\|\bd_j\|^2 ] \notag\\&= \mathbb{E}[f_s(\x_{\left(n_t-1\right) q})] -[\frac{\eta}{4}- \frac{\eta^3 q}{2}  \frac{L^2}{|\mathcal{A}|}] \sum_{j=\left(n_t-1\right) q}^t  \mathbb{E}\|\bd_j\|^2.
\end{align}

where $(a)$ extends the summation of the third term from $j$ to $t$, $(b)$ follows from the fact that
$t \leq n_t q - 1$.

We continue the proof by further driving
\begin{align}
  &   \mathbb{E}[f_s(\x_{T})] -   \mathbb{E}[f_s(\x_{0})]  \notag\\&
  =(   \mathbb{E}[f_s(\x_{q})] -   \mathbb{E}[f_s(\x_{0})] ) +  (   \mathbb{E}[f_s(\x_{2q})] -   \mathbb{E}[f_s(\x_{q})] ) +\cdot +  (   \mathbb{E}[f_s(\x_{T})] -   \mathbb{E}[f_s(\x_{(n_T-1)q})] ) 
  \notag\\&\leq  -[\frac{\eta}{4}- \frac{\eta^3 q}{2}  \frac{L^2}{|\mathcal{A}|}] \sum_{t=0}^{T-1}  \mathbb{E}\|\bd_t\|^2
\end{align}

Note that $\mathbb{E} [ f_s\left(\x_{T+1}\right) ]\geq f_s^* \triangleq \inf _{\x \in \mathbb{R}^d} f_s(\x)$. Hence, we have
\begin{align}
  &
 [\frac{\eta}{4}- \frac{\eta^3 q}{2}  \frac{L^2}{|\mathcal{A}|}] \sum_{t=0}^{T-1}  \mathbb{E}\|\bd_t\|^2\leq  [  [f_s(\x_{0})] -  [f_s(\x_{T})]  ]\leq  [  [f_s(\x_{0})] - f_s^*  ].
\end{align}

Based on the parameter setting $q=|\mathcal{A}|=\lceil\sqrt{n}\rceil$, we have 
\begin{align}
  &
 [\frac{\eta}{4}- \frac{\eta^3 L^2 }{2}  ] \sum_{t=0}^{T-1}  \|\bd_t\|^2 \leq   [  [f_s(\x_{0})] -f_s^* ].
\end{align}
Thus, we have
\begin{align}\label{eqs20}
  &
 \frac{1}{T} \sum_{t=0}^{T-1}  \mathbb{E}\|\bd_t\|^2 \leq   \frac{[  [f_s(\x_{0})] -f_s^* ]}{ [\frac{\eta}{4}- \frac{\eta^3 L^2 }{2}  ] T}.
\end{align}

Since $\frac{1}{T} \sum_{t=0}^{T-1}\mathbb{E} \|d_t \|^2$ is just common descent directions.
According to Definition. \ref{def:stationary} shown in the paper, the quantity to our interest is 
$\|\sum_{s \in [S]}\lambda_t^s \nabla f(\mathbf{x})\|^2$. 
\begin{align}\label{eqs21}
  &
 \frac{1}{T} \sum_{t=0}^{T-1}  \mathbb{E}\|\sum_{s\in [S]}\lambda_t^s\nabla f_s(\x_t)\|^2 \notag\\
\stackrel{(a)}{\le}   & \frac{1}{T} \sum_{t=0}^{T-1}  2\mathbb{E}\|\sum_{s\in [S]}\lambda_t^s\nabla f_s(\x_t)-\sum_{s\in [S]}\lambda_t^s\bu_t^s\|^2+  \frac{1}{T} \sum_{t=0}^{T-1}  2\mathbb{E}\|\sum_{s\in [S]}\lambda_t^s \bu_t^s\|^2 \notag\\
\stackrel{(b)}{=} & \frac{1}{T} \sum_{t=0}^{T-1}  2\mathbb{E}\|\sum_{s\in [S]}\lambda_t^s(\nabla f_s(\x_t)-\bu_t^s)\|^2+  \frac{1}{T} \sum_{t=0}^{T-1}  2\mathbb{E}\|\bd_t\|^2 \notag\\
\stackrel{(c)}{\le}  & \frac{1}{T} \sum_{t=0}^{T-1} 2S \sum_{s\in [S]} (\lambda_t^s)^2\mathbb{E}\|(\nabla f_s(\x_t)-\bu_t^s)\|^2+  \frac{1}{T} \sum_{t=0}^{T-1}  2\mathbb{E}\|\bd_t\|^2 \notag\\
\stackrel{(d)}{\le}  & \frac{1}{T} \sum_{t=0}^{T-1} 2S \sum_{s\in [S]} (\lambda_t^s)^2[ \mathbb{E}_t\| \nabla f_s(\x_{\left(n_t-1\right) q}) - \bu_{\left(n_t-1\right) q}^s \|^2 +  L^2 \sum_{i=\left(n_t-1\right) q}^t\frac{1}{|\mathcal{A}|} \mathbb{E}\|\x_{i+1}-\x_{i}\|^2]+  \frac{1}{T} \sum_{t=0}^{T-1}  2\mathbb{E}\|\bd_t\|^2 \notag\\
=& \frac{1}{T} \sum_{t=0}^{T-1} 2S \sum_{s\in [S]} (\lambda_t^s)^2[ \mathbb{E}_t\| \nabla f_s(\x_{\left(n_t-1\right) q}) - \bu_{\left(n_t-1\right) q}^s \|^2]\notag\\& +   2S L^2 \frac{1}{T} \sum_{t=0}^{T-1} \sum_{i=\left(n_t-1\right) q}^t\frac{1}{|\mathcal{A}|} \mathbb{E}\|\x_{i+1}-\x_{i}\|^2 +  \frac{1}{T} \sum_{t=0}^{T-1}  2\mathbb{E}\|\bd_t\|^2  \notag\\
\stackrel{(e)}{\le}& \frac{1}{T} \sum_{t=0}^{T-1} 2S \sum_{s\in [S]} (\lambda_t^s)^2[ \mathbb{E}_t\| \nabla f_s(\x_{\left(n_t-1\right) q}) - \bu_{\left(n_t-1\right) q}^s \|^2]\notag\\& +   2S L^2 \frac{1}{T} \sum_{t=0}^{T-1} \sum_{i=\left(n_t-1\right) q}^{n_t q -1}\frac{1}{|\mathcal{A}|} \mathbb{E}\|\x_{t+1}-\x_{t}\|^2 +  \frac{1}{T} \sum_{t=0}^{T-1}  2\mathbb{E}\|\bd_t\|^2 \notag\\
= &   2S L^2 \frac{1}{T} \sum_{t=0}^{T-1} \frac{q}{|\mathcal{A}|} \mathbb{E}\|\x_{t+1}-\x_{t}\|^2 +  \frac{1}{T} \sum_{t=0}^{T-1}  2\mathbb{E}\|\bd_t\|^2  \notag\\
\stackrel{(f)}{=} &   2S L^2 \eta ^2\frac{1}{T} \sum_{t=0}^{T-1} \mathbb{E}\|\bd_t\|^2 +  \frac{1}{T} \sum_{t=0}^{T-1}  2\mathbb{E}\|\bd_t\|^2 \notag\\
= &  ( 2S L^2 \eta ^2 +2)  \frac{1}{T} \sum_{t=0}^{T-1}  \mathbb{E}\|\bd_t\|^2 
\end{align}
where $(a)$  and $(c)$ hold from the triangle inequality. (b) is because the definition $\bd_t=\sum_{s \in [S]} \lambda_{t}^{s} \mathbf{u}_{t}^s$ as shown in Line 14 in Algorithm. \ref{alg}. $(d)$ follows from the Lemma. \ref{lem:bounded1}. (e) is because $t\leq n_t q -1$. $(f)$ is because we have $q=|\mathcal{A}|=\lceil\sqrt{n}\rceil$.

Then, we can conclude that 
\begin{align}
  &
 \frac{1}{T} \sum_{t=0}^{T-1}  \mathbb{E}\|\sum_{s\in [S]}\lambda_t^s\nabla f_s(\x_t)\|^2 
\stackrel{(a)}{\le}
 ( 2S L^2 \eta ^2  +2)
  \frac{[  [f_s(\x_{0})] -f_s^* ]}{ [\frac{\eta}{4}- \frac{\eta^3 L^2 }{2}  ] T},
\end{align}
where $(a)$ follows from Eqs. \eqref{eqs21} and Eqs. \eqref{eqs20}.

Let $\eta\leq \frac{1}{2L}$, we have

\begin{align}
  &
\frac{1}{T}\sum_{t=0}^{T-1}\min_{\boldsymbol{\lambda} \in C} \mathbb{E} \| \boldsymbol{\lambda}^{\top} \nabla \F(\x_t) \|^2 \leq \frac{1}{T} \sum_{t=0}^{T-1}  \mathbb{E}\|\sum_{s\in [S]}\lambda_t^s\nabla f_s(\x_t)\|^2 \notag\\ \leq&   \frac{ ( 2S L^2 \eta ^2\frac{1}{T}  +2)[  [f_s(\x_{0})] -f_s^*  ]}{ [\frac{\eta}{8} ] T}=\frac{ ( 2S L^2 \eta ^2 +2) \frac{8}{\eta}  [  [f_s(\x_{0})] -f_s^*  ]}{  T} =\mathcal{O}(\frac{1}{T}).
\end{align}

Lastly, to show the sample complexity, the number of samples with $mod(t,q)=0$ can be calculated as: $\lceil \frac{T}{q} \rceil \cdot M$.
	Also, the number of samples with $mod(t,q)\neq0$ can be calculated as $T\cdot|\mathcal{A}|$.
	Hence, the total sample complexity can be calculated as:
	$\lceil \frac{T}{q} \rceil n + T\cdot |\mathcal{A}| \leq  \frac{T+q}{q}n + T\sqrt{n}= T\sqrt{n}+n+T\sqrt{n}=O(n+ \sqrt{n} \epsilon^{-1})$.
	Thus, the overall sample complexity is $\mathcal{O}(n+ \sqrt{n} \epsilon^{-1})$.
	This completes the proof.

\end{proof}

\subsection{Proof of Theorem.~\ref{thm:STIMULUS_SC}}
\begin{proof}

\begin{align}\label{eq24}
    &f_s(\x_{t+1}) \notag\\\leq& f_s(\x_t) + \left< \nabla f_s(\x_t), -\eta \bd_t \right> + \frac{1}{2}L \| \eta \bd_t \|^2 \notag\\
  \stackrel{(a)}{\le} & f_s(\x_*) + \left< \nabla f_s(\x_t), \x_t - \x_* \right> - \frac{\mu}{2} \| \x_t - \x_* \|^2  + \left< \nabla f_s(\x_t), -\eta \bd_t \right> + \frac{1}{2}L \| \eta \bd_t \|^2\notag\\
    = &f_s(\x_*) + \left< \nabla f_s(\x_t), \x_t - \x_* -\eta \bd_t \right> - \frac{\mu}{2} \| \x_t - \x_* \|^2 + \frac{1}{2}L \| \eta \bd_t \|^2\notag\\
    \stackrel{(b)}{\le} & f_s(\x_*) + \left< \nabla f_s(\x_t)-\bu_t^s, \x_t - \x_* -\eta \bd_t \right>+ \left< \bu_t^s, \x_t - \x_* -\eta \bd_t \right> \notag\\&- \frac{\mu}{2} \| \x_t - \x_* \|^2 + \frac{1}{2}L \| \eta \bd_t \|^2\notag\\
    \stackrel{(c)}{\le}  & f_s(\x_*) + \frac{1}{2\delta}\| \nabla f_s(\x_t)-\bu_t^s\|^2+ \frac{\delta}{2}\| \x_t - \x_* -\eta \bd_t\|^2 + \left< \bu_t^s, \x_t - \x_* -\eta \bd_t \right> \notag\\&- \frac{\mu}{2} \| \x_t - \x_* \|^2 + \frac{1}{2}L \| \eta \bd_t \|^2\notag\\
   \stackrel{(d)}{\le}  & f_s(\x_*) + \frac{1}{2\delta}\| \nabla f_s(\x_t)-\bu_t^s\|^2+ \delta\| \x_t - \x_* \|^2+\delta\|\eta \bd_t\|^2 \notag\\&+ \left< \bu_t^s, \x_t - \x_* -\eta \bd_t \right> - \frac{\mu}{2} \| \x_t - \x_* \|^2 + \frac{1}{2}L \| \eta \bd_t \|^2,
\end{align}
the first inequality is due to $L$-smoothness, the second inequality follows from $\mu$-strongly convex. The last two inequality follows from the triangle inequality. 

According to Definition. \ref{def:stationary} shown in the paper, the quantity to our interest is 
$ \sum_{s \in [S]} \lambda_t^{s} \left[ f_s(\x_{t+1}) - f_s(\x_*) \right]  $, then we have

\begin{align} \label{eqs23}
    & \sum_{s \in [S]} \lambda_t^{s} \left[ f_s(\x_{t+1}) - f_s(\x_*) \right]  \notag\\
    \stackrel{(a)}{\le}& \frac{1}{2\delta}  \sum_{s \in [S]} \lambda_t^{s} \| \nabla f_s(\x_t)-\bu_t^s\|^2+ \delta\| \x_t - \x_* \|^2+\delta\|\eta \bd_t\|^2\notag\\& +\left< \sum_{s \in [S]} \lambda_t^{s}\bu_t^s, \x_t - \x_* \right> - \frac{\mu}{2} \| \x_t - \x_* \|^2 + \left< \sum_{s \in [S]} \lambda_t^{s} \bu_t^s, -\eta \bd_t \right> + \frac{1}{2}L \| \eta \bd_t \|^2 \notag\\
    =&\frac{1}{2\delta}  \sum_{s \in [S]} \lambda_t^{s} \| \nabla f_s(\x_t)-\bu_t^s\|^2+ \delta\| \x_t - \x_* \|^2+\delta\|\eta \bd_t\|^2\notag\\& +\left< \sum_{s \in [S]} \lambda_t^{s}\bu_t^s, \x_t - \x_* -\eta \bd_t \right> - \frac{\mu}{2} \| \x_t - \x_* \|^2 + \frac{1}{2}L \| \eta \bd_t \|^2 \notag\\
     \stackrel{(b)}{\le}&\frac{1}{2\delta}  \sum_{s \in [S]} \lambda_t^{s}\| \nabla f_s(\x_t)-\bu_t^s\|^2+ \delta\| \x_t - \x_* \|^2+\delta\|\eta \bd_t\|^2\notag\\& +\left< \bd_t , \x_t - \x_* -\eta \bd_t \right> - \frac{\mu}{2} \| \x_t - \x_* \|^2 + \frac{1}{2}L \| \eta \bd_t \|^2 \notag\\
   = & \left< \bd_t , \x_t - \x_* \right> - \eta \| \bd_t \|^2 - \frac{\mu}{2} \| \x_t - \x_* \|^2 + \frac{1}{2}L \eta^2 \| \bd_t \|^2 \notag\\&+ \frac{1}{2\delta}  \sum_{s \in [S]} \lambda_t^{s}\| \nabla f_s(\x_t)-\bu_t^s\|^2+ \delta\| \x_t - \x_* \|^2+\delta\|\eta \bd_t\|^2\notag\\
    \stackrel{(c)}{\le}&\frac{1}{2 \eta} \left( \| \x_t - \x_* \|^2 - \| \x_{t+1} - \x_* \|^2 \right) - \frac{1}{2} \eta \| \bd_t \|^2 - \frac{\mu}{2} \| \x_t - \x_* \|^2 + \frac{1}{2}L \eta^2 \|\bd_t \|^2 \notag\\
    &+ \frac{4}{\mu}  \sum_{s \in [S]} \lambda_t^{s}\| \nabla f_s(\x_t)-\bu_t^s\|^2+ \frac{\mu}{8}\| \x_t - \x_* \|^2+\frac{\mu}{8}\|\eta \bd_t\|^2\notag\\
     \stackrel{(d)}{\le}&\frac{1}{2 \eta} \left( (1-\frac{3\mu\eta}{4})\| \x_t - \x_* \|^2 - \| \x_{t+1} - \x_* \|^2 \right) - (\frac{1}{2} \eta-\frac{\mu}{8}\eta^2 - \frac{1}{2}L \eta^2 ) \| \bd_t \|^2 \notag\\&
    + \frac{4}{\mu} \sum_{s \in [S]} \lambda_t^{s}\| \nabla f_s(\x_t)-\bu_t^s\|^2\notag\\
    \stackrel{(e)}{\le}&\frac{1}{2 \eta} \left( (1-\frac{3\mu\eta}{4})\| \x_t - \x_* \|^2 - \| \x_{t+1} - \x_* \|^2 \right) - (\frac{1}{2} \eta-\frac{\mu}{8}\eta^2 - \frac{1}{2}L \eta^2 ) \| \bd_t \|^2 \notag\\&
    + \frac{4}{\mu}( \frac{L^2}{|\mathcal{A}|} \sum_{i=\left(n_t-1\right) q}^t \|\x_{i+1}-\x_{i}\|^2 + \sum_{s \in [S]} \lambda_t^{s}\|\nabla f_s(\x_{\left(n_t-1\right) q}) - \bu_{\left(n_t-1\right) q}^s \|^2)\notag\\
    =&\frac{1}{2 \eta} \left( (1-\frac{3\mu\eta}{4})\| \x_t - \x_* \|^2 - \| \x_{t+1} - \x_* \|^2 \right) - (\frac{1}{2} \eta-\frac{\mu}{8}\eta^2 - \frac{1}{2}L \eta^2 ) \| \bd_t \|^2 \notag\\&
    + \frac{4}{\mu}( \frac{L^2}{|\mathcal{A}|} \sum_{i=\left(n_t-1\right) q}^t \|\x_{i+1}-\x_{i}\|^2 ).
\end{align}
where $(a)$ follows from Eqs. \eqref{eq24}. (b) is because the definition $\bd_t=\sum_{s \in [S]} \lambda_{t}^{s} \mathbf{u}_{t}^s$ as shown in Line 14 in Algorithm. \ref{alg}. $(c)$ is because 
$\|\x_t - \x_* \|^2 - \| \x_{t+1} - \x_* \|^2 = - \eta^2 \| \bd_t \|^2 + 2 \left< \eta \bd_t , \x_t - \x_* \right>$, 
and we choose
$\delta = \frac{\mu}{8}$ in $(d)$. $(e)$ follows from Lemma. \ref{lem:bounded1}.

Next, telescoping the above inequality over $t$ from $\left(n_t-1\right) q$ to $t$ where $t \leq n_t q-1$ and noting that for $\left(n_t-1\right) q \leq j \leq n_t q-1, n_j=n_t$, we obtain

\begin{align} \label{eqs24}
    & \sum_{i=\left(n_t-1\right) q}^t \sum_{s \in [S]} \lambda_i^{s} \left[ f_s(\x_{i+1}) - f_s(\x_*) \right]  \notag \\
    \stackrel{(a)}{\le} &\frac{1}{2 \eta} \left( (1-\frac{3\mu\eta}{4}) \sum_{i=\left(n_t-1\right) q}^t\| \x_i - \x_* \|^2 -  \sum_{i=\left(n_t-1\right) q}^t\| \x_{i+1} - \x_* \|^2 \right) \notag\\& - (\frac{1}{2} \eta-\frac{\mu}{8}\eta^2 - \frac{1}{2}L \eta^2 )  \sum_{i=\left(n_t-1\right) q}^t\| \bd_i \|^2 
    + \frac{4}{\mu}( \frac{L^2}{|\mathcal{A}|} \sum_{j=\left(n_t-1\right) q}^t  \sum_{i=\left(n_j-1\right) q}^j\|\x_{i+1}-\x_{i}\|^2 ) \notag \\
   \stackrel{(b)}{\le}&\frac{1}{2 \eta} \left( (1-\frac{3\mu\eta}{4}) \sum_{i=\left(n_t-1\right) q}^t\| \x_i - \x_* \|^2 -  \sum_{i=\left(n_t-1\right) q}^t\| \x_{i+1} - \x_* \|^2 \right) \notag\\& - (\frac{1}{2} \eta-\frac{\mu}{8}\eta^2 - \frac{1}{2}L \eta^2 )  \sum_{i=\left(n_t-1\right) q}^t\| \bd_i \|^2 
    + \frac{4}{\mu}( \frac{L^2}{|\mathcal{A}|} \sum_{j=\left(n_t-1\right) q}^t  \sum_{i=\left(n_t-1\right) q}^t\|\x_{i+1}-\x_{i}\|^2 )\notag \\
= &\frac{1}{2 \eta} \left( (1-\frac{3\mu\eta}{4}) \sum_{i=\left(n_t-1\right) q}^t\| \x_i - \x_* \|^2 -  \sum_{i=\left(n_t-1\right) q}^t\| \x_{i+1} - \x_* \|^2 \right) \notag\\& - (\frac{1}{2} \eta-\frac{\mu}{8}\eta^2 - \frac{1}{2}L \eta^2-
    \frac{4}{\mu} \frac{L^2 q  \eta^2}{|\mathcal{A}|} ) \sum_{i=\left(n_t-1\right) q }^t\|\bd_{i}\|^2 ),
\end{align}
where $(a)$ is from Eqs. \eqref{eqs23}. $(b)$ relaxes $j$ to $t$, since $j\leq t$.
We continue the proof by further driving

\begin{align}
    & \sum_{i=0}^{T} \sum_{s \in [S]} \lambda_i^{s}\left[ f_s(\x_{i+1}) - f_s(\x_*) \right] \notag\\ = & \sum_{i=0}^q\sum_{s \in [S]} \lambda_i^{s} \left[ f_s(\x_{i+1}) - f_s(\x_*) \right] +  \sum_{i=q}^{2q}\sum_{s \in [S]} \lambda_i^{s} \left[ f_s(\x_{i+1}) - f_s(\x_*)\right]+\notag\\&\cdot \cdot \cdot+ \sum_{i=(n_T-1)q}^{T}\sum_{s \in [S]} \lambda_i^{s} \left[ f_s(\x_{i+1}) - f_s(\x_*) \right] \notag \\
  \leq &\frac{1}{2 \eta} \left( (1-\frac{3\mu\eta}{4}) \sum_{i=0 }^{T} \| \x_i - \x_* \|^2 -  \sum_{i=0}^{T}\| \x_{i+1} - \x_* \|^2 \right) \notag\\& - (\frac{1}{2} \eta-\frac{\mu}{8}\eta^2 - \frac{1}{2}L \eta^2-
    \frac{4}{\mu} \frac{L^2 q  \eta^2}{|\mathcal{A}|} ) \sum_{i=0}^{T}\|\bd_{i}\|^2 ), 
\end{align}
where the last inequality is from Eq. \eqref{eqs16} and Eq. \eqref{eqs24}.
Next, we have
\begin{align}
    & \sum_{i=0}^{T}   \sum_{s \in [S]} \lambda_i^{s} \left[ f_s(\x_i) - f_s(\x_*) \right] \notag\\ = & \sum_{i=0}^{T}   \sum_{s \in [S]} \lambda_t^{s} \left[ f_s(\x_{i+1}) - f_s(\x_*) -  f_s(\x_{i+1}) + f_s(\x_i)  \right]  \notag\\
    \leq &\sum_{i=0}^{T}  \sum_{s \in [S]} \lambda_t^{s} \left[ f_s(\x_{i+1}) - f_s(\x_*) \right]  -\sum_{i=0}^{T} \sum_{s \in [S]} \lambda_t^{s}  | f_s(\x_{i+1}) - f_s(\x_i) |\notag\\
    \leq &\frac{1}{2 \eta} \left( (1-\frac{3\mu\eta}{4}) \sum_{i=0 }^{T} \| \x_i - \x_* \|^2 -  \sum_{i=0}^{T}\| \x_{i+1} - \x_* \|^2 \right) \notag\\& - (\frac{1}{2} \eta-\frac{\mu}{8}\eta^2 - \frac{1}{2}L \eta^2-
    \frac{4}{\mu} \frac{L^2 q  \eta^2}{|\mathcal{A}|}  -[\frac{\eta}{4}- \frac{\eta^3 q}{2}  \frac{L^2}{|\mathcal{A}|}] )\sum_{i=0}^{T}  \|\bd_i\|^2
\end{align}

Let $|\mathcal{A}|=q= \lceil\sqrt{n}\rceil $ and $\eta \leq \min\{\frac{1}{2\mu},\frac{1}{8L},\frac{\mu}{64L^2} \}$, we have $(\frac{1}{2} \eta-\frac{\mu}{8}\eta^2 - \frac{1}{2}L \eta^2-
    \frac{4}{\mu} \frac{L^2 q  \eta^2}{|\mathcal{A}|}  -[\frac{\eta}{4}- \frac{\eta^3 q}{2}  \frac{L^2}{|\mathcal{A}|}] )> \frac{\eta}{16}>0$
    
    Thus, we have
\begin{align}
    & \sum_{i=0}^{T}   \sum_{s \in [S]} \lambda_i^{s} \left[ f_s(\x_i) - f_s(\x_*) \right]\leq \frac{1}{2 \eta} \left( (1-\frac{3\mu\eta}{4}) \sum_{i=0 }^{T} \| \x_i - \x_* \|^2 -  \sum_{i=0}^{T}\| \x_{i+1} - \x_* \|^2\right).
\end{align}

Then, we have 

\begin{align}
    & \mathbb{E}_t  [\sum_{s \in [S]} \lambda_i^{s} \left[ f_s(\x_t) - f_s(\x_*) \right] ]\leq \frac{1}{2 \eta} \left( (1-\frac{3\mu\eta}{4}) \mathbb{E}_t \| \x_t - \x_* \|^2 -  \mathbb{E}_t\| \x_{t+1} - \x_* \|^2\right).
\end{align}
Based on Assumption.\ref{assump: add} and averaging using weight $w_t = ( 1 - \frac{3\mu \eta }{4})^{1-t}$ and using such weight to pick output $\x$, by using Lemma 1 in \cite{Karimireddy2020SCAFFOLD} with $\eta \geq \frac{1}{uR}$, we have

\begin{align}
    \mathbb{E}\|\x_t-\x^*\|^2\left[ f_s(\x_t) - f_s(\x_*) \right] ]  &\leq \| \x_0 - \x_* \|^2 \mu \exp( - \frac{3 \eta \mu T}{4}) \\
    &= \mathcal{O}(\mu \exp( - \mu T)).
\end{align}

Then we have the convergence rate $   \mathbb{E}\|\x_t-\x^*\|^2 = \mathcal{O}(\mu \exp( - \mu T))$.

Lastly, the total sample complexity can be calculated as:
	$\lceil \frac{T}{q} \rceil n + T\cdot |\mathcal{A}| \leq  \frac{T+q}{q}n + T\sqrt{n}= T\sqrt{n}+n+T\sqrt{n}=O(n+ \sqrt{n} \ln ({\mu/\epsilon})$.
	Thus, the overall sample complexity is $\mathcal{O}(n+ \sqrt{n} \ln ({\mu/\epsilon})$.
	This completes the proof.
 
\end{proof}

\section{Proof of convergence of \algm} \label{appdx:VR_MOOM_nonconvex}

\begin{lem} \label{lem:update}
For general $L$-smooth functions $\{ f_s, s \in [S] \}$, choose the learning rate $\eta$ s.t. $\eta \leq  \frac{1}{2}$, the update $d_t$ of the VR-MOO-M algorithm satisfies:
    \begin{align}
   f_s(\x_{t+1})   \leq &  f_s(\x_t) + \frac{\eta}{2}  \sum_{i=(n_t-1)q}^t \alpha^{(t-i)}\| \nabla f_s(\x_i) - \bu_i^s \|^2 - \frac{1}{2} \eta \sum_{i=(n_t-1)q}^t \alpha^{(t-i)} \| \bd_i \|^2\notag\\&+ \frac{1}{2}L \| \x_{t+1}-\x_{t} \|^2.
\end{align}
\end{lem}

\textbf{Proof of Lemma. \ref{lem:update}.}
\begin{proof}

\begin{align}\label{eqs40}
  &  f_s(\x_{t+1}) \leq f_s(\x_t) + \left< \nabla f_s(\x_t), \x_{t+1}-\x_{t} \right> + \frac{1}{2}L \| \x_{t+1}-\x_{t} \|^2 \notag\\
  & \stackrel{(a)}{\le} f_s(\x_t) + \left< \nabla f_s(\x_t) , \alpha(\x_{t+1}-\x_{t} )\right> +\left< \nabla f_s(\x_t) , -\eta \bd_t \right>+ \frac{1}{2}L \| \eta \bd_t \|^2 \notag\\
    & \stackrel{(b)}{=} f_s(\x_t) +\sum_{i=0}^t \alpha^{(t-i)} \left< \nabla f_s(\x_i), -\eta \bd_i \right> + \frac{1}{2}L \| \eta \bd_t \|^2 \notag\\
    &= f_s(\x_t) - \eta \sum_{i=0}^t\alpha^{(t-i)} \left< \nabla f_s(\x_i) - \bu_i^s, \bd_i \right> - \eta \sum_{i=0}^t \alpha^{(t-i)} \left< \bu_i^s, \bd_i \right> + \frac{1}{2}L \| \x_{t+1}-\x_{t} \|^2 \notag\\
    & \stackrel{(c)}{\le} f_s(\x_t) - \eta \sum_{i=0}^t \alpha^{(t-i)} \left< \nabla f_s(\x_i) - \bu_i^s,  \bd_i \right> - \eta\sum_{i=0}^t\alpha^{(t-i)}  \| \bd_i \|^2 + \frac{1}{2}L \| \x_{t+1}-\x_{t} \|^2 \notag\\
    & \stackrel{(d)}{\le} f_s(\x_t) + \frac{\eta}{2} \sum_{i=0}^t \alpha^{(t-i)} \| \nabla f_s(\x_i) - \bu_i^s \|^2 + \frac{1}{2} \eta \sum_{i=0}^t \alpha^{(t-i)} \| \bd_i \|^2 \notag\\&- \eta  \sum_{i=0}^t\alpha^{(t-i)}\| \bd_i\|^2 + \frac{1}{2}L \| \x_{t+1}-\x_{t} \|^2 \notag\\
    &= f_s(\x_t) + \frac{\eta}{2}  \sum_{i=0}^t \alpha^{(t-i)}\| \nabla f_s(\x_i) - \bu_i^s \|^2 - \frac{1}{2} \eta \sum_{i=0}^t \alpha^{(t-i)} \| \bd_i \|^2+ \frac{1}{2}L \| \x_{t+1}-\x_{t} \|^2.
\end{align}

(a) follows from  the objective function $f_s$ is $L$-smooth.  $(b)$ follows from the update rule of $\x_t$ shown in Line 19 in Algorithm. \ref{alg}.  (c) follows from  $\left< \bu_t^s, \bd_t \right> \geq \| \bd_t \|^2$ since $\bd_t$ is a general solution in the convex hull of the family of vectors $\{\bu_t^s, s \in [S] \}$ (see Lemma 2.1~\cite{desideri2012multiple}). (d) follows from the triangle inequality.

\end{proof}

\textbf{Proof of Theorem.~\ref{STIMULUS_M_NCRate}}

\begin{proof}
Taking expectation on both sides of the inequality in Lemma. \ref{lem:update}, we have

\begin{align} \label{eqss35}
    & \mathbb{E} [f_s(\x_{t+1})]  \notag\\ \stackrel{(a)}{\le} &   \mathbb{E}[f_s(\x_t)] +    \frac{\eta}{2}  \sum_{i=0}^t \alpha^{(t-i)}  \mathbb{E}\| \nabla f_s(\x_i) - \bu_i^s \|^2 - \frac{1}{2} \eta \sum_{i=0}^t \alpha^{(t-i)}  \mathbb{E}\| \bd_i \|^2+ \frac{1}{2}L  \mathbb{E}\| \x_{t+1}-\x_{t} \|^2\notag\\
  \stackrel{(b)}{\le} &   \mathbb{E}[f_s(\x_t)]  - \frac{1}{2} \eta \sum_{i=0}^t \alpha^{(t-i)}  \mathbb{E}\| \bd_i \|^2+ \frac{1}{2}L  \mathbb{E}\| \x_{t+1}-\x_{t} \|^2\notag\\
   &+ \frac{\eta}{2}  \sum_{j=0}^t  \alpha^{(t-j)} [\frac{L^2}{|\mathcal{A}|} \sum_{i=\left(n_t-1\right) q}^j \mathbb{E} \|\x_{i+1}-\x_{i}\|^2 + \mathbb{E}\|\nabla f_s(\x_{\left(n_t-1\right) q}) - \bu_{\left(n_t-1\right) q}^s \|^2] \notag\\
    = &  \mathbb{E}[f_s(\x_t)] - \frac{1}{2} \eta \sum_{i=0}^t \alpha^{(t-i)}  \mathbb{E}\| \bd_i \|^2+ \frac{1}{2}L  \mathbb{E}\| \x_{t+1}-\x_{t} \|^2\notag\\&+ \frac{\eta}{2}  \sum_{j=0}^t  \alpha^{(t-j)}  [\frac{L^2}{|\mathcal{A}|} \sum_{i=\left(n_t-1\right) q}^j  \mathbb{E}\|\x_{i+1}-\x_{i}\|^2],
\end{align}

where $(a)$ follows from Eqs. \ref{eqs40}. $(b)$ follows from the Lemma. \ref{lem:bounded1}. $(c)$ follows from $\mathbb{E}\|\nabla f_s(\x_{\left(n_t-1\right) q}) - \bu_{\left(n_t-1\right) q}^s \|^2=0$ as shown in Line 5 in our Algorithm. \ref{alg}.

Next, telescoping the above inequality over $t$ from $\left(n_t-1\right) q$ to $t$ where $t \leq n_t q-1$ and noting that for $\left(n_t-1\right) q \leq j \leq n_t q-1, n_j=n_t$ and let $\eta\leq \frac{1}{4L}$, we obtain

\begin{align}
  &    \mathbb{E}[f_s(\x_{t+1})]  \notag\\ \stackrel{(a)}{\le} &     \mathbb{E}[f_s(\x_{\left(n_t-1\right) q})] -\frac{\eta}{2} \sum_{j=\left(n_t-1\right) q}^t \sum_{i=0}^j \alpha^{(j-i)}   \mathbb{E}\| \bd_i \|^2+ \frac{1}{2}L   \sum_{i=\left(n_t-1\right) q}^t   \mathbb{E}\| \x_{i+1}-\x_i \|^2 \notag\\& + \frac{\eta}{2} \sum_{j=\left(n_t-1\right) q}^t  \sum_{i=0}^j  \alpha^{(j-i)}  [\frac{L^2}{|\mathcal{A}|} \sum_{r=\left(n_t-1\right) q}^i   \mathbb{E}\|\x_{r+1}-\x_{r}\|^2]   \notag\\
  \stackrel{(b)}{\le}&     \mathbb{E}[f_s(\x_{\left(n_t-1\right) q})] -\frac{\eta}{2} \sum_{j=\left(n_t-1\right) q}^t \sum_{i=0}^j \alpha^{(j-i)}   \mathbb{E}\| \bd_i \|^2+ \frac{1}{2}L   \sum_{i=\left(n_t-1\right) q}^t  \mathbb{E}\| \x_{i+1}-\x_i \|^2 \notag\\& + \frac{\eta}{2} \sum_{j=\left(n_t-1\right) q}^t  \sum_{i=0}^j  \alpha^{(j-i)}  [\frac{L^2}{|\mathcal{A}|} \sum_{r=\left(n_t-1\right) q}^{n_t q -1}   \mathbb{E}\|\x_{r+1}-\x_{r}\|^2]   \notag\\
\leq &    \mathbb{E} [f_s(\x_{\left(n_t-1\right) q})] -\frac{\eta}{2} \sum_{j=\left(n_t-1\right) q}^t \sum_{i=0}^j \alpha^{(j-i)}    \mathbb{E}\| \bd_i \|^2+ \frac{1}{2}L   \sum_{i=\left(n_t-1\right) q}^t   \mathbb{E}\| \x_{i+1}-\x_i \|^2 \notag\\& + \frac{\eta}{2} \sum_{j=\left(n_t-1\right) q}^t  \sum_{i=0}^j  \alpha^{(j-i)}  [\frac{L^2}{|\mathcal{A}|} q  \mathbb{E} \|\x_{j+1}-\x_{j}\|^2]   \notag\\
   \stackrel{(c)}{=} &    \mathbb{E} [f_s(\x_{\left(n_t-1\right) q})] -\frac{\eta}{2} \sum_{j=\left(n_t-1\right) q}^t \sum_{i=0}^j \alpha^{(j-i)}  \mathbb{E} \| \bd_i \|^2+ \frac{1}{2}L   \sum_{i=\left(n_t-1\right) q}^t   \mathbb{E}\| \x_{i+1}-\x_i \|^2 \notag\\& + \frac{\eta}{2} \sum_{j=\left(n_t-1\right) q}^t  \sum_{i=0}^j  \alpha^{(j-i)}  [ L^2  \mathbb{E} \|\x_{j+1}-\x_{j}\|^2]   \notag\\
   \stackrel{(d)}{=}  &     \mathbb{E}[f_s(\x_{\left(n_t-1\right) q})] -\frac{\eta}{2} \sum_{j=\left(n_t-1\right) q}^t \sum_{i=0}^j \alpha^{(j-i)}  \mathbb{E} \| \bd_i \|^2+ \frac{1}{2}L   \sum_{i=\left(n_t-1\right) q}^t   \mathbb{E}\| \x_{i+1}-\x_i \|^2 \notag\\& + \frac{\eta}{2} \sum_{j=\left(n_t-1\right) q}^t  \sum_{i=0}^j  \alpha^{(j-i)}  [ L^2   \mathbb{E}\|\eta \sum_{r=0}^{j} \alpha^{(j-r)} \bd_r\|^2]   \notag\\
       \stackrel{(e)}{\leq} &    \mathbb{E} [f_s(\x_{\left(n_t-1\right) q})] -\frac{\eta}{2} \sum_{j=\left(n_t-1\right) q}^t \sum_{i=0}^j \alpha^{(j-i)}   \mathbb{E}\| \bd_i \|^2+ \frac{1}{2}L   \sum_{i=\left(n_t-1\right) q}^t   \mathbb{E}\| \x_{i+1}-\x_i \|^2 \notag\\& + \frac{\eta}{2} \sum_{j=\left(n_t-1\right) q}^t  \sum_{i=0}^j  \alpha^{2(j-i)}  [ L^2\eta ^2   \mathbb{E}\|\bd_i\|^2]   \notag\\
                 \stackrel{(f)}{\leq}  &     \mathbb{E}[f_s(\x_{\left(n_t-1\right) q})] -\frac{\eta}{4} \sum_{j=\left(n_t-1\right) q}^t \sum_{i=0}^j \alpha^{(j-i)}   \mathbb{E}\| \bd_i \|^2+ \frac{1}{2}L   \sum_{j=\left(n_t-1\right) q}^t   \mathbb{E}\| \eta \sum_{i=0}^{j} \alpha^{(j-i)} \bd_j \|^2  \notag\\
           \stackrel{(g)}{\leq} &    \mathbb{E} [f_s(\x_{\left(n_t-1\right) q})] -\frac{\eta}{8} \sum_{j=\left(n_t-1\right) q}^t   \mathbb{E}\| \bd_j \|^2,
\end{align}
where $(a)$ holds from Eqs. \eqref{eqss35}. $(b)$ is extend $i$ to $t$ since $i\leq n_t q -1$. $(c)$ is because $q=|\mathcal{A}|=\lceil\sqrt{n}\rceil$. $(d)$ follows from the update rule of $\x_t$ shown in Line 19 in Algorithm. \ref{alg}. $(e)$ follows from the triangle inequality. $(f)$ and $(g)$ hold from $\eta \leq \frac{1}{2L}$ and $0 <\alpha<1 $. We continue the proof by further driving
\begin{align}
  &   [f_s(\x_{T})] -   [f_s(\x_{0})]  \notag\\&
  =(   [f_s(\x_{q})] -   [f_s(\x_{0})] ) +  (   [f_s(\x_{2q})] -   [f_s(\x_{q})] ) +\cdot +  (   [f_s(\x_{T})] -   [f_s(\x_{(n_T-1)q})] ) 
  \notag\\&\leq  -[\frac{\eta}{8}] \sum_{t=0}^{T-1}  \|\bd_t\|^2
\end{align}

Note that $ [ f_s\left(\x_{T+1}\right) ]\geq f_s^* \triangleq \inf _{\x \in \mathbb{R}^d} f_s(\x)$. Hence, we have
\begin{align}
  &
 [\frac{\eta}{8}] \sum_{t=0}^{T-1}  \|\bd_t\|^2\leq   [  [f_s(\x_{0})] -  [f_s(\x_{T})]  ]\leq   [  [f_s(\x_{0})] - f_s^*  ].
\end{align}

Based on the parameter setting $q =|\mathcal{A}|=\sqrt{n}$, we have 
\begin{align}
  &
 [\frac{\eta}{8}  ] \sum_{t=0}^{T-1}  \|\bd_t\|^2 \leq   [  [f_s(\x_{0})] -f_s^* ].
\end{align}

Since $\frac{1}{T} \sum_{t=0}^{T-1}\mathbb{E} \|d_t \|^2$ is just common descent directions.
According to Definition. \ref{def:stationary} shown in the paper, the quantity to our interest is 
$\|\sum_{s \in [S]}\lambda_t^s \nabla f(\mathbf{x})\|^2$. 
\begin{align}
  &
 \frac{1}{T} \sum_{t=0}^{T-1}  \mathbb{E}\|\sum_{s\in [S]}\lambda_t^s\nabla f_s(\x_t)\|^2
\stackrel{(a)}{\le}  ( 2S L^2 \eta ^2\frac{1}{T}  +2)  \frac{1}{T} \sum_{t=0}^{T-1}  \mathbb{E}\|\bd_t\|^2 
\end{align}
where $(a)$ follows from Eqs. \eqref{eqs21}.

Then, we can conclude that 
\begin{align}
  &
 \frac{1}{T} \sum_{t=0}^{T-1}  \mathbb{E}\|\sum_{s\in [S]}\lambda_t^s\nabla f_s(\x_t)\|^2 
\stackrel{(a)}{\le}
 ( 2S L^2 \eta ^2  +2)
  \frac{[  \mathbb{E}[f_s(\x_{0})] -f_s^* ]}{ \frac{\eta}{8} T},
\end{align}
where $(a)$ follows from Eqs. \eqref{eqs21} and Eqs. \ref{eqs20}.

Thus, we have

\begin{align}
  &
\frac{1}{T}\sum_{t=0}^{T-1}\min_{\boldsymbol{\lambda} \in C} \mathbb{E} \| \boldsymbol{\lambda}^{\top} \nabla \F(\x_t) \|^2 \leq \frac{1}{T} \sum_{t=0}^{T-1}  \mathbb{E}\|\sum_{s\in [S]}\lambda_t^s\nabla f_s(\x_t)\|^2 =\mathcal{O}(\frac{1}{T}).
\end{align}

The total sample complexity can be calculated as:
	$\lceil \frac{T}{q} \rceil n + T\cdot |\mathcal{A}| \leq  \frac{T+q}{q}n + T\sqrt{n}= T\sqrt{n}+n+T\sqrt{n}=O(n+ \sqrt{n} \epsilon^{-1})$.
	Thus, the overall sample complexity is $\mathcal{O}(n+ \sqrt{n} \epsilon^{-1})$.
	This completes the proof.
 
\end{proof}

\subsection{Proof od Theorem.~\ref{thm:STIMULUSm_SC}}
\begin{proof}

\begin{align} \label{eqs43}
    &f_s(\x_{t+1}) \notag\\\stackrel{(a)}{\le}& f_s(\x_t) + \left< \nabla f_s(\x_t),-\eta \sum_{t=0}^{T} \alpha^{(t-i)} \bd_i \right> + \frac{1}{2}L \| \eta \sum_{t=0}^{T} \alpha^{(t-i)} \bd_i \|^2 \notag\\
 \stackrel{(b)}{\le} & f_s(\x_*) + \left< \nabla f_s(\x_t), \x_t - \x_* \right> - \frac{\mu}{2} \| \x_t - \x_* \|^2  + \left< \nabla f_s(\x_t), -\eta \sum_{t=0}^{T} \alpha^{(t-i)} \bd_i \right> \notag\\& + \frac{1}{2}L \| \eta \sum_{t=0}^{T} \alpha^{(t-i)} \bd_i \|^2\notag\\
    = &f_s(\x_*) + \left< \nabla f_s(\x_t), \x_t - \x_* -\eta \sum_{t=0}^{T} \alpha^{(t-i)} \bd_i \right> - \frac{\mu}{2} \| \x_t - \x_* \|^2 + \frac{1}{2}L \| \eta \sum_{t=0}^{T} \alpha^{(t-i)} \bd_i \|^2\notag\\
    =& f_s(\x_*) + \left< \nabla f_s(\x_t)-\bu_t^s, \x_t - \x_* -\eta \sum_{t=0}^{T} \alpha^{(t-i)} \bd_i \right>+ \left< \bu_t^s, \x_t - \x_* -\eta \sum_{t=0}^{T} \alpha^{(t-i)} \bd_i \right> \notag\\&- \frac{\mu}{2} \| \x_t - \x_* \|^2 + \frac{1}{2}L \| \eta \sum_{t=0}^{T} \alpha^{(t-i)} \bd_i \|^2\notag\\
    \stackrel{(c)}{\le} & f_s(\x_*) + \frac{1}{2\delta}\| \nabla f_s(\x_t)-\bu_t^s\|^2+ \frac{\delta}{2}\| \x_t - \x_* -\eta \sum_{t=0}^{T} \alpha^{(t-i)} \bd_i\|^2 \notag\\& + \left< \bu_t^s, \x_t - \x_* -\eta \sum_{t=0}^{T} \alpha^{(t-i)} \bd_i \right> - \frac{\mu}{2} \| \x_t - \x_* \|^2 + \frac{1}{2}L \| \eta \sum_{t=0}^{T} \alpha^{(t-i)} \bd_i \|^2\notag\\
    \stackrel{(d)}{\le} & f_s(\x_*) + \frac{1}{2\delta}\| \nabla f_s(\x_t)-\bu_t^s\|^2+ \delta\| \x_t - \x_* \|^2+\delta\|\eta \sum_{t=0}^{T} \alpha^{(t-i)} \bd_i\|^2 \notag\\&+ \left< \bu_t^s, \x_t - \x_* -\eta \sum_{t=0}^{T} \alpha^{(t-i)} \bd_i \right> - \frac{\mu}{2} \| \x_t - \x_* \|^2 + \frac{1}{2}L \| \eta \sum_{t=0}^{T} \alpha^{(t-i)} \bd_i \|^2,
\end{align}
where $(a)$ is due to $L$-smoothness, $(b)$ follows from $\mu$-strongly convex. $(c)$ and $(d)$ follow from the Young's inequality.

Next, we have
\begin{align} \label{eqs44}
    & \sum_{s \in [S]} \lambda_t^{s} \left[ f_s(\x_{t+1}) - f_s(\x_*) \right]  \notag\\
    \stackrel{(a)}{\le}& \frac{1}{2\delta} \sum_{s \in [S]} \lambda_t^{s}\| \nabla f_s(\x_t)-\bu_t^s\|^2+ \delta\| \x_t - \x_* \|^2+\delta\|\eta \sum_{t=0}^{T} \alpha^{(t-i)} \bd_i\|^2\notag\\& +\left< \sum_{s \in [S]} \lambda_t^{s}\bu_t^s, \x_t - \x_* \right> - \frac{\mu}{2} \| \x_t - \x_* \|^2 + \left< \sum_{s \in [S]} \lambda_t^{s} \bu_t^s, -\eta \sum_{t=0}^{T} \alpha^{(t-i)} \bd_i \right> \notag\\& + \frac{1}{2}L \| \eta \sum_{t=0}^{T} \alpha^{(t-i)} \bd_i \|^2 \notag\\
    =&\frac{1}{2\delta}\sum_{s \in [S]} \lambda_t^{s}\| \nabla f_s(\x_t)-\bu_t^s\|^2+ \delta\| \x_t - \x_* \|^2+\delta\|\eta \sum_{t=0}^{T} \alpha^{(t-i)} \bd_i\|^2\notag\\& +\left< \sum_{s \in [S]} \lambda_t^{s}\bu_t^s, \x_t - \x_* -\eta \sum_{t=0}^{T} \alpha^{(t-i)} \bd_i \right> - \frac{\mu}{2} \| \x_t - \x_* \|^2 + \frac{1}{2}L \| \eta \sum_{t=0}^{T} \alpha^{(t-i)} \bd_i \|^2 \notag\\
    \stackrel{(b)}{=}&\frac{1}{2\delta}\sum_{s \in [S]} \lambda_t^{s}\| \nabla f_s(\x_t)-\bu_t^s\|^2+ \delta\| \x_t - \x_* \|^2+\delta\|\eta \sum_{t=0}^{T} \alpha^{(t-i)} \bd_i\|^2\notag\\& +\left< \bd_t , \x_t - \x_* -\eta \sum_{t=0}^{T} \alpha^{(t-i)} \bd_i \right> - \frac{\mu}{2} \| \x_t - \x_* \|^2 + \frac{1}{2}L \| \eta \sum_{t=0}^{T} \alpha^{(t-i)} \bd_i \|^2 \notag\\
  \stackrel{(c)}{\le} &\frac{1}{2 \eta} \left( \| \x_t - \x_* \|^2 - \| \x_{t+1} - \x_* \|^2 \right) - \frac{1}{2} \eta \|  \sum_{t=0}^{T} \alpha^{(t-i)} \bd_i \|^2 - \frac{\mu}{2} \| \x_t - \x_* \|^2 \notag\\& + \frac{1}{2}L \|\eta \sum_{t=0}^{T} \alpha^{(t-i)} \bd_i\|^2 
    + \frac{4}{\mu}\sum_{s \in [S]} \lambda_t^{s}\| \nabla f_s(\x_t)-\bu_t^s\|^2+ \frac{\mu}{8}\| \x_t - \x_* \|^2+\frac{\mu}{8}\|\eta \sum_{t=0}^{T} \alpha^{(t-i)} \bd_i\|^2\notag\\
    =&\frac{1}{2 \eta} \left( (1-\frac{3\mu\eta}{4})\| \x_t - \x_* \|^2 - \| \x_{t+1} - \x_* \|^2 \right) - (\frac{1}{2} \eta-\frac{\mu}{8}\eta^2 - \frac{1}{2}L \eta^2 ) \| \sum_{t=0}^{T} \alpha^{(t-i)} \bd_i \|^2 \notag\\&
    + \frac{4}{\mu}\sum_{s \in [S]} \lambda_t^{s}\| \nabla f_s(\x_t)-\bu_t^s\|^2\notag\\
     \stackrel{(e)}{\leq}&\frac{1}{2 \eta} \left( (1-\frac{3\mu\eta}{4})\| \x_t - \x_* \|^2 - \| \x_{t+1} - \x_* \|^2 \right) - (\frac{1}{2} \eta-\frac{\mu}{8}\eta^2 - \frac{1}{2}L \eta^2 ) \| \sum_{t=0}^{T} \alpha^{(t-i)} \bd_i \|^2 \notag\\&
    + \frac{4}{\mu}( \frac{L^2}{|\mathcal{A}|} \sum_{i=\left(n_t-1\right) q}^t \|\x_{i+1}-\x_{i}\|^2 +\sum_{s \in [S]} \lambda_t^{s}\|\nabla f_s(\x_{\left(n_t-1\right) q}) - \bu_{\left(n_t-1\right) q}^s \|^2)\notag\\
   \stackrel{(f)}{=}&\frac{1}{2 \eta} \left( (1-\frac{3\mu\eta}{4})\| \x_t - \x_* \|^2 - \| \x_{t+1} - \x_* \|^2 \right) - (\frac{1}{2} \eta-\frac{\mu}{8}\eta^2 - \frac{1}{2}L \eta^2 ) \| \sum_{t=0}^{T} \alpha^{(t-i)} \bd_i \|^2 \notag\\&
    + \frac{4}{\mu}( \frac{L^2}{|\mathcal{A}|} \sum_{i=\left(n_t-1\right) q}^t \|\x_{i+1}-\x_{i}\|^2 ).
\end{align}
where $(a)$ follows from Eqs. \eqref{eqs43}. (b) is because the definition $\bd_t=\sum_{s \in [S]} \lambda_{t}^{s} \mathbf{u}_{t}^s$ as shown in Line 14 in Algorithm. \ref{alg}. $(c)$ is because $\|\x_t - \x_* \|^2 - \| \x_{t+1} - \x_* \|^2 = - \eta^2 \| \sum_{t=0}^{T} \alpha^{(t-i)} \bd_i \|^2 + 2 \left< \eta \sum_{t=0}^{T} \alpha^{(t-i)} \bd_i , \x_t - \x_* \right>$, 
and we choose
$\delta = \frac{\mu}{8}$. $(e)$ and $(f)$ follow from $\sum_{s \in [S]} \lambda_t^{s}=1$ and $\|\nabla f_s(\mathbf{x}_{\left(n_t-1\right) q}) - \mathbf{u}_{\left(n_t-1\right) q}^s \|^2 =0$.

Next, telescoping the above inequality over $t$ from $\left(n_t-1\right) q$ to $t$ where $t \leq n_t q-1$ and noting that for $\left(n_t-1\right) q \leq j \leq n_t q-1, n_j=n_t$, we obtain
\begin{align} \label{eqs45}
    & \sum_{i=\left(n_t-1\right) q}^t \sum_{s \in [S]} \lambda_t^{s} \left[ f_s(\x_{i+1}) - f_s(\x_*) \right]  \notag \\
   \stackrel{(a)}{=} &\frac{1}{2 \eta} \left( (1-\frac{3\mu\eta}{4}) \sum_{i=\left(n_t-1\right) q}^t\| \x_i - \x_* \|^2 -  \sum_{i=\left(n_t-1\right) q}^t\| \x_{i+1} - \x_* \|^2 \right) \notag\\& - (\frac{1}{2} \eta-\frac{\mu}{8}\eta^2 - \frac{1}{2}L \eta^2 )  \sum_{i=\left(n_t-1\right) q}^t\| \sum_{t=0}^{T} \alpha^{(t-i)} \bd_i \|^2 
    + \frac{4}{\mu}( \frac{L^2}{|\mathcal{A}|} \sum_{j=\left(n_t-1\right) q}^t  \sum_{i=\left(n_j-1\right) q}^j\|\x_{i+1}-\x_{i}\|^2 ) \notag \\
 \stackrel{(b)}{\leq} &\frac{1}{2 \eta} \left( (1-\frac{3\mu\eta}{4}) \sum_{i=\left(n_t-1\right) q}^t\| \x_i - \x_* \|^2 -  \sum_{i=\left(n_t-1\right) q}^t\| \x_{i+1} - \x_* \|^2 \right) \notag\\& - (\frac{1}{2} \eta-\frac{\mu}{8}\eta^2 - \frac{1}{2}L \eta^2 )  \sum_{i=\left(n_t-1\right) q}^t\| \sum_{t=0}^{T} \alpha^{(t-i)} \bd_i \|^2 
    + \frac{4}{\mu}( \frac{L^2}{|\mathcal{A}|} \sum_{j=\left(n_t-1\right) q}^t  \sum_{i=\left(n_t-1\right) q}^t\|\x_{i+1}-\x_{i}\|^2 )\notag \\
\stackrel{(c)}{=}&\frac{1}{2 \eta} \left( (1-\frac{3\mu\eta}{4}) \sum_{i=\left(n_t-1\right) q}^t\| \x_i - \x_* \|^2 -  \sum_{i=\left(n_t-1\right) q}^t\| \x_{i+1} - \x_* \|^2 \right) \notag\\& - (\frac{1}{2} \eta-\frac{\mu}{8}\eta^2 - \frac{1}{2}L \eta^2-
    \frac{4}{\mu} \frac{L^2 q  \eta^2}{|\mathcal{A}|} ) \sum_{i=\left(n_t-1\right) q }^t\|\sum_{t=0}^{T} \alpha^{(t-i)} \bd_i\|^2 ),
\end{align}
where $(a)$ follows from Eqs. \eqref{eqs44}, $(b)$ extend $j$ to $t$. $(c)$ follows from the update rule of $\x_{t+1}$ shown in Eqs. \eqref{vrm_update}.

We continue the proof by further driving

\begin{align}\label{eqs46}
    & \sum_{t=0}^{T} \sum_{s \in [S]} \lambda_t^{s}\left[ f_s(\x_{i+1}) - f_s(\x_*) \right] \notag\\ = & \sum_{i=0}^q\sum_{s \in [S]} \lambda_t^{s} \left[ f_s(\x_{i+1}) - f_s(\x_*) \right] +  \sum_{i=q}^{2q}\sum_{s \in [S]} \lambda_t^{s} \left[ f_s(\x_{i+1}) - f_s(\x_*)\right]+\notag\\& \sum_{i=(n_T-1)q}^{T}\sum_{s \in [S]} \lambda_t^{s} \left[ f_s(\x_{i+1}) - f_s(\x_*) \right] \notag \\
  \stackrel{(a)}{\leq} &\frac{1}{2 \eta} \left( (1-\frac{3\mu\eta}{4}) \sum_{i=0 }^{T} \| \x_i - \x_* \|^2 -  \sum_{t=0}^{T}\| \x_{i+1} - \x_* \|^2 \right) \notag\\& - (\frac{1}{2} \eta-\frac{\mu}{8}\eta^2 - \frac{1}{2}L \eta^2-
    \frac{4}{\mu} \frac{L^2 q  \eta^2}{|\mathcal{A}|} ) \sum_{t=0}^{T}\|\sum_{t=0}^{T} \alpha^{(t-i)} \bd_i\|^2 ),
\end{align}
where $(a)$ follows from Eqs. \eqref{eqs45}.
Next, we have
\begin{align}
    & \sum_{t=0}^{T}   \sum_{s \in [S]} \lambda_t^{s} \left[ f_s(\x_i) - f_s(\x_*) \right] \notag\\ = & \sum_{t=0}^{T}   \sum_{s \in [S]} \lambda_t^{s} \left[ f_s(\x_{i+1}) - f_s(\x_*) -  f_s(\x_{i+1}) + f_s(\x_i)  \right]  \notag\\
   = &\sum_{t=0}^{T}  \sum_{s \in [S]} \lambda_t^{s} \left[ f_s(\x_{i+1}) - f_s(\x_*) \right]  -\sum_{t=0}^{T} \sum_{s \in [S]} \lambda_t^{s}  | f_s(\x_{i+1}) - f_s(\x_i) |\notag\\
    \stackrel{(a)}{leq} &\frac{1}{2 \eta} \left( (1-\frac{3\mu\eta}{4}) \sum_{i=0 }^{T} \| \x_i - \x_* \|^2 -  \sum_{t=0}^{T}\| \x_{i+1} - \x_* \|^2 \right) \notag\\& - (\frac{1}{2} \eta-\frac{\mu}{8}\eta^2 - \frac{1}{2}L \eta^2-
    \frac{4}{\mu} \frac{L^2 q  \eta^2}{|\mathcal{A}|}  -[\frac{\eta}{4}- \frac{\eta^3 q}{2}  \frac{L^2}{|\mathcal{A}|}] )\sum_{t=0}^{T}  \|\sum_{t=0}^{T} \alpha^{(t-i)} \bd_i\|^2,
\end{align}
where $(a)$ follows from Eqs. \eqref{eqs46}.
Let $|\mathcal{A}|=q= \lceil \sqrt{n} \rceil$ and $\eta \leq \min\{\frac{1}{2\mu},\frac{1}{8L},\frac{\mu}{64L^2} \}$, we have $(\frac{1}{2} \eta-\frac{\mu}{8}\eta^2 - \frac{1}{2}L \eta^2-
    \frac{4}{\mu} \frac{L^2 q  \eta^2}{|\mathcal{A}|}  -[\frac{\eta}{4}- \frac{\eta^3 q}{2}  \frac{L^2}{|\mathcal{A}|}] )> \frac{\eta}{16}>0$
    
    Thus, we have
\begin{align}
    & \sum_{t=0}^{T}   \sum_{s \in [S]} \lambda_t^{s} \left[ f_s(\x_i) - f_s(\x_*) \right]\leq \frac{1}{2 \eta} \left( (1-\frac{3\mu\eta}{4}) \sum_{i=0 }^{T} \| \x_i - \x_* \|^2 -  \sum_{t=0}^{T}\| \x_{i+1} - \x_* \|^2\right).
\end{align}

Then, we have 

\begin{align}
    & \mathbb{E}  [\sum_{s \in [S]} \lambda_t^{s} \left[ f_s(\x_t) - f_s(\x_*) \right] ]\leq \frac{1}{2 \eta} \left( (1-\frac{3\mu\eta}{4}) \mathbb{E} \| \x_t - \x_* \|^2 -  \mathbb{E}\| \x_{t+1} - \x_* \|^2\right).
\end{align}
Based on Asumption. \ref{assump: add} and averaging using weight $w_t = ( 1 - \frac{3\mu \eta }{4})^{1-t}$ and using such weight to pick output $\x$, by using Lemma 1 in \cite{Karimireddy2020SCAFFOLD} with $\eta \geq \frac{1}{uR}$, we have

\begin{align}
    \mathbb{E}\|\x_t-\x^*\|^2  &\leq \| \x_0 - \x_* \|^2 \mu \exp( - \frac{3 \eta \mu T}{4}) \\
    &= \mathcal{O}(\mu \exp( - \mu T)).
\end{align}

Then we have the convergence rate $   \mathbb{E}\|\x_t-\x^*\|^2= \mathcal{O}(\mu \exp( - \mu T))$.
the total sample complexity can be calculated as:
	$\lceil \frac{T}{q} \rceil n + T\cdot |\mathcal{A}| \leq  \frac{T+q}{q}n + T\sqrt{n}= T\sqrt{n}+n+T\sqrt{n}=O(n+ \sqrt{n} \ln ({\mu/\epsilon})$.
	Thus, the overall sample complexity is $\mathcal{O}(n+ \sqrt{n} \ln ({\mu/\epsilon})$.
	This completes the proof.

\end{proof}

 \section{Proof of convergence of \algp} \label{appdx:VRP_nonconvex}

\textbf{Proof of Theorem.~\ref{thm:STIMULUSmp_nonC} [Part 1]}

\begin{proof}

Recall that $ \mathcal{N}_s =\min\{ c_{\gamma} \sigma^2(\gamma_t)^{-1}, c_{\epsilon} \sigma^2\epsilon^{-1} ,n\} $. Then we have
	\begin{align}  \label{eqs52}
		\frac{I_{( \mathcal{N}_s <n)}}{ \mathcal{N}_s}  & \leq \frac{1}{ \min\{ c_{\epsilon} \sigma^2(\epsilon)^{-1} ,c_{\gamma} \sigma^2({\gamma_t})^{-1}     \}} \notag\\
		& = \max\{ \frac{{\gamma_t}}{c_{\gamma}\sigma^2}, \frac{\epsilon}{c_{\epsilon} \sigma^2}  \}  \leq \frac{{\gamma_t}}{  c_{\gamma}\sigma^2}+\frac{\epsilon}{ c_{\epsilon}\sigma^2}.
	\end{align}

From Lemma. \ref{lem:update2}, we have

\begin{align} \label{eqs53}
     [f_s(\x_{t+1})]  &\stackrel{(a)}{\leq}  [f_s(\x_t)] + \frac{\eta}{2}   \| \nabla f_s(\x_t) - \bu_t^s \|^2 -\frac{\eta}{4}  \| \bd_t \|^2\notag\\&
   \stackrel{(b)}{\leq}  [f_s(\x_t)] -\frac{\eta}{4}  \| \bd_t \|^2\notag\\&+ \frac{\eta}{2}  [\frac{L^2}{|\mathcal{A}|} \sum_{i=\left(n_t-1\right) q}^t \|\x_{i+1}-\x_{i}\|^2 +\|\nabla f_s(\x_{\left(n_t-1\right) q}) - \bu_{\left(n_t-1\right) q}^s \|^2] \notag\\&
  \stackrel{(c)}{\leq}  [f_s(\x_t)] -\frac{\eta}{4}  \| \bd_t \|^2+ \frac{\eta}{2}  [\frac{L^2}{|\mathcal{A}|} \sum_{i=\left(n_t-1\right) q}^t \eta^2 \|\bd_i\|^2+ \frac{I_{(\mathcal{N}_s <n)}}{\mathcal{N}_s} \sigma^2  ],
\end{align}
where $(a)$ follows from  Lemma. \ref{lem:update2}. $(b)$ follows from Lemma. \ref{lem:bounded1}. (c) follows from the update rule shown in Eqs. \eqref{STIMULUSP1}.

Next, telescoping the above inequality over $t$ from $\left(n_t-1\right) q$ to $t$ where $t \leq n_t q-1$ and noting that for $\left(n_t-1\right) q \leq j \leq n_t q-1, n_j=n_t$, and aking expectation on both sides of the inequality in Eqs. \eqref{eqs53},we obtain

\begin{align}\label{eqs54}
  &   \mathbb{E}[f_s(\x_{t+1})]  \notag\\ \stackrel{(a)}{\leq}   &\mathbb{E}[f_s(\x_{\left(n_t-1\right) q})] -\frac{\eta}{4} \sum_{j=\left(n_t-1\right) q}^t  \mathbb{E}\| \bd_j \|^2\notag\\& + \frac{\eta}{2}  [\frac{L^2}{|\mathcal{A}|} \sum_{j=\left(n_t-1\right) q}^t \sum_{i=\left(n_t-1\right) q}^j\eta^2 \mathbb{E}\|\bd_i\|^2 + \sum_{i=\left(n_t-1\right) q}^t\frac{I_{(\mathcal{N}_s <n)}}{\mathcal{N}_s} \sigma^2 ]  \notag\\ \stackrel{(b)}{\leq}  & \mathbb{E}[f_s(\x_{\left(n_t-1\right) q})] -\frac{\eta}{4} \sum_{j=\left(n_t-1\right) q}^t \mathbb{E} \| \bd_j \|^2\notag\\&+ \frac{\eta}{2}  \sum_{i=\left(n_t-1\right) q}^t[\frac{L^2}{|\mathcal{A}|} \sum_{j=\left(n_t-1\right) q}^t \sum_{i=\left(n_t-1\right) q}^t\eta^2 \mathbb{E}\|\bd_i\|^2 ] + \frac{\eta}{2} \sum_{i=\left(n_t-1\right) q}^t\frac{I_{(\mathcal{N}_s <n)}}{\mathcal{N}_s} \sigma^2 \notag\\ =&   \mathbb{E}[f_s(\x_{\left(n_t-1\right) q})] -\frac{\eta}{4} \sum_{j=\left(n_t-1\right) q}^t \mathbb{E} \| \bd_j \|^2\notag\\&+ \frac{\eta^3 q}{2}  [\frac{L^2}{|\mathcal{A}|} \sum_{j=\left(n_t-1\right) q}^t  \mathbb{E}\|\bd_j\|^2 ]+ \frac{\eta}{2}\sum_{i=\left(n_t-1\right) q}^t \frac{I_{(\mathcal{N}_s <n)}}{\mathcal{N}_s} \sigma^2  \notag\\ \stackrel{(c)}{=}& \mathbb{E}[f_s(\x_{\left(n_t-1\right) q})] -[\frac{\eta}{4}- \frac{\eta^3 q}{2}  \frac{L^2}{|\mathcal{A}|}] \sum_{j=\left(n_t-1\right) q}^t  \mathbb{E}\|\bd_j\|^2 + \frac{\eta}{2} \sum_{i=\left(n_t-1\right) q}^t(\frac{\gamma_{i}}{c_{\gamma}}+\frac{\epsilon}{c_{\epsilon}} ),
\end{align}
where $(a)$ follows from Eqs. \eqref{eqs53}, $(b)$ extends $j$ to $t$. $(c)$ follows from Eqs. \eqref{eqs52}

Recall that $ \gamma_t= \frac{1}{q}  \sum_{i=(n_t-1)q}^{t} \|	\bd_t \|^2 $. Then, we have
We continue the proof by further driving
\begin{align}
  & \mathbb{E}[f_s(\x_{T}) -   f_s(\x_{0})] \notag\\&
  = \mathbb{E}[(   [f_s(\x_{q})] -   [f_s(\x_{0})] ) +  (   [f_s(\x_{2q})] -   [f_s(\x_{q})] ) +\cdot +  (   [f_s(\x_{T})] -   [f_s(\x_{(n_T-1)q})] ) ]
  \notag\\&\stackrel{(a)}{\leq}   -[\frac{\eta}{4}- \frac{\eta^3 q}{2}  \frac{L^2}{|\mathcal{A}|} ] \sum_{t=0}^{T-1}  \mathbb{E}\|\bd_t\|^2+ \frac{\eta}{2} \sum_{t= 0 }^{T-1}(\frac{\mathbb{E}[\gamma_{i}]}{c_{\gamma}}+\frac{\epsilon}{c_{\epsilon}} )
    \notag\\&\stackrel{(b)}{\leq}   -[\frac{\eta}{4}- \frac{\eta^3 q}{2}  \frac{L^2}{|\mathcal{A}|} -\frac{\eta}{2 c_{\gamma}}  ] \sum_{t=0}^{T-1} \mathbb{E}\|\bd_t\|^2+ \frac{\eta}{2} T\frac{\epsilon}{c_{\epsilon}} ,
\end{align}

where $(a)$ is from Eqs. \eqref{eqs54}. $(b)$ follows from $\gamma_t= \frac{1}{q}  \sum_{i=(n_t-1)q}^{t} \|	\bd_t \|^2$.

Note that $ [ f_s\left(\x_{T+1}\right) ]\geq f_s^* \triangleq \inf _{\x \in \mathbb{R}^d} f_s(\x)$. Let $c_{\gamma}>4$. Hence, we have
\begin{align}
  &
 [\frac{\eta}{8}- \frac{\eta^3 q}{2}  \frac{L^2}{|\mathcal{A}|}-\frac{\eta}{2 c_{\gamma}} ] \sum_{t=0}^{T-1}  \mathbb{E}\|\bd_t\|^2\leq   \mathbb{E}[  [f_s(\x_{0})] -  [f_s(\x_{T})]  ]\leq   \mathbb{E}[  [f_s(\x_{0})] - f_s^*  ]+ \frac{\eta}{2} T\frac{\epsilon}{c_{\epsilon}}.
\end{align}

Based on the parameter setting $q=|\mathcal{A}|=\lceil\sqrt{n}\rceil$, we have 
\begin{align}
  &
 [\frac{\eta}{8}- \frac{\eta^3 L^2 }{2}  -\frac{\eta}{2 c_{\gamma}} ] \sum_{t=0}^{T-1} \mathbb{E} \|\bd_t\|^2 \leq   \mathbb{E}[  [f_s(\x_{0})] -f_s^* ]+ \frac{\eta}{2} T\frac{\epsilon}{c_{\epsilon}}.
\end{align}
Thus, we have
\begin{align}
  &
 \frac{1}{T} \sum_{t=0}^{T-1}  \mathbb{E}\|\bd_t\|^2 \leq   \frac{\mathbb{E}[  [f_s(\x_{0})] -f_s^* ]}{ [\frac{\eta}{8}- \frac{\eta^3 L^2 }{2}  -\frac{\eta}{2 c_{\gamma}} ] T}+ \frac{\eta}{2} \frac{\epsilon}{c_{\epsilon}}.
\end{align}

Let $\eta\leq \frac{1}{4L},c_{\gamma}\geq 8, c_{\epsilon}\geq \eta $, we have

Since $\frac{1}{T} \sum_{t=0}^{T-1}\mathbb{E} \|d_t \|^2$ is just common descent directions.
According to Definition. \ref{def:stationary} shown in the paper, the quantity to our interest is 
$\|\sum_{s \in [S]}\lambda_t^s \nabla f(\mathbf{x})\|^2$. 
\begin{align}
  &
 \frac{1}{T} \sum_{t=0}^{T-1}  \mathbb{E}\|\sum_{s\in [S]}\lambda_t^s\nabla f_s(\x_t)\|^2
\stackrel{(a)}{\le}  ( 2S L^2 \eta ^2 +2)  \frac{1}{T} \sum_{t=0}^{T-1}  \mathbb{E}\|\bd_t\|^2 
\end{align}
where $(a)$ follows from Eqs. \eqref{eqs21}.

Then, we can conclude that 
\begin{align}
  &
 \frac{1}{T} \sum_{t=0}^{T-1}  \mathbb{E}\|\sum_{s\in [S]}\lambda_t^s\nabla f_s(\x_t)\|^2 
\stackrel{(a)}{\le}
 ( 2S L^2 \eta ^2 +2)
(\frac{\mathbb{E}[  [f_s(\x_{0})] -f_s^* ]}{ [\frac{\eta}{8}- \frac{\eta^3 L^2 }{2}  -\frac{\eta}{2 c_{\gamma}} ] T}+ \frac{\eta}{2} \frac{\epsilon}{c_{\epsilon}}),
\end{align}
where $(a)$ follows from Eqs. \eqref{eqs21} and Eqs. \ref{eqs20}.

Thus, we have

\begin{align}
  &
\frac{1}{T}\sum_{t=0}^{T-1}\min_{\boldsymbol{\lambda} \in C} \mathbb{E} \| \boldsymbol{\lambda}^{\top} \nabla \F(\x_t) \|^2 \leq \frac{1}{T} \sum_{t=0}^{T-1}  \mathbb{E}\|\sum_{s\in [S]}\lambda_t^s\nabla f_s(\x_t)\|^2 =\mathcal{O}(\frac{1}{T}).
\end{align}

The total sample complexity can be calculated as:
	$\lceil \frac{T}{q} \rceil n + T\cdot |\mathcal{A}| \leq  \frac{T+q}{q}n + T\sqrt{n}= T\sqrt{n}+n+T\sqrt{n}=O(n+ \sqrt{n} \epsilon^{-1})$.
	Thus, the overall sample complexity is $\mathcal{O}(n+ \sqrt{n} \epsilon^{-1})$.
	This completes the proof.
 
\end{proof}

\subsection{Proof of Theorem.~\ref{thm:STIMULUSP_SC} [Part 1]}
\begin{proof}

\begin{align} \label{eqs62}
    &f_s(\x_{t+1}) \notag\\\stackrel{(a)}{\le}& f_s(\x_t) + \left< \nabla f_s(\x_t), -\eta \bd_t \right> + \frac{1}{2}L \| \eta \bd_t \|^2 \notag\\
    \stackrel{(b)}{\le}& f_s(\x_*) + \left< \nabla f_s(\x_t), \x_t - \x_* \right> - \frac{\mu}{2} \| \x_t - \x_* \|^2  + \left< \nabla f_s(\x_t), -\eta \bd_t \right> + \frac{1}{2}L \| \eta \bd_t \|^2\notag\\
   = &f_s(\x_*) + \left< \nabla f_s(\x_t), \x_t - \x_* -\eta \bd_t \right> - \frac{\mu}{2} \| \x_t - \x_* \|^2 + \frac{1}{2}L \| \eta \bd_t \|^2\notag\\
    =& f_s(\x_*) + \left< \nabla f_s(\x_t)-\bu_t^s, \x_t - \x_* -\eta \bd_t \right>+ \left< \bu_t^s, \x_t - \x_* -\eta \bd_t \right> \notag\\&- \frac{\mu}{2} \| \x_t - \x_* \|^2 + \frac{1}{2}L \| \eta \bd_t \|^2\notag\\
   \stackrel{(c)}{\le}& f_s(\x_*) + \frac{1}{2\delta}\| \nabla f_s(\x_t)-\bu_t^s\|^2+ \frac{\delta}{2}\| \x_t - \x_* -\eta \bd_t\|^2 + \left< \bu_t^s, \x_t - \x_* -\eta \bd_t \right> \notag\\&- \frac{\mu}{2} \| \x_t - \x_* \|^2 + \frac{1}{2}L \| \eta \bd_t \|^2\notag\\
    \stackrel{(d)}{\le}& f_s(\x_*) + \frac{1}{2\delta}\| \nabla f_s(\x_t)-\bu_t^s\|^2+ \delta\| \x_t - \x_* \|^2+\delta\|\eta \bd_t\|^2 \notag\\&+ \left< \bu_t^s, \x_t - \x_* -\eta \bd_t \right> - \frac{\mu}{2} \| \x_t - \x_* \|^2 + \frac{1}{2}L \| \eta \bd_t \|^2,
\end{align}
where $(a)$ follows from $L$-smoothness, $(b)$ follows from $\mu$-strongly convexity. $(c)$ follows from Young's inequality, and $(d)$ follows from triangle inequality.

Then, we have
\begin{align} \label{eqs70}
    & \sum_{s \in [S]} \lambda_t^{s} \left[ f_s(\x_{t+1}) - f_s(\x_*) \right]  \\
    \stackrel{(a)}{\le}& \frac{1}{2\delta}\| \nabla f_s(\x_t)-\bu_t^s\|^2+ \delta\| \x_t - \x_* \|^2+\delta\|\eta \bd_t\|^2\notag\\& +\left< \sum_{s \in [S]} \lambda_t^{s}\bu_t^s, \x_t - \x_* \right> - \frac{\mu}{2} \| \x_t - \x_* \|^2 + \left< \sum_{s \in [S]} \lambda_t^{s} \bu_t^s, -\eta \bd_t \right> + \frac{1}{2}L \| \eta \bd_t \|^2 \\
    =&\frac{1}{2\delta}\| \nabla f_s(\x_t)-\bu_t^s\|^2+ \delta\| \x_t - \x_* \|^2+\delta\|\eta \bd_t\|^2\notag\\& +\left< \sum_{s \in [S]} \lambda_t^{s}\bu_t^s, \x_t - \x_* -\eta \bd_t \right> - \frac{\mu}{2} \| \x_t - \x_* \|^2 + \frac{1}{2}L \| \eta \bd_t \|^2 \\
     \stackrel{(b)}{\le}&\frac{1}{2\delta}\| \nabla f_s(\x_t)-\bu_t^s\|^2+ \delta\| \x_t - \x_* \|^2+\delta\|\eta \bd_t\|^2\notag\\& +\left< \bd_t , \x_t - \x_* -\eta \bd_t \right> - \frac{\mu}{2} \| \x_t - \x_* \|^2 + \frac{1}{2}L \| \eta \bd_t \|^2 \\
     =& \left< \bd_t , \x_t - \x_* \right> - \eta \| \bd_t \|^2 - \frac{\mu}{2} \| \x_t - \x_* \|^2 + \frac{1}{2}L \eta^2 \| \bd_t \|^2 \notag\\&+ \frac{1}{2\delta}\| \nabla f_s(\x_t)-\bu_t^s\|^2+ \delta\| \x_t - \x_* \|^2+\delta\|\eta \bd_t\|^2\notag\\
     \stackrel{(c)}{=}&\frac{1}{2 \eta} \left( \| \x_t - \x_* \|^2 - \| \x_{t+1} - \x_* \|^2 \right) - \frac{1}{2} \eta \| \bd_t \|^2 - \frac{\mu}{2} \| \x_t - \x_* \|^2 + \frac{1}{2}L \eta^2 \|\bd_t \|^2 \notag\\
    &+ \frac{4}{\mu}\| \nabla f_s(\x_t)-\bu_t^s\|^2+ \frac{\mu}{8}\| \x_t - \x_* \|^2+\frac{\mu}{8}\|\eta \bd_t\|^2\\
    =&\frac{1}{2 \eta} \left( (1-\frac{3\mu\eta}{4})\| \x_t - \x_* \|^2 - \| \x_{t+1} - \x_* \|^2 \right) - (\frac{1}{2} \eta-\frac{\mu}{8}\eta^2 - \frac{1}{2}L \eta^2 ) \| \bd_t \|^2 \notag\\&
    + \frac{4}{\mu}\| \nabla f_s(\x_t)-\bu_t^s\|^2\\
     \stackrel{(d)}{\le}&\frac{1}{2 \eta} \left( (1-\frac{3\mu\eta}{4})\| \x_t - \x_* \|^2 - \| \x_{t+1} - \x_* \|^2 \right) - (\frac{1}{2} \eta-\frac{\mu}{8}\eta^2 - \frac{1}{2}L \eta^2 ) \| \bd_t \|^2 \notag\\&
    + \frac{4}{\mu}( \frac{L^2}{|\mathcal{A}|} \sum_{i=\left(n_t-1\right) q}^t \|\x_{i+1}-\x_{i}\|^2 +\|\nabla f_s(\x_{\left(n_t-1\right) q}) - \bu_{\left(n_t-1\right) q}^s \|^2)\\
     \stackrel{(f)}{\le}&\frac{1}{2 \eta} \left( (1-\frac{3\mu\eta}{4})\| \x_t - \x_* \|^2 - \| \x_{t+1} - \x_* \|^2 \right) - (\frac{1}{2} \eta-\frac{\mu}{8}\eta^2 - \frac{1}{2}L \eta^2 ) \| \bd_t \|^2 \notag\\&
    + \frac{4}{\mu}( \frac{L^2}{|\mathcal{A}|} \sum_{i=\left(n_t-1\right) q}^t \|\x_{i+1}-\x_{i}\|^2 ) +\frac{\mu}{4}\frac{I_{(\mathcal{N}_s <n)}}{\mathcal{N}_s} \sigma^2  .
\end{align}
where $(a)$ follows from Eqs.\eqref{eqs62}. (b) follows from the definition $\bd_t=\sum_{s \in [S]} \lambda_{t}^{s} \mathbf{u}_{t}^s$ as shown in Line 14 in Algorithm. \ref{alg}. $(c)$ is because $\|\x_t - \x_* \|^2 - \| \x_{t+1} - \x_* \|^2 = - \eta^2 \| \bd_t \|^2 + 2 \left< \eta \bd_t , \x_t - \x_* \right>$. $(d)$ is from Lemma. \ref{lem:bounded1} and we choose
$\delta = \frac{\mu}{8}$. $(e)$ is from Eqs. \eqref{eqs52}.

Next, telescoping the above inequality over $t$ from $\left(n_t-1\right) q$ to $t$ where $t \leq n_t q-1$ and noting that for $\left(n_t-1\right) q \leq j \leq n_t q-1, n_j=n_t$, we obtain

\begin{align} \label{eqs71}
    & \sum_{i=\left(n_t-1\right) q}^t \sum_{s \in [S]} \lambda_t^{s} \left[ f_s(\x_{i+1}) - f_s(\x_*) \right]  \notag \\
   \stackrel{(a)}{\le}&\frac{1}{2 \eta} \left( (1-\frac{3\mu\eta}{4}) \sum_{i=\left(n_t-1\right) q}^t\| \x_i - \x_* \|^2 -  \sum_{i=\left(n_t-1\right) q}^t\| \x_{i+1} - \x_* \|^2 \right) \notag\\& - (\frac{1}{2} \eta-\frac{\mu}{8}\eta^2 - \frac{1}{2}L \eta^2 )  \sum_{i=\left(n_t-1\right) q}^t\| \bd_i \|^2 
    + \frac{4}{\mu  }( \frac{L^2}{|\mathcal{A}|} \sum_{j=\left(n_t-1\right) q}^t  \sum_{i=\left(n_j-1\right) q}^j\|\x_{i+1}-\x_{i}\|^2 )\notag\\
    &+ \frac{\mu S}{4}\sum_{i=\left(n_t-1\right) q}^t\frac{I_{(\mathcal{N}_s <n)}}{\mathcal{N}_s} \sigma^2 \notag \\
    \stackrel{(b)}{\le}&\frac{1}{2 \eta} \left( (1-\frac{3\mu\eta}{4}) \sum_{i=\left(n_t-1\right) q}^t\| \x_i - \x_* \|^2 -  \sum_{i=\left(n_t-1\right) q}^t\| \x_{i+1} - \x_* \|^2 \right) \notag\\& - (\frac{1}{2} \eta-\frac{\mu}{8}\eta^2 - \frac{1}{2}L \eta^2 )  \sum_{i=\left(n_t-1\right) q}^t\| \bd_i \|^2 
    + \frac{4}{\mu}( \frac{L^2}{|\mathcal{A}|} \sum_{j=\left(n_t-1\right) q}^t  \sum_{i=\left(n_t-1\right) q}^t\|\x_{i+1}-\x_{i}\|^2 )\notag\\
    &+\frac{\mu}{4} \sum_{i=\left(n_t-1\right) q}^t\frac{I_{(\mathcal{N}_s <n)}}{\mathcal{N}_s} \sigma^2 \notag \\
   \stackrel{(c)}{\le}&\frac{1}{2 \eta} \left( (1-\frac{3\mu\eta}{4}) \sum_{i=\left(n_t-1\right) q}^t\| \x_i - \x_* \|^2 -  \sum_{i=\left(n_t-1\right) q}^t\| \x_{i+1} - \x_* \|^2 \right) \notag\\& - (\frac{1}{2} \eta-\frac{\mu}{8}\eta^2 - \frac{1}{2}L \eta^2-
    \frac{4}{\mu} \frac{L^2 q  \eta^2}{|\mathcal{A}|} ) \sum_{i=\left(n_t-1\right) q }^t\|\bd_{i}\|^2 )\notag\\
    &+ \frac{\mu}{4} \sum_{i=\left(n_t-1\right) q}^t(\frac{[\gamma_i]}{c_{\gamma}}+\frac{\epsilon}{c_{\epsilon}} ),
\end{align}
where $(a)$ follows from Eqs. \eqref{eqs70} and the fact that $\lambda_t^s \leq 1 \forall s\in [S]$. $(b)$ extends $j$ to $t$. $(c)$ is because $t-(n_t -1)q\geq q$.
We continue the proof by further driving

\begin{align}\label{eqs72}
    & \sum_{t=0}^{T} \sum_{s \in [S]} \lambda_t^{s}\left[ f_s(\x_{i+1}) - f_s(\x_*) \right] \notag\\ = & \sum_{i=0}^q\sum_{s \in [S]} \lambda_t^{s} \left[ f_s(\x_{i+1}) - f_s(\x_*) \right] +  \sum_{i=q}^{2q}\sum_{s \in [S]} \lambda_t^{s} \left[ f_s(\x_{i+1}) - f_s(\x_*)\right]+\notag\\&\cdot+ \sum_{i=(n_T-1)q}^{T}\sum_{s \in [S]} \lambda_t^{s} \left[ f_s(\x_{i+1}) - f_s(\x_*) \right] \notag \\
\stackrel{(a)}{\le} &\frac{1}{2 \eta} \left( (1-\frac{3\mu\eta}{4}) \sum_{i=0 }^{T} \| \x_i - \x_* \|^2 -  \sum_{t=0}^{T}\| \x_{i+1} - \x_* \|^2 \right) \notag\\& - (\frac{1}{2} \eta-\frac{\mu}{8}\eta^2 - \frac{1}{2}L \eta^2-
    \frac{4}{\mu} \frac{L^2 q  \eta^2}{|\mathcal{A}|} +\frac{\mu}{4 c_{\gamma}} )\sum_{t=0}^{T}\|\bd_{i}\|^2 + \frac{\mu}{4} T\frac{\epsilon}{c_{\epsilon}},
\end{align}
where $(a)$ follows from Eqs. \eqref{eqs71} and $\gamma_t= \frac{1}{q}  \sum_{i=(n_t-1)q}^{t} \|	\bd_t \|^2$.

Next, we have
\begin{align}
    & \sum_{t=0}^{T}   \sum_{s \in [S]} \lambda_t^{s} \left[ f_s(\x_i) - f_s(\x_*) \right] \notag\\ = & \sum_{t=0}^{T}   \sum_{s \in [S]} \lambda_t^{s} \left[ f_s(\x_{i+1}) - f_s(\x_*) -  f_s(\x_{i+1}) + f_s(\x_i)  \right]  \notag\\
   = &\sum_{t=0}^{T}  \sum_{s \in [S]} \lambda_t^{s} \left[ f_s(\x_{i+1}) - f_s(\x_*) \right]  +\sum_{t=0}^{T} \sum_{s \in [S]} \lambda_t^{s}  | f_s(\x_{i+1}) - f_s(\x_i) |\notag\\
 \stackrel{(a)}{\le} &\frac{1}{2 \eta} \left( (1-\frac{3\mu\eta}{4}) \sum_{i=0 }^{T} \| \x_i - \x_* \|^2 -  \sum_{t=0}^{T}\| \x_{i+1} - \x_* \|^2 \right) \notag\\& - (\frac{1}{2} \eta-\frac{\mu}{8}\eta^2 - \frac{1}{2}L \eta^2-
    \frac{4}{\mu} \frac{L^2 q  \eta^2}{|\mathcal{A}|}  -[\frac{\eta}{4}- \frac{\eta^3 q}{2}  \frac{L^2}{|\mathcal{A}|}] -\frac{\mu}{4 c_{\gamma}} )\sum_{t=0}^{T}  \|\bd_i\|^2+\frac{\mu}{4} T\frac{\epsilon}{c_{\epsilon}},
\end{align}
where $(a)$ follows from Eqs. \eqref{eqs72}.

Let $|\mathcal{A}|=q= \lceil\sqrt{n}\rceil $ and $\eta \leq \min\{\frac{1}{2\mu},\frac{1}{8L},\frac{\mu}{64L^2} \},c_{\gamma}\geq \frac{8\mu}{\eta}$, we have $(\frac{1}{2} \eta-\frac{\mu}{8}\eta^2 - \frac{1}{2}L \eta^2-
    \frac{4}{\mu} \frac{L^2 q  \eta^2}{|\mathcal{A}|}  -[\frac{\eta}{4}- \frac{\eta^3 q}{2}  \frac{L^2}{|\mathcal{A}|}]-\frac{\mu}{4 c_{\gamma}} )> \frac{\eta}{32}>0$
    
    Thus, we have
\begin{align}
    & \sum_{t=0}^{T}   \sum_{s \in [S]} \lambda_t^{s} \left[ f_s(\x_i) - f_s(\x_*) \right]\leq \frac{1}{2 \eta} \left( (1-\frac{3\mu\eta}{4}) \sum_{i=0 }^{T} \| \x_i - \x_* \|^2 -  \sum_{t=0}^{T}\| \x_{i+1} - \x_* \|^2\right).
\end{align}

Then, we have 

\begin{align}
    & \mathbb{E}  [\sum_{s \in [S]} \lambda_t^{s} \left[ f_s(\x_t) - f_s(\x_*) \right] ]\leq \frac{1}{2 \eta} \left( (1-\frac{3\mu\eta}{4}) \mathbb{E} \| \x_t - \x_* \|^2 -  \mathbb{E}\| \x_{t+1} - \x_* \|^2\right)+\frac{\mu}{4} T\frac{\epsilon}{c_{\epsilon}}.
\end{align}
Averaging using weight $w_t = ( 1 - \frac{3\mu \eta }{4})^{1-t}$ and using such weight to pick output $\x$.
By using Lemma 1 in \cite{Karimireddy2020SCAFFOLD} with $\eta \geq \frac{1}{uR}, c_{\epsilon}>\frac{\mu}{2}$ and Assumption. \ref{assump: add}, we have

\begin{align}
   \mathbb{E}\|\x_t-\x^*\|^2 &\leq \| \x_0 - \x_* \|^2 \mu \exp( - \frac{3 \eta \mu T}{4}) +\frac{\mu}{4} T\frac{\epsilon}{c_{\epsilon}}\\
    &= \mathcal{O}(\mu \exp( - \mu T)).
\end{align}

Then we have the convergence rate $   \mathbb{E}\|\x_t-\x^*\|^2 = \mathcal{O}(\mu \exp( - \mu T))$.

The total sample complexity can be calculated as:
	$\lceil \frac{T}{q} \rceil n + T\cdot |\mathcal{A}| \leq  \frac{T+q}{q}n + T\sqrt{n}= T\sqrt{n}+n+T\sqrt{n}=O(n+ \sqrt{n} \ln ({\mu/\epsilon})$.
	Thus, the overall sample complexity is $\mathcal{O}(n+ \sqrt{n} \ln ({\mu/\epsilon})$.
	This completes the proof.
\end{proof}

\section{Proof of convergence of \algmp} \label{appdx:VRMP_nonconvex}

\textbf{Proof of Theorem.~\ref{thm:STIMULUSmp_nonC} [Part 2]}

\begin{proof}
From Lemma. \ref{lem:update}, we have

\begin{align}\label{eqs78}
    & [f_s(\x_{t+1})]  \notag\\ \stackrel{(a)}{\le}&  [f_s(\x_t)] +    \frac{\eta}{2}  \sum_{i=0}^t \alpha^{(t-i)} \| \nabla f_s(\x_i) - \bu_i^s \|^2 - \frac{1}{2} \eta \sum_{i=0}^t \alpha^{(t-i)} \| \bd_i \|^2+ \frac{1}{2}L \| \x_{t+1}-\x_{t} \|^2\notag\\
\stackrel{(b)}{\le} &  [f_s(\x_t)]  - \frac{1}{2} \eta \sum_{i=0}^t \alpha^{(t-i)} \| \bd_i \|^2+ \frac{1}{2}L \| \x_{t+1}-\x_{t} \|^2\notag\\
   &+ \frac{\eta}{2}  \sum_{j=0}^t  \alpha^{(t-j)} [\frac{L^2}{|\mathcal{A}|} \sum_{i=\left(n_t-1\right) q}^j \|\x_{i+1}-\x_{i}\|^2 +\|\nabla f_s(\x_{\left(n_t-1\right) q}) - \bu_{\left(n_t-1\right) q}^s \|^2] \notag\\
    \stackrel{(c)}{\le} & [f_s(\x_t)] - \frac{1}{2} \eta \sum_{i=0}^t \alpha^{(t-i)} \| \bd_i \|^2+ \frac{1}{2}L \| \x_{t+1}-\x_{t} \|^2+ \frac{\eta}{2}  \sum_{j=0}^t  \alpha^{(t-j)}  [\frac{L^2}{|\mathcal{A}|} \sum_{i=\left(n_t-1\right) q}^j \|\x_{i+1}-\x_{i}\|^2] \notag\\&+ \frac{\eta}{2} \sum_{i=0}^t\alpha^{(t-i)} (\frac{\gamma_{i}}{c_{\gamma}}+\frac{\epsilon}{c_{\epsilon}} ),
\end{align}
where $(a)$ follows from Lemma \ref{lem:update}. $(b)$ follows from Lemma. \ref{lem:bounded1}. $(c)$ follows from Eqs. \eqref{eqs52}.

Next, telescoping the above inequality over $t$ from $\left(n_t-1\right) q$ to $t$ where $t \leq n_t q-1$ and noting that for $\left(n_t-1\right) q \leq j \leq n_t q-1, n_j=n_t$ and let $\eta\leq \frac{1}{4L}$, we obtain

\begin{align}
  &   [f_s(\x_{t+1})]  \notag\\   \stackrel{(a)}{\le} &   [f_s(\x_{\left(n_t-1\right) q})] -\frac{\eta}{2} \sum_{j=\left(n_t-1\right) q}^t \sum_{i=0}^j \alpha^{(j-i)} \| \bd_i \|^2+ \frac{1}{2}L   \sum_{i=\left(n_t-1\right) q}^t \| \x_{i+1}-\x_i \|^2  \notag\\& + \frac{\eta}{2} \sum_{j=\left(n_t-1\right) q}^t  \sum_{i=0}^j  \alpha^{(j-i)}  [\frac{L^2}{|\mathcal{A}|} \sum_{r=\left(n_t-1\right) q}^i \|\x_{r+1}-\x_{r}\|^2]  \notag\\&+ \frac{\eta}{2} \sum_{j=\left(n_t-1\right) q}^t\sum_{i=0}^j \alpha^{(j-i)}(\frac{[\gamma_{i}]}{c_{\gamma}}+\frac{\epsilon}{c_{\epsilon}} ) \notag\\
\stackrel{(b)}{\le} &   [f_s(\x_{\left(n_t-1\right) q})] -\frac{\eta}{2} \sum_{j=\left(n_t-1\right) q}^t \sum_{i=0}^j \alpha^{(j-i)}  \| \bd_i \|^2+ \frac{1}{2}L   \sum_{i=\left(n_t-1\right) q}^t \| \x_{i+1}-\x_i \|^2 \notag\\& + \frac{\eta}{2} \sum_{j=\left(n_t-1\right) q}^t  \sum_{i=0}^j  \alpha^{(j-i)}  [\frac{L^2}{|\mathcal{A}|} q \|\x_{j+1}-\x_{j}\|^2]  \notag\\&+ \frac{\eta}{2} \sum_{j=\left(n_t-1\right) q}^t\sum_{i=0}^j \alpha^{(j-i)}(\frac{[\gamma_{i}]}{c_{\gamma}}+\frac{\epsilon}{c_{\epsilon}} ) \notag\\
\stackrel{(c)}{=}&   [f_s(\x_{\left(n_t-1\right) q})] -\frac{\eta}{2} \sum_{j=\left(n_t-1\right) q}^t \sum_{i=0}^j \alpha^{(j-i)} \| \bd_i \|^2+ \frac{1}{2}L   \sum_{i=\left(n_t-1\right) q}^t \| \x_{i+1}-\x_i \|^2 \notag\\& + \frac{\eta}{2} \sum_{j=\left(n_t-1\right) q}^t  \sum_{i=0}^j  \alpha^{(j-i)}  [ L^2 \|\x_{j+1}-\x_{j}\|^2] \notag\\&+ \frac{\eta}{2} \sum_{j=\left(n_t-1\right) q}^t\sum_{i=0}^j \alpha^{(j-i)}(\frac{[\gamma_{i}]}{c_{\gamma}}+\frac{\epsilon}{c_{\epsilon}} ) \notag\\
 \stackrel{(d)}{\le} &   [f_s(\x_{\left(n_t-1\right) q})] -\frac{\eta}{2} \sum_{j=\left(n_t-1\right) q}^t \sum_{i=0}^j \alpha^{(j-i)} \| \bd_i \|^2+ \frac{1}{2}L   \sum_{i=\left(n_t-1\right) q}^t \| \x_{i+1}-\x_i \|^2 \notag\\& + \frac{\eta}{2} \sum_{j=\left(n_t-1\right) q}^t  \sum_{i=0}^j  \alpha^{(j-i)}  [ L^2 \|\eta \sum_{r=0}^{j} \alpha^{(j-r)} \bd_r\|^2]   \notag\\&+ \frac{\eta}{2} \sum_{j=\left(n_t-1\right) q}^t\sum_{i=0}^j \alpha^{(j-i)}(\frac{[\gamma_{i}]}{c_{\gamma}}+\frac{\epsilon}{c_{\epsilon}} ) \notag\\
       = &   [f_s(\x_{\left(n_t-1\right) q})] -\frac{\eta}{2} \sum_{j=\left(n_t-1\right) q}^t \sum_{i=0}^j \alpha^{(j-i)} \| \bd_i \|^2+ \frac{1}{2}L   \sum_{i=\left(n_t-1\right) q}^t \| \x_{i+1}-\x_i \|^2 \notag\\& + \frac{\eta}{2} \sum_{j=\left(n_t-1\right) q}^t  \sum_{i=0}^j  \alpha^{3(j-i)}  [ L^2\eta ^2 \|\bd_i\|^2]  \notag\\&+  \frac{\eta}{2} \sum_{j=\left(n_t-1\right) q}^t\sum_{i=0}^j \alpha^{(j-i)}(\frac{[\gamma_{i}]}{c_{\gamma}}+\frac{\epsilon}{c_{\epsilon}} ) \notag\\
            \stackrel{(e)}{\le}&   [f_s(\x_{\left(n_t-1\right) q})] -\frac{\eta}{2} \sum_{j=\left(n_t-1\right) q}^t \sum_{i=0}^j \alpha^{(j-i)} \| \bd_i \|^2+ \frac{1}{2}L   \sum_{i=\left(n_t-1\right) q}^t \| \x_{i+1}-\x_i \|^2 \notag\\& + \frac{\eta}{2} \sum_{j=\left(n_t-1\right) q}^t  \sum_{i=0}^j  \alpha^{(j-i)}  [ L^2\eta ^2 \|\bd_i\|^2]  \notag\\&+  \frac{\eta}{2} \sum_{j=\left(n_t-1\right) q}^t\sum_{i=0}^j \alpha^{(j-i)}(\frac{[\gamma_{i}]}{c_{\gamma}}+\frac{\epsilon}{c_{\epsilon}} ) \notag\\
            \stackrel{(f)}{\le}&   [f_s(\x_{\left(n_t-1\right) q})] -\frac{\eta}{4} \sum_{j=\left(n_t-1\right) q}^t \sum_{i=0}^j\alpha^{(j-i)}\| \bd_i \|^2+ \frac{1}{2}L   \sum_{i=\left(n_t-1\right) q}^t \| \x_{i+1}-\x_i \|^2  \notag\\&+  \frac{\eta}{2} \sum_{j=\left(n_t-1\right) q}^t\sum_{i=0}^j \alpha^{(j-i)}(\frac{[\gamma_{i}]}{c_{\gamma}}+\frac{\epsilon}{c_{\epsilon}} ) \notag\\
                 \stackrel{(g)}{\le}&   [f_s(\x_{\left(n_t-1\right) q})] -\frac{\eta}{4} \sum_{j=\left(n_t-1\right) q}^t \sum_{i=0}^j \alpha^{(j-i)} \| \bd_i \|^2+ \frac{1}{2}L   \sum_{j=\left(n_t-1\right) q}^t \| \eta \sum_{i=0}^{j} \alpha^{(j-i)} \bd_j \|^2 \notag\\&+  \frac{\eta}{2} \sum_{j=\left(n_t-1\right) q}^t\sum_{i=0}^j \alpha^{(j-i)}(\frac{[\gamma_{i}]}{c_{\gamma}}+\frac{\epsilon}{c_{\epsilon}} )  \notag\\
               \stackrel{(h)}{\le} &   [f_s(\x_{\left(n_t-1\right) q})] -\frac{\eta}{8} \sum_{j=\left(n_t-1\right) q}^t \sum_{i=0}^j\alpha^{(j-i)} \| \bd_i \|^2 +  \frac{\eta}{2} \sum_{j=\left(n_t-1\right) q}^t\sum_{i=0}^j \alpha^{(j-i)}(\frac{[\gamma_{i}]}{c_{\gamma}}+\frac{\epsilon}{c_{\epsilon}} ) ,
\end{align}
where $(a)$ follows from Eqs. \eqref{eqs78}. $(b)$ follows from $i\leq n_t q$. $(c)$ follows from $q=|\mathcal{A}|=\lceil\sqrt{n}\rceil$. $(d)$ and $(g)$ follow from the update rule of $\x_t$ shown in Line 19 in Algorithm. \ref{alg}.
$(e)$ follows from $0<\alpha<1$, then we have $\alpha^2{(j-i)}<\alpha^{(j-i)} $. $(f)$ and $(h)$ follow from $\eta\leq \frac{1}{4L}$
Recall that $ \gamma_t= \frac{1}{q}  \sum_{i=(n_t-1)q}^{t} \|	\bd_t \|^2 $. Then, we have
\begin{align}
  &   \mathbb{E}[f_s(\x_{T})] -   [f_s(\x_{0})]  \notag\\&
  =\mathbb{E}(   [f_s(\x_{q})] -   [f_s(\x_{0})] ) +  (   [f_s(\x_{2q})] -   [f_s(\x_{q})] ) +\cdot +  (   [f_s(\x_{T})] -   [f_s(\x_{(n_T-1)q})] ) 
  \notag\\& \stackrel{(a)}{\le} -[\frac{\eta}{8}] \sum_{t=0}^{T-1}  \sum_{i=0}^j\alpha^{(j-i)} \mathbb{E} \|\bd_t\|^2 +\frac{\eta}{2 c_{\gamma}}   \sum_{t=0}^{T-1} \sum_{i=0}^j\alpha^{(j-i)} \mathbb{E} \|\bd_t\|^2+ \frac{\eta}{2} T q \frac{\epsilon}{c_{\epsilon}} 
    \notag\\& 
 \stackrel{(b)}{\le} -[\frac{\eta}{16}] \sum_{t=0}^{T-1}  \sum_{i=0}^j\alpha^{(j-i)} \mathbb{E} \|\bd_t\|^2 + \frac{\eta}{2} T q \frac{\epsilon}{c_{\epsilon}}  
    \notag\\& \stackrel{(c)}{\le}-[\frac{\eta}{16}] \sum_{t=0}^{T-1}   \mathbb{E} \|\bd_t\|^2 + \frac{\eta}{2} T q \frac{\epsilon}{c_{\epsilon}} ,
\end{align}
where $(a)$ follows from $c_{\gamma}\geq 8$, $(c)$ follows from $0<\alpha<1$.

Note that $ [ f_s\left(\x_{T+1}\right) ]\geq f_s^* \triangleq \inf _{\x \in \mathbb{R}^d} f_s(\x)$. Hence, we have
\begin{align}
  &
 [\frac{\eta}{16}] \sum_{t=0}^{T-1}  \|\bd_t\|^2\leq   [  [f_s(\x_{0})] -  [f_s(\x_{T})]  ]\leq   [  [f_s(\x_{0})] - f_s^*  ].
\end{align}

Based on the parameter setting $q^2 =|\mathcal{A}|=\sqrt{n}$, we have 
\begin{align}
  &
 [\frac{\eta}{16}  ] \sum_{t=0}^{T-1}  \|\bd_t\|^2 \leq   [  [f_s(\x_{0})] -f_s^* ].
\end{align}
Thus, we have
\begin{align}
  &
 \frac{1}{T} \sum_{t=0}^{T-1}  \|\bd_t\|^2 \leq   \frac{[  [f_s(\x_{0})] -f_s^* ]}{ [\frac{\eta}{16} ] T}.
\end{align}

Since $\frac{1}{T} \sum_{t=0}^{T-1}\mathbb{E} \|d_t \|^2$ is just common descent directions.
According to Definition. \ref{def:stationary} shown in the paper, the quantity to our interest is 
$\|\sum_{s \in [S]}\lambda_t^s \nabla f(\mathbf{x})\|^2$. 
\begin{align}
  &
 \frac{1}{T} \sum_{t=0}^{T-1}  \mathbb{E}\|\sum_{s\in [S]}\lambda_t^s\nabla f_s(\x_t)\|^2
\stackrel{(a)}{\le}  ( 2S L^2 \eta ^2 +2)  \frac{1}{T} \sum_{t=0}^{T-1}  \mathbb{E}\|\bd_t\|^2 
\end{align}
where $(a)$ follows from Eqs. \eqref{eqs21}.

Then, we can conclude that 

\begin{align}
  &
\frac{1}{T}\sum_{t=0}^{T-1}\min_{\boldsymbol{\lambda} \in C} \mathbb{E} \| \boldsymbol{\lambda}^{\top} \nabla \F(\x_t) \|^2 \leq \frac{1}{T} \sum_{t=0}^{T-1}  \mathbb{E}\|\sum_{s\in [S]}\lambda_t^s\nabla f_s(\x_t)\|^2 =\mathcal{O}(\frac{1}{T}).
\end{align}

The total sample complexity can be calculated as:
	$\lceil \frac{T}{q} \rceil n + T\cdot |\mathcal{A}| \leq  \frac{T+q}{q}n + T\sqrt{n}= T\sqrt{n}+n+T\sqrt{n}=O(n+ \sqrt{n} \epsilon^{-1})$.
	Thus, the overall sample complexity is $\mathcal{O}(n+ \sqrt{n} \epsilon^{-1})$.
	This completes the proof.
 
\end{proof}

\subsection{Proof of Theorem.~\ref{thm:STIMULUSP_SC} [Part 2]}
\begin{proof}

\begin{align} \label{eqs86}
    &f_s(\x_{t+1}) \notag\\  \stackrel{(a)}{\le}& f_s(\x_t) + \left< \nabla f_s(\x_t),-\eta \sum_{t=0}^{T} \alpha^{(t-i)} \bd_i \right> + \frac{1}{2}L \| \eta \sum_{t=0}^{T} \alpha^{(t-i)} \bd_i \|^2 \notag\\
      \stackrel{(b)}{\le}& f_s(\x_*) + \left< \nabla f_s(\x_t), \x_t - \x_* \right> - \frac{\mu}{2} \| \x_t - \x_* \|^2  + \left< \nabla f_s(\x_t), -\eta \sum_{t=0}^{T} \alpha^{(t-i)} \bd_i \right>\notag\\& + \frac{1}{2}L \| \eta \sum_{t=0}^{T} \alpha^{(t-i)} \bd_i \|^2\notag\\
    = &f_s(\x_*) + \left< \nabla f_s(\x_t), \x_t - \x_* -\eta \sum_{t=0}^{T} \alpha^{(t-i)} \bd_i \right> - \frac{\mu}{2} \| \x_t - \x_* \|^2 + \frac{1}{2}L \| \eta \sum_{t=0}^{T} \alpha^{(t-i)} \bd_i \|^2\notag\\
    =& f_s(\x_*) + \left< \nabla f_s(\x_t)-\bu_t^s, \x_t - \x_* -\eta \sum_{t=0}^{T} \alpha^{(t-i)} \bd_i \right>+ \left< \bu_t^s, \x_t - \x_* -\eta \sum_{t=0}^{T} \alpha^{(t-i)} \bd_i \right> \notag\\&- \frac{\mu}{2} \| \x_t - \x_* \|^2 + \frac{1}{2}L \| \eta \sum_{t=0}^{T} \alpha^{(t-i)} \bd_i \|^2\notag\\
    \stackrel{(c)}{\le}& f_s(\x_*) + \frac{1}{2\delta}\| \nabla f_s(\x_t)-\bu_t^s\|^2+ \frac{\delta}{2}\| \x_t - \x_* -\eta \sum_{t=0}^{T} \alpha^{(t-i)} \bd_i\|^2 \notag\\&+ \left< \bu_t^s, \x_t - \x_* -\eta \sum_{t=0}^{T} \alpha^{(t-i)} \bd_i \right> \notag\\&- \frac{\mu}{2} \| \x_t - \x_* \|^2 + \frac{1}{2}L \| \eta \sum_{t=0}^{T} \alpha^{(t-i)} \bd_i \|^2\notag\\
     \stackrel{(d)}{\le} & f_s(\x_*) + \frac{1}{2\delta}\| \nabla f_s(\x_t)-\bu_t^s\|^2+ \delta\| \x_t - \x_* \|^2+\delta\|\eta \sum_{t=0}^{T} \alpha^{(t-i)} \bd_i\|^2 \notag\\&+ \left< \bu_t^s, \x_t - \x_* -\eta \sum_{t=0}^{T} \alpha^{(t-i)} \bd_i \right> - \frac{\mu}{2} \| \x_t - \x_* \|^2 + \frac{1}{2}L \| \eta \sum_{t=0}^{T} \alpha^{(t-i)} \bd_i \|^2,
\end{align}
where $(a)$ follows from $L$-smoothness assumption, $(b)$ follows from $\mu$-strongly convex. $(c)$ and $(d)$ follow from the triangle inequality.

\begin{align} \label{eqs88}
    & \sum_{s \in [S]} \lambda_t^{s} \left[ f_s(\x_{t+1}) - f_s(\x_*) \right]  \\
       \stackrel{(a)}{\le} & \frac{1}{2\delta} \sum_{s \in [S]} \lambda_t^{s} \| \nabla f_s(\x_t)-\bu_t^s\|^2+ \delta\| \x_t - \x_* \|^2+\delta\|\eta \sum_{t=0}^{T} \alpha^{(t-i)} \bd_i\|^2\notag\\& +\left< \sum_{s \in [S]} \lambda_t^{s}\bu_t^s, \x_t - \x_* \right> - \frac{\mu}{2} \| \x_t - \x_* \|^2 + \left< \sum_{s \in [S]} \lambda_t^{s} \bu_t^s, -\eta \sum_{t=0}^{T} \alpha^{(t-i)} \bd_i \right> \notag\\&+ \frac{1}{2}L \| \eta \sum_{t=0}^{T} \alpha^{(t-i)} \bd_i \|^2  \notag\\
    =&\frac{1}{2\delta} \sum_{s \in [S]} \lambda_t^{s} \| \nabla f_s(\x_t)-\bu_t^s\|^2+ \delta\| \x_t - \x_* \|^2+\delta\|\eta \sum_{t=0}^{T} \alpha^{(t-i)} \bd_i\|^2\notag\\& +\left< \sum_{s \in [S]} \lambda_t^{s}\bu_t^s, \x_t - \x_* -\eta \sum_{t=0}^{T} \alpha^{(t-i)} \bd_i \right> - \frac{\mu}{2} \| \x_t - \x_* \|^2 \notag\\&+ \frac{1}{2}L \| \eta \sum_{t=0}^{T} \alpha^{(t-i)} \bd_i \|^2  \notag\\
    =&\frac{1}{2\delta} \sum_{s \in [S]} \lambda_t^{s} \| \nabla f_s(\x_t)-\bu_t^s\|^2+ \delta\| \x_t - \x_* \|^2+\delta\|\eta \sum_{t=0}^{T} \alpha^{(t-i)} \bd_i\|^2\notag\\& +\left< \bd_t , \x_t - \x_* -\eta \sum_{t=0}^{T} \alpha^{(t-i)} \bd_i \right> - \frac{\mu}{2} \| \x_t - \x_* \|^2 + \frac{1}{2}L \| \eta \sum_{t=0}^{T} \alpha^{(t-i)} \bd_i \|^2  \notag\\
  \stackrel{(b)}{\le} &\frac{1}{2 \eta} \left( \| \x_t - \x_* \|^2 - \| \x_{t+1} - \x_* \|^2 \right) - \frac{1}{2} \eta \|  \sum_{t=0}^{T} \alpha^{(t-i)} \bd_i \|^2 - \frac{\mu}{2} \| \x_t - \x_* \|^2 \notag\\&+ \frac{1}{2}L \|\eta \sum_{t=0}^{T} \alpha^{(t-i)} \bd_i\|^2 \notag\\
    &+ \frac{4}{\mu}\sum_{s \in [S]} \lambda_t^{s}\| \nabla f_s(\x_t)-\bu_t^s\|^2+ \frac{\mu}{8}\| \x_t - \x_* \|^2+\frac{\mu}{8}\|\eta \sum_{t=0}^{T} \alpha^{(t-i)} \bd_i\|^2 \notag\\
   = &\frac{1}{2 \eta} \left( (1-\frac{3\mu\eta}{4})\| \x_t - \x_* \|^2 - \| \x_{t+1} - \x_* \|^2 \right) - (\frac{1}{2} \eta-\frac{\mu}{8}\eta^2 - \frac{1}{2}L \eta^2 ) \| \sum_{t=0}^{T} \alpha^{(t-i)} \bd_i \|^2 \notag\\&
    + \frac{4}{\mu}\sum_{s \in [S]} \lambda_t^{s}\| \nabla f_s(\x_t)-\bu_t^s\|^2 \notag\\
    \stackrel{(c)}{\leq}&\frac{1}{2 \eta} \left( (1-\frac{3\mu\eta}{4})\| \x_t - \x_* \|^2 - \| \x_{t+1} - \x_* \|^2 \right) - (\frac{1}{2} \eta-\frac{\mu}{8}\eta^2 - \frac{1}{2}L \eta^2 ) \| \sum_{t=0}^{T} \alpha^{(t-i)} \bd_i \|^2 \notag\\&
    + \frac{4}{\mu}( \frac{L^2}{|\mathcal{A}|} \sum_{i=\left(n_t-1\right) q}^t \|\x_{i+1}-\x_{i}\|^2 +\sum_{s \in [S]} \lambda_t^{s}\|\nabla f_s(\x_{\left(n_t-1\right) q}) - \bu_{\left(n_t-1\right) q}^s \|^2) \notag\\
      \stackrel{(d)}{\leq} &\frac{1}{2 \eta} \left( (1-\frac{3\mu\eta}{4})\| \x_t - \x_* \|^2 - \| \x_{t+1} - \x_* \|^2 \right) - (\frac{1}{2} \eta-\frac{\mu}{8}\eta^2 - \frac{1}{2}L \eta^2 ) \| \sum_{t=0}^{T} \alpha^{(t-i)} \bd_i \|^2 \notag\\&
    + \frac{4}{\mu}( \frac{L^2}{|\mathcal{A}|} \sum_{i=\left(n_t-1\right) q}^t \|\x_{i+1}-\x_{i}\|^2 )+\frac{\mu S }{4}\frac{I_{(\mathcal{N}_s <n)}}{\mathcal{N}_s} \sigma^2 .
\end{align}
where $(a)$ follows from Eqs. \eqref{eqs86}, (b) follows from $\|\x_t - \x_* \|^2 - \| \x_{t+1} - \x_* \|^2 = - \eta^2 \| \bd_t \|^2 + 2 \left< \eta \bd_t , \x_t - \x_* \right>$ and we choose
$\delta = \frac{\mu}{8}$. $(c)$ is from Lemma. \ref{lem:bounded1}. $(d)$ is from Eqs. \eqref{eqs52}. $(d)$ follows from $0<\lambda_t^s<1, \forall s\in[S]$

Next, telescoping the above inequality over $t$ from $\left(n_t-1\right) q$ to $t$ where $t \leq n_t q-1$ and noting that for $\left(n_t-1\right) q \leq j \leq n_t q-1, n_j=n_t$, we obtain

\begin{align}
    & \sum_{i=\left(n_t-1\right) q}^t \sum_{s \in [S]} \lambda_t^{s} \left[ f_s(\x_{i+1}) - f_s(\x_*) \right]  \notag \\
 \stackrel{(a)}{\leq}&\frac{1}{2 \eta} \left( (1-\frac{3\mu\eta}{4}) \sum_{i=\left(n_t-1\right) q}^t\| \x_i - \x_* \|^2 -  \sum_{i=\left(n_t-1\right) q}^t\| \x_{i+1} - \x_* \|^2 \right) \notag\\& - (\frac{1}{2} \eta-\frac{\mu}{8}\eta^2 - \frac{1}{2}L \eta^2 )  \sum_{i=\left(n_t-1\right) q}^t\| \sum_{i=0 }^t \alpha^{(t-i)} \bd_i \|^2 
   \notag\\& + \frac{4}{\mu}( \frac{L^2}{|\mathcal{A}|} \sum_{j=\left(n_t-1\right) q}^t  \sum_{i=\left(n_j-1\right) q}^j\|\x_{i+1}-\x_{i}\|^2 ) +\frac{\mu}{4 c_{\gamma}} \sum_{i=\left(n_t-1\right) q }^t\|\alpha^{(t-i)}\bd_{i}\|^2 \notag \\&+\frac{\mu}{4 c_{\gamma}} \sum_{t=\left(n_t-1\right) q }^t\|\alpha^{(t-i)}\bd_{i}\|^2 + +\frac{\mu}{4} \sum_{i=\left(n_t-1\right) q}^t\frac{\epsilon}{c_{\epsilon}}\notag \\
 \stackrel{(b)}{\leq} &\frac{1}{2 \eta} \left( (1-\frac{3\mu\eta}{4}) \sum_{i=\left(n_t-1\right) q}^t\| \x_i - \x_* \|^2 -  \sum_{i=\left(n_t-1\right) q}^t\| \x_{i+1} - \x_* \|^2 \right) \notag\\& - (\frac{1}{2} \eta-\frac{\mu}{8}\eta^2 - \frac{1}{2}L \eta^2 )  \sum_{i=\left(n_t-1\right) q}^t\| \sum_{t=0}^{T} \alpha^{(t-i)} \bd_i \|^2 
   \notag\\& + \frac{4}{\mu}( \frac{L^2}{|\mathcal{A}|} \sum_{j=\left(n_t-1\right) q}^t  \sum_{i=\left(n_t-1\right) q}^t\|\x_{i+1}-\x_{i}\|^2 )\notag\\&+\frac{\mu}{4 c_{\gamma}} \sum_{t=\left(n_t-1\right) q }^t\|\alpha^{(t-i)}\bd_{i}\|^2 + +\frac{\mu}{4} \sum_{i=\left(n_t-1\right) q}^t\frac{\epsilon}{c_{\epsilon}}\notag \\
 \stackrel{(c)}{\leq}&\frac{1}{2 \eta} \left( (1-\frac{3\mu\eta}{4}) \sum_{i=\left(n_t-1\right) q}^t\| \x_i - \x_* \|^2 -  \sum_{i=\left(n_t-1\right) q}^t\| \x_{i+1} - \x_* \|^2 \right) + \frac{\mu}{4} \sum_{i=\left(n_t-1\right) q}^t\frac{\epsilon}{c_{\epsilon}}\notag\\& - (\frac{1}{2} \eta-\frac{\mu}{8}\eta^2 - \frac{1}{2}L \eta^2-
    \frac{4}{\mu} \frac{L^2 q  \eta^2}{|\mathcal{A}|} ) \sum_{i=\left(n_t-1\right) q }^t\|\sum_{t=0}^{T} \alpha^{(t-i)} \bd_i\|^2 )\notag\\&+\frac{\mu}{4 c_{\gamma}} \sum_{t=\left(n_t-1\right) q }^t\|\alpha^{(t-i)}\bd_{i}\|^2+ \frac{\mu}{4} \sum_{i=\left(n_t-1\right) q}^t\frac{\epsilon}{c_{\epsilon}},
\end{align}
where $(a)$ follows from Eqs. \eqref{eqs88}, $(b)$ extends $j$ to $t$. $(c)$ follows from $t\leq n_t q -1$.

We continue the proof by further driving
\begin{align} \label{eqs90}
    & \sum_{t=0}^{T} \sum_{s \in [S]} \lambda_t^{s}\left[ f_s(\x_{i+1}) - f_s(\x_*) \right] \notag\\ = & \sum_{i=0}^q\sum_{s \in [S]} \lambda_t^{s} \left[ f_s(\x_{i+1}) - f_s(\x_*) \right] +  \sum_{i=q}^{2q}\sum_{s \in [S]} \lambda_t^{s} \left[ f_s(\x_{i+1}) - f_s(\x_*)\right]+\cdot+ \sum_{i=(n_T-1)q}^{T}\sum_{s \in [S]} \lambda_t^{s} \left[ f_s(\x_{i+1}) - f_s(\x_*) \right] \notag \\
  \leq &\frac{1}{2 \eta} \left( (1-\frac{3\mu\eta}{4}) \sum_{i=0 }^{T} \| \x_i - \x_* \|^2 -  \sum_{t=0}^{T}\| \x_{i+1} - \x_* \|^2 \right) - (\frac{1}{2} \eta-\frac{\mu}{8}\eta^2 - \frac{1}{2}L \eta^2-
    \frac{4}{\mu} \frac{L^2 q  \eta^2}{|\mathcal{A}|} ) \sum_{t=0}^{T}\|\sum_{t=0}^{T} \alpha^{(t-i)} \bd_i\|^2 )\notag\\&+\frac{\mu}{4 c_{\gamma}} \sum_{t=0}^{T}\|\alpha^{(t-i)} \bd_{i}\|^2 + \frac{\mu}{4} T\frac{\epsilon}{c_{\epsilon}}.
\end{align}

Next, we have
\begin{align}
    & \sum_{t=0}^{T}   \sum_{s \in [S]} \lambda_t^{s} \left[ f_s(\x_i) - f_s(\x_*) \right] \notag\\ = & \sum_{t=0}^{T}   \sum_{s \in [S]} \lambda_t^{s} \left[ f_s(\x_{i+1}) - f_s(\x_*) -  f_s(\x_{i+1}) + f_s(\x_i)  \right]  \notag\\
    \leq &\sum_{t=0}^{T}  \sum_{s \in [S]} \lambda_t^{s} \left[ f_s(\x_{i+1}) - f_s(\x_*) \right]  +\sum_{t=0}^{T} \sum_{s \in [S]} \lambda_t^{s}  | f_s(\x_{i+1}) - f_s(\x_i) |\notag\\
     \stackrel{(a)}{\leq} &\frac{1}{2 \eta} \left( (1-\frac{3\mu\eta}{4}) \sum_{i=0 }^{T} \| \x_i - \x_* \|^2 -  \sum_{t=0}^{T}\| \x_{i+1} - \x_* \|^2 \right) \notag\\& - (\frac{1}{2} \eta-\frac{\mu}{8}\eta^2 - \frac{1}{2}L \eta^2-
    \frac{4}{\mu} \frac{L^2 q  \eta^2}{|\mathcal{A}|}  -[\frac{\eta}{4}- \frac{\eta^3 q}{2}  \frac{L^2}{|\mathcal{A}|}] -\frac{\mu}{4 c_{\gamma}})\sum_{t=0}^{T}  \|\alpha^{(t-i)} \bd_i\|^2+ \frac{\mu}{4} T\frac{\epsilon}{c_{\epsilon}},
\end{align}
where $(a)$ follows from Eqs. \eqref{eqs90}.
Let $|\mathcal{A}|=q= \lceil\sqrt{n}\rceil $ and $\eta \leq \min\{\frac{1}{2\mu},\frac{1}{8L},\frac{\mu}{64L^2} \},c_{\gamma}\geq \frac{8\mu}{\eta}, c_{\epsilon}\geq \frac{\mu}{2}$, we have $(\frac{1}{2} \eta-\frac{\mu}{8}\eta^2 - \frac{1}{2}L \eta^2-
    \frac{4}{\mu} \frac{L^2 q  \eta^2}{|\mathcal{A}|}  -[\frac{\eta}{4}- \frac{\eta^3 q}{2}  \frac{L^2}{|\mathcal{A}|}] -\frac{\mu}{4 c_{\gamma}})> \frac{\eta}{32}>0$
    
    Thus, we have
\begin{align}
    & \sum_{t=0}^{T}   \sum_{s \in [S]} \lambda_t^{s} \left[ f_s(\x_i) - f_s(\x_*) \right]\notag\\&\leq \frac{1}{2 \eta} \left( (1-\frac{3\mu\eta}{4}) \sum_{i=0 }^{T} \| \x_i - \x_* \|^2 -  \sum_{t=0}^{T}\| \x_{i+1} - \x_* \|^2\right) +\frac{\epsilon}{2}.
\end{align}

Then, we have 

\begin{align}
    & \mathbb{E}  [\sum_{s \in [S]} \lambda_t^{s} \left[ f_s(\x_t) - f_s(\x_*) \right] ]\notag\\&\leq \frac{1}{2 \eta} \left( (1-\frac{3\mu\eta}{4}) \mathbb{E} \| \x_t - \x_* \|^2 -  \mathbb{E}\| \x_{t+1} - \x_* \|^2\right)+\frac{\epsilon}{2}.
\end{align}
Based on Assumption. \ref{assump: add} and averaging using weight $w_t = ( 1 - \frac{3\mu \eta }{4})^{1-t}$ and using such weight to pick output $\x$.
By using Lemma 1 in \cite{Karimireddy2020SCAFFOLD} with $\eta \geq \frac{1}{uR}$, we have

\begin{align}
    \mathbb{E}\|\x_t-\x^*\|^2 &\leq \| \x_0 - \x_* \|^2 \mu \exp( - \frac{3 \eta \mu T}{4}) \\
    &= \mathcal{O}(\mu \exp( - \mu T)).
\end{align}

Then we have the convergence rate $   \mathbb{E}\|\x_t-\x^*\|^2  = \mathcal{O}(\mu \exp( - \mu T))$.

The total sample complexity can be calculated as:
	$\lceil \frac{T}{q} \rceil n + T\cdot |\mathcal{A}| \leq  \frac{T+q}{q}n + T\sqrt{n}= T\sqrt{n}+n+T\sqrt{n}=O(n+ \sqrt{n} \ln ({\mu/\epsilon})$.
	Thus, the overall sample complexity is $\mathcal{O}(n+ \sqrt{n} \ln ({\mu/\epsilon})$.
	This completes the proof. 

\end{proof}

\section{Additional experiment results} \label{sec:addexp}

\textbf{1) Strongly-Convex Optimization:}

We conducted experiments to assess the performance of our algorithms on a strongly-convex optimization problem, where $\mathbf{F}(\x)=[f_1(\x) = \x^2, f_2(\x) = {e}^{-\x}]$. For this experiment, we selected hyperparameters $\eta=0.005$ and $\alpha=0.3$, while introducing stochasticity into the gradient by adding Gaussian noise with a range of (-1, 1).
As shown in Fig. \ref{fig_compare_sc}, it is evident that all of the algorithms successfully achieved convergence. Notably, the momentum-based algorithms, namely MOCO, \algmns, and \algmpns, exhibited faster convergence compared to MGD, MSGD, \algns, and \algpns. We would also like to note that there isn't a significant difference between the stochastic algorithms (SMGD, MGD) and other algorithms. This is not necessarily because the stochastic algorithms are inferior, but perhaps because the strongly-convex function in question is too simplistic.

\begin{figure}[h]
  \centering
  \includegraphics[width=0.3\textwidth]{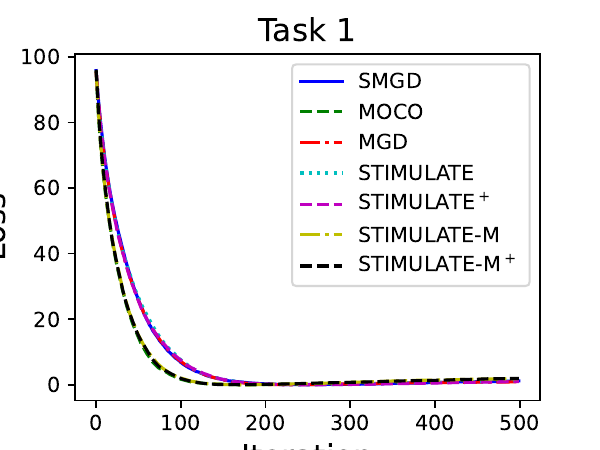}
  \includegraphics[width=0.3\textwidth]{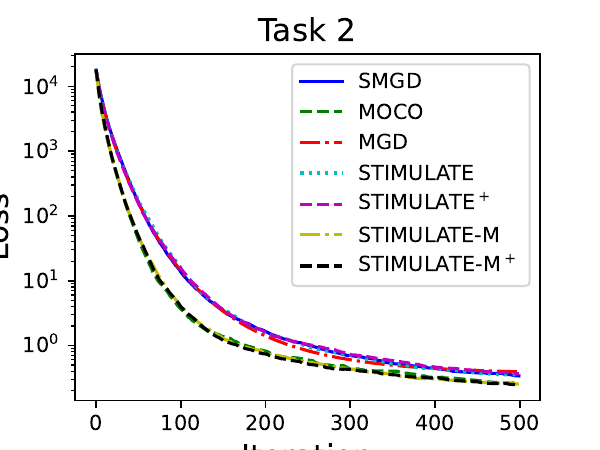}
\caption{Convergence comparison on strongly-convex optimization problem.}

\label{fig_compare_sc}
\end{figure}

\textbf{2) Eight-Objective Experiments on River Flow Dataset:}

\begin{wrapfigure}{r}{0.27\textwidth}
  \includegraphics[width=0.2\textwidth]{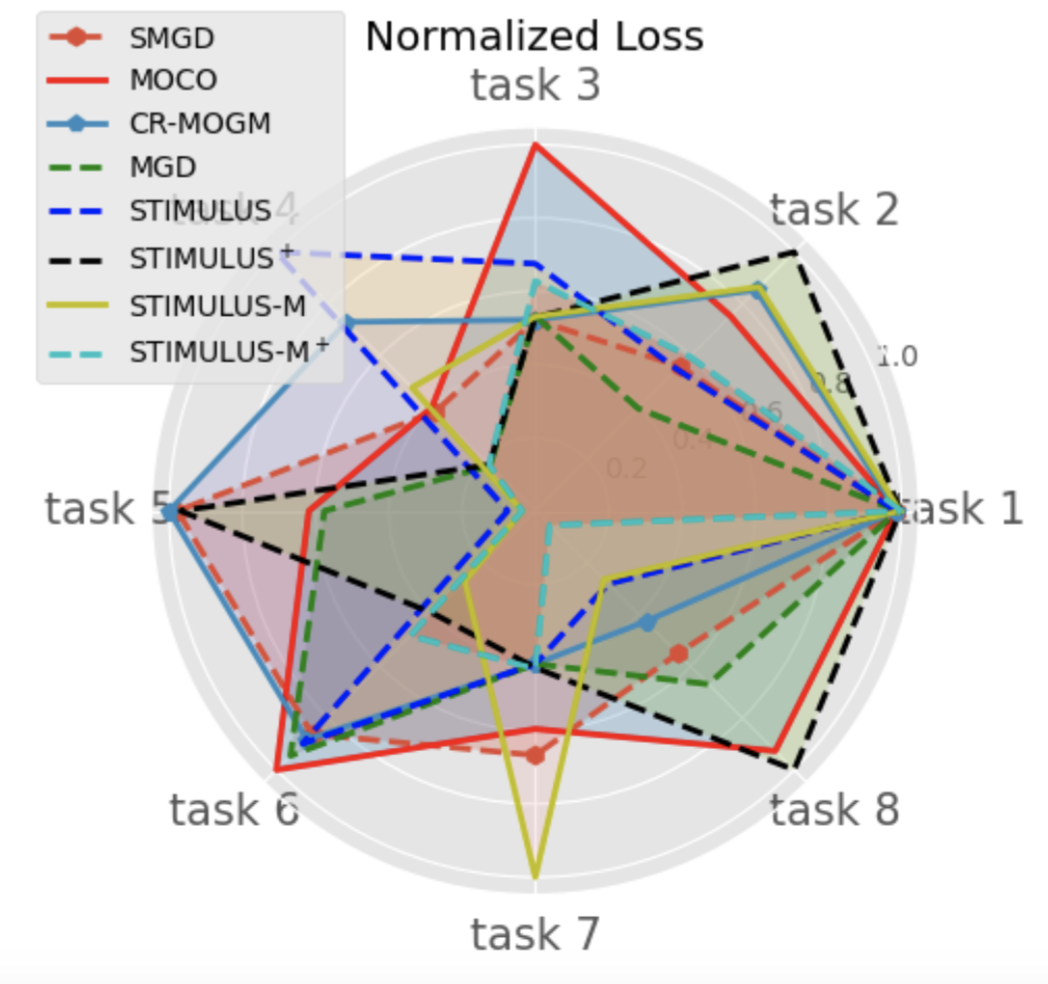}
\caption{Training loss convergence comparison (8-objective).}
\label{fig_compare_8tasks}
\end{wrapfigure}

We further test our algorithms on an 8-task problem with the river flow dataset~\citep{nie2017image}, which is for flow prediction at eight locations in the Mississippi river network. 
%
In this experiment, we set $\eta=0.001, \alpha=0.1$, the batch size for MOCO, CR-MOGM and SMGD is $8$, the full batch size for MGD is $128$, and the inner loop batch size $|\mathcal{N}_s|$ for \algns, \algmns, \algpns, \algmpns is eight. 
To better visualize different tasks, we plot the normalized loss in a radar chart as shown in Fig.~\ref{fig_compare_8tasks}, where we can see that our \alg algorithms achieve a much smaller footprint, which is desirable. 
Further, we compare the sample complexity results of all algorithms in Table~\ref{table_river}, which reveals a significant reduction in sample utilization by \algpns/\algmp compared to MGD, while achieving a much better loss compared to SGMD and MOCO (cf. Fig.~\ref{fig_compare_8tasks}).

\begin{table*}[h]
\centering
\caption{Results of normalized loss with the river flow dataset and learning tasks.}
{
\begin{tabular}{lcccccccccc}
\toprule
 & \multirow{2}{*}{\textbf{\!\!\!\!\!\!\!\!\!\!\!\!\!\!\!\!$\#$ of samples}} & \multicolumn{5}{c}{\textbf{Tasks}} \\
\cmidrule(lr){3-10} 
 && \multicolumn{1}{c}{0} & \multicolumn{1}{c}{1} & \multicolumn{1}{c}{2} & \multicolumn{1}{c}{3} & \multicolumn{1}{c}{4} & \multicolumn{1}{c}{5} & \multicolumn{1}{c}{6} & \multicolumn{1}{c}{7} \\
\midrule
 SMGD &8000 & 0.985&0.558&0.521&0.384& 1&
       0.862& 0.667& 0.550 \\
 MOCO &8000 & 0.985&0.753& 1&0.399& 0.632&1& 0.595& 0.926  \\
 MGDA & 128000 &0.989& 0.396&0.532& 0.174& 0.589&
       0.945&0.417&0.669&\\
\algns&  27200  &0.985& 0.546&0.675&1& 0.077&
       0.898&0.417&0.281 \\
\algpns & 20947 & 0.996&1&0.528&0.178&0.990&
       0.395&0.427& 1\\
\algmns&  27200 & 0.996&0.864&0.530&0.475&0.036&
       0.271& 1&0.264\\
\algmpns &21085 & 1&0.596& 0.627& 0.1781&0.0376&
       0.482& 0.430& 0.055  \\
\bottomrule
\label{table_river}
\end{tabular}}
\end{table*}

\textbf{3) Ablation study on momentum in STIMULUS-M:}

\begin{table}[h]
    \centering
    \caption{Loss value vs. Iteration on tasks L and R of STIMULUS-M.}
    \begin{tabular}{c|cccc|cccc}
        \toprule
        \multirow{2}{*}{Momentum Term $\alpha$} & \multicolumn{4}{c|}{Task L} & \multicolumn{4}{c}{Task R} \\
        \cmidrule(lr){2-5} \cmidrule(lr){6-9}
        & 100 & 200 & 300 & 500 & 100 & 200 & 300 & 500 \\
        \midrule
        0.1 & 0.0228 & 0.0207 & 0.0203 & 0.0153 & 0.0229 & 0.0223 & 0.0205 & 0.0215 \\
        0.3 & 0.0228 & 0.0182 & 0.0179 & 0.0120 & 0.0223 & 0.0191 & 0.0168 & 0.0143 \\
        0.5 & 0.0227 & 0.0174 & 0.0146 & 0.0078 & 0.0215 & 0.0180 & 0.0124 & 0.0091 \\
        0.8 & 0.0225 & 0.0158 & 0.0127 & 0.0065 & 0.0210 & 0.0152 & 0.0113 & 0.0078 \\
        \bottomrule
    \end{tabular}
    \label{tablem}
\end{table}

We performed additional experiments to analyze the impact of varying the momentum term in our proposed STIMULUS-M algorithm, as shown in Table \ref{tablem}, on the classification task of the MultiMNIST dataset. The experimental settings are consistent with those in Section 5.1 of the main paper. These results indicate that a larger momentum term leads to faster convergence.

\end{document}